\documentclass[twoside]{article}
\usepackage[accepted]{aistats2023}

\usepackage{parskip}
\usepackage{braket}
\usepackage{hyperref}
\usepackage{enumerate}
\usepackage{wrapfig}
\usepackage[low-sup]{subdepth}
\usepackage{bm}
\usepackage[normalem]{ulem}
\usepackage{algorithm}
\usepackage{algpseudocode}
\usepackage{xcolor}

\usepackage{hyperref}
\usepackage{dsfont}
\usepackage{float}
\usepackage{amsfonts,amsthm}
\usepackage[capitalise]{cleveref}
\usepackage{multirow}
\usepackage{rotating}

\usepackage[textsize=tiny]{todonotes}

\newcommand{\oO}{\mathcal{O}}
\newcommand{\cO}{\mathcal{O}}
\newcommand{\cC}{\mathcal{C}}
\newcommand{\cG}{\mathcal{G}}
\newcommand{\cM}{\mathcal{M}}
\newcommand{\cS}{\mathcal{S}}
\newcommand{\sS}{\mathcal{S}}
\newcommand{\cA}{\mathcal{A}}
\newcommand{\aA}{\mathcal{A}}
\newcommand{\cH}{\mathcal{H}}
\newcommand{\hH}{\mathcal{H}}
\newcommand{\RR}{\mathbb{R}}
\newcommand{\NN}{\mathbb{N}}
\newcommand{\PP}{\mathbb{P}}
\newcommand{\T}{\top}
\newcommand{\bR}{\mathbb{R}}
\newcommand{\poly}{\text{poly}}
\newcommand{\prodA}{\prod_{i=1}^m |A^\ag{i}|}
\newcommand{\sumA}{\sum_{i=1}^m |A^\ag{i}|}
\newcommand{\allm}{{(1:m)}}
\newcommand{\ag}[1]{{(#1)}}
\newcommand{\eye}{\mathbf{I}}
\newcommand{\critinf}{\tilde \Gamma(\lambda, \log(2))}
\newcommand{\scritinf}{\tilde \Gamma}
\newcommand{\kernel}{\mathbf{k}}

\DeclareMathOperator*{\argmax}{arg\,max}
\DeclareMathOperator*{\argmin}{arg\,min}

\newtheorem{example}{Example} 
\newtheorem{theorem}{Theorem}
\newtheorem{lemma}[theorem]{Lemma} 
\newtheorem{proposition}[theorem]{Proposition}

\newtheorem{assumption}{Assumption}

\crefname{assumption}{Assumption}{Assumptions}
\Crefname{assumption}{Assumption}{Assumptions}

\usepackage[round]{natbib}

\begin{document}

%

%
\runningauthor{Volodymyr Tkachuk, Seyed Alireza Bakhtiari, Johannes Kirschner, Matej Jusup, Ilija Bogunovic, Csaba Szepesvári}
\newcommand{\maintitle}{Efficient Planning in Combinatorial Action Spaces with Applications to Cooperative Multi-Agent Reinforcement Learning}

\twocolumn[
\aistatstitle{\maintitle}

\aistatsauthor{ 
	Volodymyr Tkachuk$^*$  \And
	Seyed Alireza Bakhtiari$^*$ \And
	Johannes Kirschner}

\aistatsaddress{ University of Alberta \And University of Alberta  \And University of Alberta}

\aistatsauthor{ 
	Matej Jusup \And
	Ilija Bogunovic\And
	Csaba Szepesvári }

\aistatsaddress{ETH Zurich \And University College London \And University of Alberta/DeepMind }  ]

\begin{abstract}
A practical challenge in reinforcement learning are combinatorial action spaces that make planning computationally demanding. For example, in cooperative multi-agent reinforcement learning, a potentially large number of agents jointly optimize a global reward function, which leads to a combinatorial blow-up in the action space by the number of agents. As a minimal requirement, we assume access to an argmax oracle that allows to efficiently compute the greedy policy for any Q-function in the model class. Building on recent work in planning with local access to a simulator and linear function approximation, we propose efficient algorithms for this setting that lead to polynomial compute and query complexity in all relevant problem parameters. For the special case where the feature decomposition is additive, we further improve the bounds and extend the results to the kernelized setting with an efficient algorithm. 
\end{abstract}

\section{\uppercase{Introduction}} \label{sec: introduction}

Reinforcement learning (RL) is concerned with training data-driven agents to make optimal decisions in interactive environments. 
An agent interacts with an environment by choosing actions and observing its state and a reward signal. 
The goal is to learn an optimal policy that maximizes the total reward. 
Efficiently computing optimal policies, also known as \emph{planning}, is therefore at the heart of any reinforcement learning algorithm.

Recent works have successfully applied reinforcement learning algorithms to complex domains including video games \citep{mnih2013playing}, tokamak plasmas control \citep{degrave2022magnetic}, robotic manipulation tasks \citep{akkaya2019solving}, to name a few. A common theme of these works is that the agent is trained on a simulated environment. This provides additional flexibility on how the agent can interact with the environment. A reasonable assumption is that the internal state of the simulator can be saved (`checkpointing') and later revisited. 
\looseness=-1


In this work, we formally study \emph{efficient planning with local access to a simulator}. The local access model was recently proposed by \citet{yin2021efficient} with the goal of making the simulation access model more practical in applications. Local access means that the only states at which the planner can query the simulator are the initial state or states returned in response to previously issued queries. Efficient planning means that given an initial state, the learner outputs a near-optimal policy using polynomial compute and queries in all relevant parameters. 

\newcommand{\otherfactors}{\frac{1}{1 - \gamma}, \frac{1}{\kappa}, \log(\frac{1}{\delta}), \log(b), \frac{1}{\epsilon}}
\newcommand{\otherfactorshidden}{}
\begin{table*}[t]
	\centering
	\begin{tabular}{c | c | c || c  c | c} 
		\hline
  		\multicolumn{2}{c|}{{Algorithms}}
            & Query $\epsilon = 0$
            & Query $\epsilon > 0$
		& Sub-optimality $\epsilon > 0$
		& Computation
		\\ \cline{3-4} 
		\hline
		\multirow{4}{*}[-1.3em]{\begin{turn}{90}\textsc{LSPI}\end{turn}}
	    & \textsc{Naive}
	    & $\tilde{\mathcal{O}} \left( \frac{d^3}{\kappa^2 (1-\gamma)^8} \right)$
	    & $\tilde{\mathcal{O}} \left( \frac{d^2}{\epsilon^2 (1-\gamma)^4} \right)$
		& $\tilde{\mathcal{O}} \left( \frac{\epsilon \sqrt{d}}{(1-\gamma)^2} \right)$
		& $ \poly(|\cA|, d\otherfactorshidden)$
	    \\
	    & \textsc{EGSS}
		& $\tilde{\mathcal{O}} \left( \frac{d^{3 + \boldsymbol{1}}}{\kappa^2 (1-\gamma)^8} \right)$
	    & $\tilde{\mathcal{O}} \left( \frac{d^2}{\epsilon^2 (1-\gamma)^4} \right)$
		& $\tilde{\mathcal{O}} \left( \frac{\epsilon \sqrt{d} \boldsymbol{\sqrt{d}}}{(1-\gamma)^2} \right)$
		& $ \poly(d\otherfactorshidden)$
		\\
	    
	    & \textsc{DAV}
	    & $\tilde{\mathcal{O}} \left( \frac{\boldsymbol{m^2} d^3}{\kappa^2 (1-\gamma)^8} \right)$
	    & $\tilde{\mathcal{O}} \left( \frac{d^2}{\epsilon^2 (1-\gamma)^4} \right)$
		& $\tilde{\mathcal{O}} \left( \frac{\epsilon \sqrt{d} \boldsymbol{m}}{(1-\gamma)^2} \right)$
		& $\poly(\sum_{i=1}^m |\cA^\ag{i}|, d\otherfactorshidden)$
		\\
	    & \textsc{KERNEL-DAV}
		& $\tilde{\mathcal{O}} \left( \frac{\boldsymbol{m^2} \tilde \Gamma^3}{\kappa^2 (1-\gamma)^8} \right)$
	    & $\tilde{\mathcal{O}} \left( \frac{\tilde \Gamma^2}{\epsilon^2 (1-\gamma)^4} \right)$
		& $\tilde{\mathcal{O}} \left( \frac{\epsilon \sqrt{\tilde \Gamma} \boldsymbol{m}}{(1-\gamma)^2} \right)$
		& $\poly(\sum_{i=1}^m |\cA^\ag{i}|, \scritinf \otherfactorshidden)$
		\\ 
		\hline
		\multirow{4}{*}[-1.3em]{\begin{turn}{90}\textsc{POLITEX}\end{turn}}
	    & \textsc{Naive}
		& $\tilde{\mathcal{O}} \left( \frac{d^{3}}{\kappa^4 (1-\gamma)^9} \right)$
	    & $\tilde{\mathcal{O}} \left( \frac{d}{\epsilon^4 (1-\gamma)^5} \right)$
		& $\tilde{\mathcal{O}} \left( \frac{\epsilon \sqrt{d}}{(1-\gamma)} \right)$
		& $ \poly(|\cA|, d\otherfactorshidden)$
	    \\
	    
	    & \textsc{EGSS}
		& $\tilde{\mathcal{O}} \left( \frac{\boldsymbol{m} d^{3 + \boldsymbol{1}}}{\kappa^4 (1-\gamma)^9} \right)$
	    & $\tilde{\mathcal{O}} \left( \frac{\boldsymbol{m} d}{\epsilon^4 (1-\gamma)^5} \right)$
		& $\tilde{\mathcal{O}} \left( \frac{\epsilon \sqrt{d} \boldsymbol{\sqrt{d}}}{(1-\gamma)} \right)$
		& $ \poly(\sum_{i=1}^m |\cA^\ag{i}|, d \otherfactorshidden)$
		\\
	    
	    & \textsc{DAV}
		& $\tilde{\mathcal{O}} \left( \frac{\boldsymbol{m^3} d^3}{\kappa^4 (1-\gamma)^9} \right)$
	    & $\tilde{\mathcal{O}} \left( \frac{\boldsymbol{m} d}{\epsilon^4 (1-\gamma)^5} \right)$
		& $\tilde{\mathcal{O}} \left( \frac{\epsilon \sqrt{d} \boldsymbol{m}}{(1-\gamma)} \right)$
		& $\poly(\sum_{i=1}^m |\cA^\ag{i}|, d\otherfactorshidden)$
		\\
	    
	    & \textsc{KERNEL-DAV}
		& $\tilde{\mathcal{O}} \left( \frac{\boldsymbol{m^2} \tilde \Gamma^3}{\kappa^2 (1-\gamma)^8} \right)$
	    & $\tilde{\mathcal{O}} \left( \frac{\boldsymbol{m} \tilde \Gamma}{\epsilon^4 (1-\gamma)^5} \right)$
		& $\tilde{\mathcal{O}} \left( \frac{\epsilon \sqrt{\tilde \Gamma} \boldsymbol{m}}{(1-\gamma)} \right)$
		& $\poly(\sum_{i=1}^m |\cA^\ag{i}|, \scritinf \otherfactorshidden)$
		\\
		\hline
	\end{tabular}
	\caption{Query complexity and sub-optimality bounds of algorithms proposed in Section \ref{sec: algorithms} and \ref{sec:additive} in the realizable ($\epsilon=0$) and $\epsilon$-misspecified ($\epsilon> 0$) setting. \textsc{Naive} refers to a direct implementation of the approach by \citet{yin2021efficient}. For $\epsilon=0$,the sub-optimality gap is $\kappa>0$, while for $\epsilon > 0$, the sub-optimality gap is given in the third column. All algorithms require $\cO(\poly(\otherfactors))$ computation. \textsc{LSPI-EGSS} requires access to a `greedy oracle' (\cref{ass: argmax oracle}). Results for \textsc{LSPI-(Kernel-)DAV} and $\textsc{POLITEX}$ hold for product action sets $\cA = \cA^{(1)} \times \cdots \times \cA^{(m)}$ and \cref{ass: feature decomposition}.\looseness=-1
	}\label{table:2}
\end{table*}

Motivated by the increasing complexity of applications, we specifically study the case where the state space is large or continuous.
To avoid the query complexity scaling with the size of the state space, it is standard to introduce linear function approximation \citep[e.g.,][]{bertsekas1996temporal,lagoudakis2003least,munos2005error,lattimore2020learning}. 
In particular, we assume linear $\epsilon$-realizability of joint state-action value functions for \emph{all} policies. 
This is motivated by the recent realization that realizability of the optimal state-action value function alone is not sufficient to develop a query-efficient planner~\citep{weisz2021exponential}. 
However, even under stronger realizability assumptions, previous approaches are not computationally efficient in the case where the action space is \emph{combinatorially large}, and direct enumeration of the action space becomes infeasible.
Therefore, we work with a minimal oracle assumption that allows us to compute the greedy policy for any Q-function in the model class (which amounts to solving a linear optimization over the feature space).

One prominent special case of this setting is multi-agent reinforcement learning. 
Multi-agent reinforcement learning has been a recent research focus with multiple promising attempts at tackling complex multi-agent problems, e.g., team games \citep{baker2019emergent}, large scale traffic signal control \citep{chu2019multi}, cooperative controls in powergrids \citep{chen2021powernet} among others.
Naively applying single-agent planning algorithms fails to achieve efficiency in the multi-agent setting because the single-agent algorithms typically face an exponential blow-up of the action space in the number agents. 
In many practical tasks, however, there is an inherent structure in the underlying dynamics that can be exploited to address both efficiency and scalability issues. \looseness=-1

\paragraph{Contributions} Our first contribution is a novel oracle-efficient variant of the Confident Monte-Carlo least-squares policy iteration (\textsc{Confident MC-LSPI}) algorithm by \citet{yin2021efficient}, for combinatorially large action spaces. 
The key insight is an efficient implementation of the \emph{uncertainty check}, that determines the diversity of the state and action set used for estimation. 
We also study a special case where the Q-function has an additive structure in the features (formally introduced in \Cref{ass: feature decomposition}), which leads to improved bounds in the regime where the dimension is large. 
In the multi-agent setting, the decomposition corresponds to agent-specific features, and the proposed algorithms achieve polynomial compute and query complexity in the number of agents and other quantities of interest. 
We further introduce a kernelized variant, which under the corresponding additivity assumption admits an efficient implementation. 
Lastly, the additive structure leads to an efficient implementation of the \textsc{Confident MC-Politex} algorithm that admits improved bounds in the misspecified setting. 
The formal results are summarized in \cref{table:2}.\looseness=-1


\section{\uppercase{Related Work}} \label{sec:related work}

Computing optimal policies, also known as \emph{planning}, is a central challenge in reinforcement learning \citep{sutton2018reinforcement,szepesvari2010algorithms}. The two most classical planning algorithms are value iteration \citep{bellman1957markovian} and policy iteration \citep{howard1960dynamic}. Approximate versions of value and policy iteration were analyzed by \citet{munos2003error,munos2005error,farahmand2010error}. A common setting is planning with a \emph{generative model} (also \emph{global} simulator access), where the learner can query the transition kernel at any state and action \citep{kakade2003sample}. In the corresponding tabular setting the query complexity of value and policy iteration are completely understood \citep[e.g.,][]{azar2012sample,azar2013minimax}. When combined with function approximation, the picture becomes more nuanced. A lower bound under misspecification was provided by \citet{du2019good}. Sample complexity bounds for least-squares policy iteration \citep{bertsekas1996temporal,lagoudakis2003least} are by \citet{lattimore2020learning}. The latter work combines a G-experimental design over state-action pairs with Monte-Carlo rollouts to obtain value estimates for the policies. In similar fashion, least-squares value iteration (LSVI) was analyzed in the generative model setting \citep{agarwal2020reinforcement}. Yet another approach is Politex \citep{abbasi2019politex,rltheory2022}, which uses mirror descent to improve the policy. \looseness=-1

A much larger body of work focuses on the online setting, where the learner interacts with the environment in one or multiple episodes. Early work that uses function approximation includes \citep{bradtke1996linear,melo2007q}. Recent works provide query complexity guarantees under various models \citep{osband2016generalization,yang2020function,ayoub2020model,zanette2020learning,du2021bilinear,zhou2021nearly}. This includes approaches that are computationally efficient for small action sets \citep{jin2020provably,agarwal2020pc}. We are not aware of provably query efficient algorithms with \emph{only} linear $Q_\pi$-realizability (\cref{asm:linear-q-pi})  for the online setting. \Citet{abbasi2019politex,lazic2021improved,wei2021learning} prove bounds with a \emph{feature excitation} condition, although these works do not consider large action sets. Negative results under weaker assumptions are known, e.g.~for $Q^*$-realizability \citep{weisz2021exponential} and approximate $Q_\pi$-realizability \citep{du2019good}.\looseness=-1

Recently, \citet{yin2021efficient} introduced the \emph{local access} model, in which the learner can query the simulator at the initial state or any state observed during planning. They further introduce a Monte-Carlo policy iteration algorithm that provides the basis of our work. Different to this previous work, we consider the combinatorial action set setting, and provide new algorithms that avoid scaling of the computational complexity with the size of the action set. 
Least-squares value iteration with local access was analyzed by \citet{hao2022confident}. For a detailed discussion on different simulators models we refer the reader to  \citep{yin2021efficient}.


Relatively few related works on computationally efficient planning in MDPs are concerned with combinatorial action spaces. 
This topic has received attention in the context of factored MDPs in planning \citep{dean1998solving,geisser2020trial,raghavan2012planning}, online RL \citep{osband2014near, xu2020near, tian2020towards, chen2020efficient} and in the empirical literature \citep{delarue2020reinforcement,hubert2021learning} with applications to vehicle routing and control problems. We are not aware of prior work with query complexity guarantee for MDPs with large action sets, however there is a long line of works on combinatorial bandits \citep[e.g.,][]{cesa2012combinatorial,chen2013combinatorial,shleyfman2014combinatorial,combes2015combinatorial,jourdan2021efficient}. Relevant in this context are also kernelized bandit algorithms (Bayesian optimization) that exploit additive structure of the reward function \citep{kandasamy2015high,wang2019improving,kirschner2021bias,mutny2018efficient,rolland2018high}. We consider a similar assumption in \cref{sec:additive} as a special case.


Multi-agent reinforcement learning  \citep{busoniu2008comprehensive,zhang2021multi} can be understood as a combinatorial setting, which has a large body of works on its own. Query complexity bounds focus mostly on the competitive setting, e.g. in tabular Markov games (e.g., \cite{shapley1953stochastic, song2021can,tian2021online,bai2020provable,liu2021sharp,leonardos2021global}). 
One of the key challenges is the exponential blowup in the action space with the number of agents, which is sometimes refered to as `curse of multi-agents'. \Citet{jin2021v} introduce a computationally efficient algorithm for tabular Markov games. Multi-agent reinforcement learning with function approximation is studied by \cite{huang2021towards,chen2021almost,jin2020provably}. These works consider the competitive setting and focus on obtaining query efficient algorithms, while the approaches are not computationally tractable.  In the limit where the number of agents becomes large, previous work uses mean-field approximations \citep{yang2018mean,pasztor2021efficient}.





Most closely related is \emph{cooperative} multi-agent learning. Early work by \citet{guestrin2001multiagent} proposes the use of factored MDPs to make planning tractable via message passing algorithms. \Citet{rashid2018qmix} propose a neural network architecture that allows to decouple the agent rewards in a way such that the greedy policy can be computed efficiently. 
The goal of these works is to ensure the greedy policy can be computed efficiently. \Citet{zohar2021locality} consider a setting where a graph structure captures the reward dependencies across the agents; however the guarantees they provide apply only to the bandit setting.\looseness=-1

\section{\uppercase{Preliminaries}} \label{sec:preliminaries}
We consider reinforcement learning in an infinite-horizon Markov decision process (MDP) specified by a tuple $\cM = \left(\cS, \cA, \PP, r, \gamma\right)$. 
As usual, $\cS$ denotes the state space, $\cA$ denotes the action space, and
$\PP : \cS \times \cA \rightarrow \Delta_{\cS}$ is the transition kernel, where $\Delta_{\cS}$ denotes the set of probability measures over $\cS$. 
Given a state $s \in \cS$ and action vector $a \in \cA$, the system transits to a new state $s' \sim \PP(s, a)$. The reward function is 
$r : \cS \times \cA \rightarrow [0, 1]$ and 
$\gamma \in [0, 1)$ is the discount factor.\looseness=-1

A stationary policy $\pi :\cS \rightarrow \Delta_\cA$ maps states to a distribution over $\cA$.
The state value function $V_\pi: \cS \to \RR$ of a policy $\pi$ from a state {$s \in \cS$} is 
\begin{align*}
	V_\pi(s) = \mathbb{E}_{\pi}\Bigg[\sum_{t=0}^\infty \gamma^t r(s_t, a_t)\Big|s_0 = s\Bigg] \,.
\end{align*}
The expectation is over the sequence of states $(s_t)_{t \in \NN}$ and actions  $(a_t)_{t \in \NN}$ queried from the transition kernel $\PP$ and the policy $\pi$. 
A policy $\pi^*$ is \emph{optimal} if $V_{\pi^*} = \max_{\pi} V_{\pi}$.

The Q-function $Q_\pi: \cS \times \cA \to \RR$ of a policy $\pi$ 
is defined for $s \in \cS$ and $a \in \cA$ as
\begin{align*}
	Q_{\pi}(s,a) = r(s,a)  + \gamma \mathbb{E}_{s' \sim \PP(s,a)}\left[V_{\pi}(s')\right].
\end{align*} 

In the following we assume that we are given a state-action feature map $\phi: \cS \times \cA \to \RR^d$, that allows to approximate the $Q$-function of any policy as a linear function.\looseness=-1
\begin{assumption}[Linear $Q_\pi$-realizability]
	\label{asm:linear-q-pi}
	For each policy $\pi$ there exists a weight vector $w_\pi \in \RR^d, \|w_\pi\|_2 \leq b $ satisfying  $\max_{s,a} |Q_\pi(s,a) -w_\pi^\T  \phi(s,a)| \le \epsilon$.
\end{assumption}
The assumption is commonly used in combination with policy iteration algorithms \citep{lattimore2020learning,zanette2020learning}.  In particular, the assumption allows to obtain query complexity results that are independent of the number of states and actions. We remark that the linear MDP assumption \citep{jin2020provably} implies $Q_\pi$-realizability, but not vice versa. We also make the following standard boundedness assumption:

\begin{assumption}[Bounded features]
	\label{ass: bounded features}
	We assume that $\|\phi(s, a)\|_2 \le 1$ for all $(s, a) \in \cS \times \cA$.
\end{assumption}

Our main objective is to obtain query and computationally efficient algorithms for the case where the action set $\cA$ is \emph{combinatorially} large, and direct enumeration becomes infeasible. To obtain meaningful results in this setting, we assume that the \emph{offline problem} of computing the greedy policy given a \emph{fixed} approximator $w \in \RR^d$ can be solved efficiently. This is formally captured in the next assumption.\looseness=-1
\begin{assumption}[Greedy oracle]
	\label{ass: argmax oracle}
	We have access to an oracle $\cG$ which takes as input a vector $w \in \RR^d$, a state $s \in \cS$ and a feature function $\phi: \cS \times \cA \to \RR^d$ and returns an action that maximizes $w^\T \phi(s, a)$.
	Formally
	\begin{align*}
		\cG(w,\phi) = \argmax_{a \in \cA} w^\T \phi(s, a)\,,
	\end{align*}
	with ties broken arbitrarily.
\end{assumption}
Combined with the linear $Q_\pi$-realizability (\cref{asm:linear-q-pi}), the greedy oracle amounts to solving a \emph{linear} optimization over the action set $\cA$. This is a reasonable assumption, as optimized solvers are available for many settings. It is also a \emph{minimal} assumption in the sense that it is required to implement a policy iteration procedure. Note that the assumption can be relaxed to require only an $\epsilon$-approximate solution, which is essentially equivalent to misspecification (\cref{asm:linear-q-pi}). In \cref{sec:additive} we provide an additive model where the oracle can be directly implemented.

Our goal is to find a computational and query efficient algorithm that given a starting state $\rho \in \cS$ returns a $\kappa$-optimal policy $\hat \pi$, i.e.~$V_{\pi^*}(\rho) - V_{\hat \pi}(\rho) \leq \kappa$ for $\kappa > 0$ while minimizing the number of queries needed. 
To obtain queries, the learner is given \emph{local access} to a simulator of the MDP \citep{yin2021efficient}.
A simulator of the MDP takes as input a state-action pair $(s,a) \in \cS \times \cA$ and returns a next state $s^\prime \sim \PP(s, a)$ and reward $r(s,a)$.
A local access simulator restricts the input state $s \in \cS$ only to those states which have been visited previously.





An important example where the action set is typically large is cooperative multi-agent reinforcement learning.
\begin{example}[Cooperative multi-agent RL] \label{ex:multi-agent} In the multi-agent setting, $m \in \NN$ agents act jointly on the MDP $\cM$.
Each agent $i \in [m]$ has a set of actions $\cA^\ag{i}$ available where $[m] := \{1,\dots,m\}$.
We denote the joint action set by $\cA = \cA^{(1:m)} := \cA^{(1)} \times ... \times \cA^{(m)}$. 
	 The state space $\cS$ is joint for all agents.
	 A centralized, stationary policy $\pi :\cS \rightarrow \Delta_{\cA^{(1:m)}}$ maps states to a distribution over $\cA^{(1:m)}$. In the \emph{cooperative} setting, the agents jointly maximize a global reward function $r : \cS \times \cA^{(1:m)} \rightarrow [0, 1]$.

\end{example}
Note that the size of the joint action set is exponential in the number of agents, which makes approaches designed for the single agent setting computationally intractable. We will revisit this example in \cref{sec:additive} where we discuss how an additive feature decomposition leads to algorithms that scale polynomially in the number of agents $m$. We remark that prior work on multi-agent RL has focused on architectures where the greedy policy can be computed efficiently \citep[e.g.,][]{guestrin2001multiagent,rashid2018qmix,delarue2020reinforcement,zohar2021locality}. \looseness=-1


\newcommand{\MAMCLSPIPOLITEX}{\textsc{MARL-MC-LSPI/Politex}}
\newcommand{\MCLSPI}{\textsc{Confident MC-LSPI}}
\newcommand{\MCLSPIEGSS}{\textsc{Confident MC-LSPI-EGSS}}
\newcommand{\ConfidentRollout}{\textsc{Rollout}}
\newcommand{\UncertaintyCheck}{\textsc{UncertaintyCheck}}
\newcommand{\UncertaintyCheckEGSS}{\textsc{UncertaintyCheck-EGSS}}
\newcommand{\DONE}{\textsc{done}}
\newcommand{\NONE}{\textsc{none}}
\newcommand{\UNCERTAIN}{\textsc{uncertain}}
\newcommand{\CERTAIN}{\textsc{certain}}
\newcommand{\LSPI}{\textsc{LSPI}}
\newcommand{\EGSS}{\textsc{EGSS}}
\newcommand{\DAV}{\textsc{DAV}}
\newcommand{\NAIVE}{\textsc{NAIVE}}
\newcommand{\Politex}{\textsc{Politex}}
\newcommand{\amax}{\hat a}

\section{\uppercase{Efficient MC-LSPI}} \label{sec: algorithms}
In this section, we extend the \textsc{Confident MC-LSPI} algorithm proposed by \cite{yin2021efficient} to the combinatorial action setting. 
More precisely, \cref{alg:confident ma mc-lspi} with \cref{alg:uncertainty check} used for the \UncertaintyCheck \ is equivalent to the \textsc{Confident MC-LSPI} algorithm presented in \cite{yin2021efficient}, which relies on either enumerating the action set or solving a quadratic maximization problem, both which become infeasible for large $\cA$ in general \citep[e.g.,][]{bhattiprolu2021framework}. 
The main challenge is to come up with a procedure that uses only polynomially many calls to the greedy oracle while also scaling polynomially in all other quantities of interest.\looseness=-1


At a high level, \cref{alg:confident ma mc-lspi} alternates between policy evaluation and policy improvement. 
For evaluation, a core set is constructed that holds a small but sufficiently diverse set of features corresponding to state-action pairs. 
For each element of the core set, the \ConfidentRollout~routine (\cref{alg:confident rollout ma}) returns a Monte-Carlo estimate of the Q-value. 
During each rollout, the \UncertaintyCheck~subroutine (\cref{alg:uncertainty check egss}) determines if a feature should be added to the core set. 
This procedure is repeated until no more elements are added to the core set. 
The Monte-Carlo returns from the rollouts are then used to construct a least-squares estimate of $Q_\pi(s,a)$, which in turn is used to improve the policy.

Formally, the outer loop aims to complete $K$ iterations of policy iteration.
The goal of each iteration $k$ is to estimate $Q_{\pi_{k-1}}$ using a weight vector $w_k \in \RR^d$ and derive a new greedy policy $\pi_k$, w.r.t.~$w_k$. 
For estimation, the algorithm maintains a \emph{core set} $\cC$ with elements corresponding to state-action pairs. The elements of the core set $z = (z_s, z_{a}, z_\phi, z_q) \in \cC$ are tuples containing a state $z_s \in \cS$, an action $z_{a} \in \cA$, the corresponding feature $z_\phi \in \RR^d$ , and a value estimate $z_q \in \RR \cup \{\NONE\}$.
We denote the vector of all value estimates in the core set as $q_\cC = (z_q)_{z \in \cC} \in \bR^{|\cC|}$.
The weight vector $w_k$ to estimate $Q_{\pi_{k-1}}$ is computed using regularized least squares, with $q_\cC$ as the targets (line 16). 
An improved policy based on $w_k$ is then calculated by following the greedy policy with respect to $w^\T \phi(s,a)$ (line 17). 
The core set is initialized in lines 3-8  by adding the initial state with a \emph{default action} $\bar a$, so that there is at least one element in the core set to rollout from (line 3). 
Then we continuously run the \UncertaintyCheck~algorithm until it stops returning a status of \UNCERTAIN, and add the uncertain tuple to the core set each time. 
This is to ensure that the final policy\footnote{The algorithm returns $\pi_{K-1}$ instead of $\pi_K$ because the proof requires that the uncertainty checks for the final policy pass. This is only ensured for $\pi_{K-1}$.} $\pi_{K-1}$  returned by the main algorithm is approximately optimal from the initial state $\rho$, and this can be insured if all the uncertain actions (from $\rho$) are added to the core set (details in Appendix \ref{app:efficient uncertainty check}).\looseness=-1

\begin{algorithm}[t]
	\caption{\textsc{Confident MC-LSPI}} \label{alg:confident ma mc-lspi}  
	\begin{algorithmic}[1]
		\State \textbf{Input:} initial state $\rho$, initial policy $\pi_0$, number of iterations $K$, threshold $\tau$, number of rollouts $n$, length of rollout $H$
		\State \textbf{Globals:} default action $\bar a$, regularization coefficient $\lambda$, discount $\gamma$, subroutine \UncertaintyCheck
		\State {$\cC \gets \{(\rho, \bar a, \phi(\rho, \bar a), \NONE)\}$} 
		\State status, result $\gets \UncertaintyCheck(\rho,  \cC, \tau)$
		\While {status $=$ \UNCERTAIN}
		\State $\cC \gets \cC \cup \{\text{result}\}$
		\State status, result $\gets \text{\UncertaintyCheck}(\rho, \cC, \tau)$
		\EndWhile
		\State $z_q \gets \text{\NONE}, \, \forall z \in \cC$ \quad \Comment{Policy iteration starts $(*)$}
		\For {$k \in 1, \dots, K$}
		\For {$z \in \cC$}
		\State status, result $\gets \text{\ConfidentRollout}(n, H, \pi_{k-1}, z, \cC, \tau)$
		\State \textbf{if} status $=$ \DONE, \textbf{then} $z_q = \text{result}$
		\State \textbf{else} $\cC \gets \cC \cup \{\text{result}\}$ and \textbf{goto} line $(*)$ 
		\EndFor 
		\State $w_k \gets (\Phi_\cC^\top \Phi_\cC + \lambda I)^{-1} \Phi_\cC^\top q_\cC$
		\State $\pi_k(a|s) \gets \mathds{1}\big(a = \argmax\limits_{\tilde{a} \in \cA} w_k^\top \phi(s, \tilde{a})\big)$
	\EndFor
	\State \Return $\pi_{K-1}$ 
\end{algorithmic}
\end{algorithm}

\begin{algorithm}[t]
	\caption{\ConfidentRollout} \label{alg:confident rollout ma}  
	\begin{algorithmic}[1]
		\State \textbf{Input:} number of rollouts $n$, length of rollouts $H$, rollout policy $\pi$, core set element $z$, core set $\cC$, threshold $\tau$.
		\For {$i = 1, ..., n$}
		\State $s_{i, 0} \gets z_s, a_{i, 0} \gets z_a$
		\State Query the simulator, obtain $r_{i, 0} \gets r(s_{i, 0}, a_{i, 0})$, and the next state $s_{i, 1}$
		\For {$t = 1, ..., H$} 
		\State status, result${\gets} \text{\UncertaintyCheck}(s_{i, t}, \cC, \tau)$
		\If {status = \UNCERTAIN} 
		\State \Return {status, result}
		\EndIf
		\State Sample $a_{i, t} \sim \pi(\cdot | s_{i,t})$
		\State Query the simulator with $s_{i, t}, a_{i, t}$, obtain $r_{i, t} \gets r(s_{i, t}, a_{i, t})$, and next state $s_{i, {t+1}}$
		\EndFor 
		\EndFor
		\State result $\gets \frac{1}{n} \sum_{i=1}^n \sum_{t=0}^H \gamma^t r_{i, t}$ 
		\State \Return \DONE, result
	\end{algorithmic}
\end{algorithm}

In each iteration $k$, a Monte-Carlo estimation procedure (\ConfidentRollout, \cref{alg:confident rollout ma}) is launched for every element $z \in \cC$ in the core set. 
An estimate (result in line 14) is obtained via taking the average return of $n$ Monte-Carlo rollouts of length $H$ while following policy $\pi_{k-1}$.
\ConfidentRollout~is \emph{successful} if it returns a status of~\DONE~and an estimate of $Q_{\pi_{k-1}}(z_s, z_{a})$, which is assigned to $z_q$.
If at iteration $k$ \ConfidentRollout~is successful for every core set element then $z_q$ has a value estimate for all $z \in \cC$, and the iteration is completed with the policy improvement step. 
The way the core set is constructed guarantees that the features of all the elements in the core set are sufficiently different to provide good target values $q_\cC$ for least squares (\cref{prop: approx value function bound for DAV,prop: approx value function bound for EGSS}).\looseness=-1



Each time when \ConfidentRollout~is \emph{unsuccessful}, it returns a status of \UNCERTAIN~and a corresponding tuple. The uncertain tuple is added to the core set and policy iteration is restarted (line 14) and the value estimates for all the core set elements are reset to \NONE~(line 9).
Roughly speaking, a tuple is flagged as uncertain when during the rollout a features is observed that is sufficiently different from all the features in the core set $\{ z_\phi: z \in \cC \}$.
Important is that adding tuples to the core set in this way ensures that the size of the core set is bounded by a $\oO(d)$ (\cref{lemma:bound on core set size}). 
Restarting policy iteration is mainly to simplify the analysis; in practice it is reasonable to continue with the same policy.\looseness=-1


It remains to specify the \UncertaintyCheck~subroutine that is used in \cref{alg:confident ma mc-lspi,alg:confident rollout ma}.
For a fixed state $s \in \cS$ the purpose of the uncertainty check is to search for an \emph{uncertain action} that satisfies 
\begin{align}
\phi(s, a)^\top (\Phi_\cC^\top \Phi_\cC + \lambda I)^{-1} \phi(s, a) > \tau\label{eq:uncertainty-check}
\end{align}
Here $\Phi_\cC \in \RR^{|\cC| \times d}$ is a matrix of all the features from the tuples in the core set stacked vertically. Solving \cref{eq:uncertainty-check} \emph{exactly} recovers the approach by \citet{yin2021efficient}. However, as this amounts to solving a positive-definite maximization problem, this is infeasible in general.

\subsection{Efficient Good Set Search (EGSS, \cref{alg:uncertainty check egss})}
Next, we show how to efficiently approximate the uncertainty check in \cref{eq:uncertainty-check}.
Define $V_\cC = \Phi_\cC^\top \Phi_\cC + \lambda I$ and a weighted matrix norm as $\|x\|_{B}^2 = x^\top B x, \ x \in \RR^d, B \in \RR^{d \times d}$.
Using this notation, \cref{eq:uncertainty-check} becomes
\begin{align*}
\phi(s, a)^\top (\Phi_\cC^\top \Phi_\cC + \lambda I)^{-1} \phi(s, a) = \|\phi(s, a)\|_{V_\cC^{-1}}^2\,> \tau.
\end{align*}
We define the \textit{good set} to be the set of all features with $\| . \|_{V_\cC^{-1}}^2$ weighted norm less than or equal to $\tau$ as follows
\begin{align*}
    \mathcal{D} = \{\phi(s, a): \|\phi(s, a)\|_{V_\cC^{-1}}^2 \le \tau\}.
\end{align*}
Fix a state $s \in \cS$.
We want to check if there exists an action outside of the good set (i.e. $a \in \cA$ that satisfies $\|\phi(s, a)\|_{V_\cC^{-1}}^2 > \tau$) with computation that does not depend on $|\cA|$.
To this end, let $L L^\T  = V_\cC^{-1}$ be a Cholesky decomposition of $V_{\cC}^{-1}$ and define  
 $\amax = \argmax_{a \in \cA} \|L^\T \phi(s, a)\|_\infty$. Note that $\amax$ satisfies the following norm inequalities:
$$\frac{1}{d}\|\phi(s, \amax)\|_{V_\cC^{-1}}^2 \le \|L^\T \phi(s, \amax)\|_\infty^2 \le \|\phi(s, \amax)\|_{V_\cC^{-1}}^2$$
 In other words, if $\|L^\T \phi(s, \amax)\|_\infty^2 > \tau$ holds, then  we have $\|\phi(s, \amax)\|_{V_\cC^{-1}}^2 > \tau$ and we have found an uncertain state-action pair.
At the same time if $\|L^\T \phi(s, \amax)\|_\infty^2 \le \tau$ then we are sure that $\|\phi(s, a)\|_{V_\cC^{-1}}^2 \le d \tau$ for all $a \in \cA$.
The fact that the last inequality is still sufficient to provide bounds on the sub-optimality of policy evaluation manifests in Proposition \ref{prop: approx value function bound for EGSS}, where only an extra factor of $\sqrt{d}$ is introduced.
Finally, notice that 
\begin{align}
&\max_{a \in \cA} \|L^\T \phi(s, a)\|_\infty = \max_{v \in \{\pm e_i\}_{i=1}^d} \max_{a \in \cA} \langle Lv, \phi(s, a)\rangle\label{eq:infty-norm-oracle}
\end{align}
can be computed efficiently using $2d$ calls to the greedy oracle (\cref{ass: argmax oracle}).

%

\begin{algorithm}[t]
	\caption{\UncertaintyCheckEGSS} \label{alg:uncertainty check egss}  
	\begin{algorithmic}[1]
		\State \textbf{Input:} state $s$, core set $\cC$, threshold $\tau$
		\State $L \gets \text{Cholesky}((\Phi_\cC^\top \Phi_\cC + \lambda I)^{-1})$ 
		\For {$v \in \{\pm e_l\}_{l = 1}^d$}
		\State $\amax \gets \argmax_{a \in \cA} \phi(s, a)^\top L v $
		\If {$\left(\phi(s, \amax)^\top L v\right)^2 > \tau$}
		\State result $\gets (s, \amax, \phi(s, \amax), \text{\NONE})$
		\State \Return {\UNCERTAIN, result}
		\EndIf
		\EndFor 
		\State \Return \CERTAIN, \NONE 
	\end{algorithmic}
\end{algorithm}

\subsection{Theoretical Guarantees}\label{ss:mc-lspi-egss-theory}
The result that 
characterizes the performance of
\MCLSPI~combined with \UncertaintyCheckEGSS~is summarized in the next theorem.
\begin{theorem}[\textsc{Confident MC-LSPI EGSS} Sub-Optimality] \label{thm:mc-lspi-egss sub-optimality}	
        Suppose \cref{asm:linear-q-pi,ass: bounded features,ass: argmax oracle} hold.
	If $\epsilon = 0$, for any $\kappa > 0$, with probability at least $1 - \delta$, the final policy $\pi_{K-1}$, returned by \MCLSPI~combined with \UncertaintyCheckEGSS~satisfies
	\begin{equation*}
		V^*(\rho) - V_{\pi_{K-1}}(\rho) \leq \kappa.
	\end{equation*}
	  The query and computation complexity are $\cO\big(\tfrac{d^4}{\kappa^2 (1-\gamma)^8} \big)$ and $\poly(d, \frac{1}{1 - \gamma}, \frac{1}{\kappa}, \log(\frac{1}{\delta}))$ respectively.
	  If $\epsilon > 0$, then with probability at least $1 - \delta$, the policy $\pi_{K-1}$, output satisfies
	\begin{equation*}
		V^*(\rho) - V_{\pi_{K-1}}(\rho) \leq \tfrac{64 \epsilon d}{(1-\gamma)^2} (1 +\log(1+b^2 \epsilon^{-2} d^{-1}))^{1/2}.
	\end{equation*}
	  The query and computation complexity are $\cO\big(\tfrac{d^2}{\epsilon^2 (1-\gamma)^4} \big)$ and $\poly(d, \frac{1}{1 - \gamma}, \frac{1}{\epsilon}, \log(\frac{1}{\delta}), \log(1+b))$, respectively.
       All parameter settings are in \cref{app: theorem proofs}.
\end{theorem}
When compared to the result in \citet[Theorem 5.1]{yin2021efficient} we have an extra factor of $d$ in the query complexity for $\epsilon = 0$, while for $\epsilon \neq 0$ we only have an extra factor of $\sqrt{d}$ in the sub-optimality of the output policy.
This is similar to linear bandits, where an extra $\sqrt{d}$ is suffered in the regret for oracle-efficient methods \citep{dani2008stochastic,agrawal2013thompson,abeille2017linear}. \looseness=-1

The full proof is given in \cref{app: theorem proofs}. 
The proof essentially follows the ideas in \citet{yin2021efficient} while carefully arguing how \UncertaintyCheckEGSS~affects the query complexity. 
For the computational complexity, note that \UncertaintyCheckEGSS~can be implemented in $\poly(d)$ by \cref{eq:infty-norm-oracle}, and linear algebra operations. 
Since the core set size is bounded (\cref{lemma:bound on core set size}), policy iteration only restarts $\cO(d)$ times.
Lastly, the policy improvement step is trivially implemented using the greedy oracle (\cref{ass: argmax oracle}). \looseness=-1


\newcommand{\UncertaintyCheckDAV}{\textsc{UncertaintyCheck-DAV}}
\newcommand{\UncertaintyCheckKDAV}{\textsc{UncertaintyCheck-K-DAV}}
\newcommand{\MCKern}{\textsc{Confident Kernel MC-LSPI/Politex}}
\newcommand{\MCLSPIKern}{\textsc{Confident Kernel MC-LSPI}}
\newcommand{\MCPolitexKern}{\textsc{Confident Kernel MC-Politex}}
\newcommand{\MCPolitex}{\textsc{Confident MC-Politex}}

\section{\uppercase{Additive Q-Functions}}\label{sec:additive}

The result in \cref{ss:mc-lspi-egss-theory} makes no restriction on the choice of features as long as the greedy policy can be computed efficiently (\cref{asm:linear-q-pi,ass: argmax oracle}).
Next, we introduce an additive feature model for which the oracle can be implemented efficiently.

With the greedy oracle (\cref{ass: argmax oracle}), one can use \MCLSPI~combined with \UncertaintyCheckEGSS~and directly invoke \cref{thm:mc-lspi-egss sub-optimality}.
However, in \cref{ss:dav} we introduce a new uncertainty check algorithm, \UncertaintyCheckDAV, that explicitly uses the additive structure.
The additive feature structure leads to improved results in the regimes where the dimension is large, but more importantly facilitates an efficient kernelized version of the \MCLSPI~algorithm (\cref{ss:kernel}). 
The additive model also allows an efficient implementation of \MCPolitex~\citep{yin2021efficient}, which leads to an improved dependence on the suboptimality in the misspecified setting (\cref{ss:politex}).

In the following, we assume that the action space can be decomposed into a product $\cA = \cA^\allm := \cA^{(1)} \times \cdots \times \cA^{(m)}$ for $m \geq 1$ (borrowing the standard notation from the multi-agent setting).
We further assume access to feature maps $\phi_i: \cS \times \cA^{(i)} \to \RR^d$ for each $i \in [m]$ and define $\phi(s, a^\allm) = \sum_{i=1}^m \phi_i(s, a^\ag{i})$.
The next assumption states that for any policy $\pi$, the $Q_\pi$-function is (approximately) linear in the feature map $\phi$ and decomposes additively across the components $\cA^{(i)}$. 
\begin{assumption}
	\label{ass: feature decomposition}
	For each policy $\pi$ there exists a weight vector $w_\pi \in \RR^d, \|w_\pi\|_2 \leq b$ satisfying \looseness=-1
 $\max\limits_{(s, a^\allm) \in \cS \times \cA} |Q_\pi(s,a^\allm) -w_\pi^\T  \sum_{j=1}^m \phi_j(s,a^{(j)})| \le \epsilon$.
\end{assumption}
In the context of the multi-agent setting (\cref{ex:multi-agent}), the interpretation is that each $\phi_i(s,a^\ag{i})$ models the contribution to the $Q$-function of each agent individually.
Moreover, when \cref{ass: feature decomposition} is satisfied, then for any weight vector $w \in \RR^d$ the greedy policy can be implemented with $\cO(d\sum_{i=1}^m|\cA^\ag{i}|)$ computation: \looseness=-1
\begin{align*}
	&\argmax\nolimits_{a^\allm \in \cA} w^\T \phi(s, a^\allm)
	\\
	&=\big(\argmax_{a^{(1)} \in \cA^{(1)}} w^\T \phi_1(s, a^{(1)}), ..., \argmax_{a^{(m)} \in \cA^{(m)}} w^\T \phi_m(s, a^{(m)})\big)
\end{align*}

A simple example when \cref{ass: feature decomposition} holds is when $m$ agents ``live'' in $m$ separate MDPs such that in each MDP the action-value functions are linearly realizable with their respective feature-maps and the goal is to maximize the sum of the rewards across the MDPs.
In cases like this, we say that the ``large'' MDP is a \emph{product MDP}.
Note that in this setting agents only observe a joint reward after taking their actions, so an optimal policy for the joint MDP may not always be learned by simply applying single agent algorithms in each individual MDP. 
In \cref{app:additive mdp example} we show that \cref{ass: feature decomposition} also captures MDPs that require cooperation between agents, and provide some empirical results.\looseness=-1

\subsection{Uncertainty Check using a Default Action Vector} \label{ss:dav}

In this section we introduce the uncertainty check with a default action vector (\UncertaintyCheckDAV, \cref{alg:uncertainty check dav}).
The goal of the uncertainty check is to ultimately bound the estimation error of $w_k$, i.e. 
\begin{align}
	|w_k^\T \phi(s, a^\allm) - Q_{\pi_{k-1}}(s, a^\allm)| \le \eta \,.\label{eq:suboptimality}
\end{align}
Lemma \ref{lemma: yin lemma b.2} shows that a sufficient condition is to ensure that $\|\phi(s, a^\allm)\|_{V_\cC^{-1}}^2 \le \tau$ for all $(s, a^\allm) \in (\cS \times \cA)$ that are queried during policy evaluation.\looseness=-1

We show that under \cref{ass: feature decomposition}, it is possible to achieve \cref{eq:suboptimality} while running the uncertainty check for a much smaller set of actions of size $\sum_{i=1}^m |\cA^\ag{i}|$. 
Recall that \MCLSPI~sets a \emph{default action vector} $\bar a^\allm \in \cA$ as a global. 
Define a subset of $\cA$ as $\bar \cA^\allm = \{(a^\ag{i}, \bar a^\ag{-i}): a^\ag{i} \in \cA^\ag{i}, i \in [m]\}$, where we define $(a^\ag{i}, \bar a^\ag{-i}) =(\bar a^\ag{1},...,\bar a^\ag{i - 1}, a^\ag{i}, \bar a^\ag{i + 1},...,\bar a^\ag{m})$ as the action vector resulting from changing agent $i$'s default action in $\bar a^\allm$ with $a^\ag{i}$.
Then, by \cref{ass: feature decomposition} for any $a^\allm \in \cA$ we have
\begin{align}
	&w_k^\T \phi(s, a^\allm) \nonumber
	= w_k^\T \sum_{i=1}^m \phi_i(s, a^{\ag{i}}) \nonumber \\
	&= w_k^\T \Big(\sum_{i=1}^m \phi_i(s, a^\ag{i}) \pm (m-1) \phi_i(s, \bar a^\allm) \Big) \nonumber \\
	&= w_k^\T \Big(\sum_{i=1}^m \phi_i(s, (a^\ag{i}, \bar a^\ag{-i})) - (m-1) \phi_i(s, \bar a^\allm) \Big) \nonumber 
\end{align}
Notice that $\bar a^\allm, (a^\ag{i}, \bar a^\ag{-i}) \in \bar \cA^\allm, \forall i \in [m]$. 
Thus, when $\|\phi(s, \tilde a^\allm)\|_{V_\cC^{-1}}^2 \le \tau$ for all $\tilde a^\allm \in \bar \cA^\allm$ we can ensure that for all action-vectors $a^\allm \in \cA^\allm$ that $|w_k^\T \phi(s, a^\allm) - Q_{\pi_{k-1}}(s, a^\allm)| \le (2m-1)\eta$.
In words, by checking the uncertainty of action-vectors that differ from the default action vector by at most one position $\tilde a^\allm \in \bar \cA^\allm$ we can bound the sub-optimality of our estimate $w_k$, since the feature of any action vector can be related to the feature of the default action vector under Assumption \ref{ass: feature decomposition}. 
Since $\bar \cA^\allm$ only contains $\sum_{i=1}^m |\cA^\ag{i}|$ elements, this procedure is $\poly(d,\sum_{i=1}^m |\cA^\ag{i}|)$.\looseness=-1

\begin{algorithm}
	\caption{\UncertaintyCheckDAV} \label{alg:uncertainty check dav}  
	\begin{algorithmic}[1]
		\State \textbf{Input:} state $s$, core set $\cC$, threshold $\tau$.
		\State \textbf{Globals:} number of action components $m$ 
		\For {$j \in [m]$}
		\For {$a^{(j)} \in \cA^{(j)}$}
		\State $\tilde{a} \gets (a^{(j)}, \bar a^{(-j)})$
		\If {$\phi(s, \tilde{a})^\top V_\cC^{-1} \phi(s, \tilde a) > \tau$}
		\State result $\gets (s, \tilde a, \phi(s, \tilde a), \text{\NONE})$
		\State \Return {\UNCERTAIN, result}
		\EndIf
		\EndFor 
		\EndFor 
		\State \Return \CERTAIN, \NONE 
	\end{algorithmic}
\end{algorithm}

The result 
that characterizes the performance of
\MCLSPI~combined with \UncertaintyCheckDAV ~is summarized in the next theorem.
\begin{theorem}[\textsc{Confident MC-LSPI DAV} Sub-Optimality] \label{thm:mc-lspi-dav sub-optimality}	
        Suppose Assumption \ref{ass: feature decomposition}, and \ref{ass: bounded features} hold.
	If $\epsilon = 0$, for any $\kappa > 0$, with probability at least $1 - \delta$, the policy $\pi_{K-1}$, output by \MCLSPI~combined with \UncertaintyCheckDAV~satisfies
	\begin{equation*}
		V^*(\rho) - V_{\pi_{K-1}}(\rho) \leq \kappa.
	\end{equation*}
	  The query and computation complexity are $\cO\left(\tfrac{m^2d^3}{\kappa^2 (1-\gamma)^8} \right)$.
	and $\poly(\sum_{i=1}^m |\cA^\ag{i}|, d, \frac{1}{1 - \gamma}, \frac{1}{\kappa}, \log(\frac{1}{\delta}))$ respectively. 
	  If $\epsilon > 0$, then with probability at least $1 - \delta$, the output policy $\pi_{K-1}$ satisfies\looseness=-1
	\begin{equation*}
		V^*(\rho) - V_{\pi_{K-1}}(\rho) \leq \tfrac{128 \epsilon \sqrt{d} m}{(1-\gamma)^2} (1 +\log(1+ b^2 \epsilon^{-2} d^{-1}))^{1/2}.
	\end{equation*}
	  The query and computation complexity are $\cO\left(\tfrac{d^2}{\epsilon^2 (1-\gamma)^4} \right)$.
	  and $\poly(\sum_{i=1}^m |\cA^\ag{i}|, d, \frac{1}{1 - \gamma}, \frac{1}{\epsilon}, \log(\frac{1}{\delta}), \log(1+b))$ respectively.
       All parameter settings are in \cref{app: theorem proofs}.
\end{theorem}
When compared to the result in \citet[Theorem 5.1]{yin2021efficient} we have an extra factor of $m^2$ in the query complexity for $\epsilon = 0$, while for $\epsilon \neq 0$ we only have an extra factor of $m$ in the sub-optimality of the output policy. On the other hand, the computational complexity is improved from $\cO(\prod_{i=1}^m |\cA^\ag{i}|)$ for the prior work to $\cO(\sum_{i=1}^m |\cA^\ag{i}|)$.
When compared to \cref{thm:mc-lspi-egss sub-optimality} where \UncertaintyCheckEGSS~was used instead of \UncertaintyCheckDAV~the extra dependence on $\sqrt{d}$ changed to $m$.

\subsection{Kernelized Setting}\label{ss:kernel}
The kernelized setting is a standard extension of the finite-dimensional linear setup \citep{srinivas2009gaussian,abbasi2012online}. 
It lifts the restriction that the features and parameter vector are elements of $\bR^d$. 
Formally the kernel is $\kernel : (\sS \times \aA^\allm)^2 \rightarrow \bR$, which gives rise to a reproducing kernel Hilbert space (RKHS) $\cH$, defined as a vector space $V_\cH := \RR^{\cS \times \cA^\allm}$ with inner product $\langle \cdot, \cdot \rangle_\cH : V_\cH \times V_\cH \to \RR$.
We require that the $Q_\pi$-function is approximately contained in an RKHS. 
This includes cases where the linear dimension of function class is infinite.
\begin{assumption}[Kernel $Q_\pi$-realizability]
\label{ass:kernel q-pi}
For each policy $\pi$ there exists a vector $\tilde Q_\pi \in \cH, \|\tilde Q_\pi\|_\cH \le b$ that satisfies 
    $\sup_{s \in \cS, a^\allm \in \cA^\allm} |Q_\pi(s,a^\allm) - \tilde Q_\pi (s,a^\allm)| \le \epsilon$, where $\tilde Q_\pi(s,a^\allm) = \langle 
    \tilde Q_\pi, \kernel(s, a^\allm, \cdot, \cdot )\rangle_\cH.$
\end{assumption}
Similar to the finite setting we assume an additive structure (on the kernel now) to allow efficient implementation.
For component $j \in [m]$, define the kernel as $\kernel_j : (\sS \times \aA^\ag{j})^2 \rightarrow \bR$, which gives rise to an RKHS $\cH_j$, defined as a vector space $V_{\cH_j} := \RR^{\cS \times \cA^\ag{j}}$ with inner product $\langle \cdot, \cdot \rangle_\cH : V_{\cH_j} \times V_{\cH_j} \to \RR$. 
\begin{assumption}
\label{ass:kernel additive}
    The kernel $\kernel$ can be written as $\kernel(s_1,a_1^\allm, s_2,a_2^\allm) = \sum_{j=1}^m \kernel_j(s_1,a_1^{(j)}, s_2,a_2^{(j)})$ where $s_1, s_2 \in \cS, \ a_1^\allm, a_2^\allm \in \cA^\allm$.
\end{assumption}

The kernel setting  requires us to address two main challenges. 
First, the scaling of the query complexity with the dimension $d$ needs to be improved to a notion of effective dimension. 
Following \citet{du2021bilinear, huang2021short} we make use of the critical information gain $\tilde \Gamma$ (defined in \cref{eq:critical infogain}, \cref{app:kernel setting}) which can be bounded for different RKHS of interest \citep{srinivas2009gaussian,huang2021short}. 
Second, computationally we cannot directly work with infinite dimensional features $\phi(s,a) = \kernel(s, a, \cdot, \cdot)$. 
Instead, we rely on the `kernel trick' and compute all quantities of interest in the finite-dimensional data space \citep{scholkopf2001generalized}. 
After formally arguing as stated above, one can show that a kernelized version of \MCLSPI~and \UncertaintyCheckDAV~provide the following sub-optimality guarantees on the output policy (proof in \cref{app:kernel setting}).

\begin{theorem}[\textsc{Confident Kernel MC-LSPI DAV} Sub-Optimality] \label{thm:kernel mc-lspi-dav sub-optimality}	
        Suppose Assumption \ref{ass:kernel q-pi}, \ref{ass:kernel additive}, and \ref{ass: bounded features} hold.
        Define $\scritinf := \critinf$.
	If $\epsilon = 0$, for any $\kappa > 0$, with probability at least $1 - \delta$, the policy $\pi_{K-1}$, returned by \MCLSPIKern~(\cref{alg:confident kernel mc-lspi/politex}) combined with \UncertaintyCheckKDAV~(\cref{alg:uncertainty check k-dav}) satisfies
	\begin{equation*}
		V^*(\rho) - V_{\pi_{K-1}}(\rho) \leq \kappa.
	\end{equation*}
	The query and computation complexity are $\cO\left(\tfrac{m^2\scritinf^3}{\kappa^2 (1-\gamma)^8} \right)$ and $\poly(\sum_{i=1}^m |\cA^\ag{i}|, \scritinf, \frac{1}{1 - \gamma}, \frac{1}{\kappa}, \log(\frac{1}{\delta}))$ respectively. 
	  If $\epsilon > 0$, then with probability at least $1 - \delta$, the final policy $\pi_{K-1}$ satisfies
	\begin{equation*}
		V^*(\rho) - V_{\pi_{K-1}}(\rho) \leq \tfrac{32 \epsilon m \sqrt{\scritinf}}{(1-\gamma)^2}.
	\end{equation*}
	  The query and computation complexity are $\cO\left(\tfrac{\scritinf^2}{\epsilon^2 (1-\gamma)^4} \right)$
	  and $\poly(\sum_{i=1}^m |\cA^\ag{i}|, \scritinf, \frac{1}{1 - \gamma}, \frac{1}{\epsilon}, \log(\frac{1}{\delta}), \log(1+b))$ respectively.
       All parameter settings are in \cref{app: theorem proofs}.
\end{theorem}
The result is identical to \cref{thm:mc-lspi-dav sub-optimality} except with $d$ replaced with the critical information gain $\critinf$.

\subsection{Politex}\label{ss:politex}
The Politex algorithm has been shown to obtain better sub-optimality gaurantees than LSPI by \citet{abbasi2019politex}.
In this section we show that \MCPolitex~presented by \citet{yin2021efficient} can be extended to combinatorially large action spaces.
Although Politex is a also based on policy iteration, like LSPI, an important difference is that it uses stochastic policies based on an exponential weighting of each actions $Q$-value.
Efficiently sampling from such a policy is not always possible when the action space is combinatorially large.
We show \cref{ass: feature decomposition} is sufficient to do so (Proposition \ref{prop: efficient politex policy sampling}).
Moreover, using similar arguments as in \cref{ss:dav}, indeed, \MCPolitex~combined with \UncertaintyCheckDAV~achieves better sub-optimality guarantees than \MCLSPI.

\begin{theorem}[\textsc{Confident MC-Politex} Sub-Optimality] \label{thm:mc-politex sub-optimality}	
        Suppose Assumption \ref{ass: feature decomposition}, and \ref{ass: bounded features} hold.
	If $\epsilon > 0$, for any $\kappa > 0$, with probability at least $1 - \delta$, the policy $\bar \pi_{K-1}$, output by \MCPolitex~(\cref{alg:confident ma mc-politex}) combined with \UncertaintyCheckDAV~(\cref{alg:uncertainty check dav}) satisfies
	\begin{equation*}
		V^*(\rho) - V_{\bar \pi_{K-1}}(\rho) \leq \tfrac{64 \epsilon m \sqrt{d}}{1-\gamma} (1 +\log(1+b^2 \epsilon^{-2} d^{-1}))^{1/2}.
	\end{equation*}
	The query and computation complexity are $\cO\left(\tfrac{md}{\epsilon^4 (1-\gamma)^5} \right)$
	  and $\poly(\sum_{i=1}^m |\cA^\ag{i}|, d, \frac{1}{1 - \gamma}, \frac{1}{\epsilon}, \log(\frac{1}{\delta}), \log(1+b))$ respectively.
       All parameter settings are in \cref{app: theorem proofs}.
\end{theorem}

As expected the sub-optimality is better (scales with $1/(1-\gamma)$) than that of  \MCLSPI (\cref{thm:mc-lspi-dav sub-optimality}), which scales with $1/(1-\gamma)^2$.
However, the query complexity is worse (as is typical for Politex), and an extra factor of $m$ is introduced, since mirror descent needs to be run on the entire action space of size $\prod_{i=1}^m |\cA^\ag{i}|$ for each state. 
We also extend the result to the kernelized setting in \cref{app: theorem proofs},
and show that \UncertaintyCheckEGSS~can be used when \cref{ass: feature decomposition} is satisfied.

\section{\uppercase{Conclusion}} \label{sec:discussion}

In this work, we considered the problem of planning with a local access simulator when the action space is combinatorially large.
We introduced several algorithms that achieve polynomial computational and query complexity guarantees, while still maintaining a reasonable sub-optimality of the output policy under various assumptions. The main novelty is an efficient implementation of the uncertainty check under the mild assumption of having access to a greedy oracle. If the $Q$-functions for all policies satisfy an additive structure we provide nuanced results that show how the sample complexity can be improved in the regime where the dimension is large. Under the same additive structure our results also extend to the kernelized setting. An interesting direction for future work is to extend the results to the Factored MDP model \citep{guestrin2001multiagent} or the Confident LSVI algorithm \citep{hao2022confident}. 



 \subsubsection*{Acknowledgements}
 Johannes Kirschner gratefully acknowledges funding from the SNSF Early Postdoc.Mobility fellowship P2EZP2\_199781.
 Matej Jusup gratefully acknowledges support by the Swiss National Science Foundation under the research project DADA/181210.
Csaba Szepesv\'ari gratefully acknowledges funding  from the Canada CIFAR AI Chairs Program, Amii and NSERC.

\bibliographystyle{unsrtnat}
\bibliography{references}

\begin{thebibliography}{83}
\providecommand{\natexlab}[1]{#1}
\providecommand{\url}[1]{\texttt{#1}}
\expandafter\ifx\csname urlstyle\endcsname\relax
  \providecommand{\doi}[1]{doi: #1}\else
  \providecommand{\doi}{doi: \begingroup \urlstyle{rm}\Url}\fi

\bibitem[Mnih et~al.(2013)Mnih, Kavukcuoglu, Silver, Graves, Antonoglou,
  Wierstra, and Riedmiller]{mnih2013playing}
Volodymyr Mnih, Koray Kavukcuoglu, David Silver, Alex Graves, Ioannis
  Antonoglou, Daan Wierstra, and Martin Riedmiller.
\newblock Playing atari with deep reinforcement learning.
\newblock \emph{arXiv preprint arXiv:1312.5602}, 2013.

\bibitem[Degrave et~al.(2022)Degrave, Felici, Buchli, Neunert, Tracey,
  Carpanese, Ewalds, Hafner, Abdolmaleki, de~Las~Casas,
  et~al.]{degrave2022magnetic}
Jonas Degrave, Federico Felici, Jonas Buchli, Michael Neunert, Brendan Tracey,
  Francesco Carpanese, Timo Ewalds, Roland Hafner, Abbas Abdolmaleki, Diego
  de~Las~Casas, et~al.
\newblock Magnetic control of tokamak plasmas through deep reinforcement
  learning.
\newblock \emph{Nature}, 602\penalty0 (7897):\penalty0 414--419, 2022.

\bibitem[Akkaya et~al.(2019)Akkaya, Andrychowicz, Chociej, Litwin, McGrew,
  Petron, Paino, Plappert, Powell, Ribas, et~al.]{akkaya2019solving}
Ilge Akkaya, Marcin Andrychowicz, Maciek Chociej, Mateusz Litwin, Bob McGrew,
  Arthur Petron, Alex Paino, Matthias Plappert, Glenn Powell, Raphael Ribas,
  et~al.
\newblock Solving rubik's cube with a robot hand.
\newblock \emph{arXiv preprint arXiv:1910.07113}, 2019.

\bibitem[Yin et~al.(2021)Yin, Hao, Abbasi-Yadkori, Lazi{\'c}, and
  Szepesv{\'a}ri]{yin2021efficient}
Dong Yin, Botao Hao, Yasin Abbasi-Yadkori, Nevena Lazi{\'c}, and Csaba
  Szepesv{\'a}ri.
\newblock Efficient local planning with linear function approximation.
\newblock \emph{arXiv preprint arXiv:2108.05533}, 2021.

\bibitem[Bertsekas and Ioffe(1996)]{bertsekas1996temporal}
Dimitri~P Bertsekas and Sergey Ioffe.
\newblock Temporal differences-based policy iteration and applications in
  neuro-dynamic programming.
\newblock \emph{Lab. for Info. and Decision Systems Report LIDS-P-2349, MIT,
  Cambridge, MA}, 14, 1996.

\bibitem[Lagoudakis and Parr(2003)]{lagoudakis2003least}
Michail~G Lagoudakis and Ronald Parr.
\newblock Least-squares policy iteration.
\newblock \emph{The Journal of Machine Learning Research}, 4:\penalty0
  1107--1149, 2003.

\bibitem[Munos(2005)]{munos2005error}
R{\'e}mi Munos.
\newblock Error bounds for approximate value iteration.
\newblock In \emph{Proceedings of the National Conference on Artificial
  Intelligence}, volume~20, page 1006. Menlo Park, CA; Cambridge, MA; London;
  AAAI Press; MIT Press; 1999, 2005.

\bibitem[Lattimore et~al.(2020)Lattimore, Szepesvari, and
  Weisz]{lattimore2020learning}
Tor Lattimore, Csaba Szepesvari, and Gellert Weisz.
\newblock Learning with good feature representations in bandits and in rl with
  a generative model.
\newblock In \emph{International Conference on Machine Learning}, pages
  5662--5670. PMLR, 2020.

\bibitem[Weisz et~al.(2021)Weisz, Amortila, and
  Szepesv{\'a}ri]{weisz2021exponential}
Gell{\'e}rt Weisz, Philip Amortila, and Csaba Szepesv{\'a}ri.
\newblock Exponential lower bounds for planning in mdps with
  linearly-realizable optimal action-value functions.
\newblock In \emph{Algorithmic Learning Theory}, pages 1237--1264. PMLR, 2021.

\bibitem[Baker et~al.(2019)Baker, Kanitscheider, Markov, Wu, Powell, McGrew,
  and Mordatch]{baker2019emergent}
Bowen Baker, Ingmar Kanitscheider, Todor Markov, Yi~Wu, Glenn Powell, Bob
  McGrew, and Igor Mordatch.
\newblock Emergent tool use from multi-agent autocurricula.
\newblock \emph{arXiv preprint arXiv:1909.07528}, 2019.

\bibitem[Chu et~al.(2019)Chu, Wang, Codec{\`a}, and Li]{chu2019multi}
Tianshu Chu, Jie Wang, Lara Codec{\`a}, and Zhaojian Li.
\newblock Multi-agent deep reinforcement learning for large-scale traffic
  signal control.
\newblock \emph{IEEE Transactions on Intelligent Transportation Systems},
  21\penalty0 (3):\penalty0 1086--1095, 2019.

\bibitem[Chen et~al.(2021{\natexlab{a}})Chen, Chen, Li, Chu, Yao, Qiu, and
  Lin]{chen2021powernet}
Dong Chen, Kaian Chen, Zhaojian Li, Tianshu Chu, Rui Yao, Feng Qiu, and
  Kaixiang Lin.
\newblock Powernet: Multi-agent deep reinforcement learning for scalable
  powergrid control.
\newblock \emph{IEEE Transactions on Power Systems}, 37\penalty0 (2):\penalty0
  1007--1017, 2021{\natexlab{a}}.

\bibitem[Sutton and Barto(2018)]{sutton2018reinforcement}
Richard~S Sutton and Andrew~G Barto.
\newblock \emph{Reinforcement learning: An introduction}.
\newblock MIT press, 2018.

\bibitem[Szepesv{\'a}ri(2010)]{szepesvari2010algorithms}
Csaba Szepesv{\'a}ri.
\newblock Algorithms for reinforcement learning.
\newblock \emph{Synthesis lectures on artificial intelligence and machine
  learning}, 4\penalty0 (1):\penalty0 1--103, 2010.

\bibitem[Bellman(1957)]{bellman1957markovian}
Richard Bellman.
\newblock A markovian decision process.
\newblock \emph{Journal of mathematics and mechanics}, pages 679--684, 1957.

\bibitem[Howard(1960)]{howard1960dynamic}
Ronald~A Howard.
\newblock Dynamic programming and markov processes.
\newblock 1960.

\bibitem[Munos(2003)]{munos2003error}
R{\'e}mi Munos.
\newblock Error bounds for approximate policy iteration.
\newblock In \emph{ICML}, volume~3, pages 560--567, 2003.

\bibitem[Farahmand et~al.(2010)Farahmand, Szepesv{\'a}ri, and
  Munos]{farahmand2010error}
Amir-massoud Farahmand, Csaba Szepesv{\'a}ri, and R{\'e}mi Munos.
\newblock Error propagation for approximate policy and value iteration.
\newblock \emph{Advances in Neural Information Processing Systems}, 23, 2010.

\bibitem[Kakade(2003)]{kakade2003sample}
Sham~Machandranath Kakade.
\newblock \emph{On the sample complexity of reinforcement learning}.
\newblock University of London, University College London (United Kingdom),
  2003.

\bibitem[Azar et~al.(2012)Azar, Munos, and Kappen]{azar2012sample}
Mohammad~Gheshlaghi Azar, R{\'e}mi Munos, and Bert Kappen.
\newblock On the sample complexity of reinforcement learning with a generative
  model.
\newblock \emph{arXiv preprint arXiv:1206.6461}, 2012.

\bibitem[Gheshlaghi~Azar et~al.(2013)Gheshlaghi~Azar, Munos, and
  Kappen]{azar2013minimax}
Mohammad Gheshlaghi~Azar, R{\'e}mi Munos, and Hilbert~J Kappen.
\newblock Minimax pac bounds on the sample complexity of reinforcement learning
  with a generative model.
\newblock \emph{Machine learning}, 91\penalty0 (3):\penalty0 325--349, 2013.

\bibitem[Du et~al.(2019)Du, Kakade, Wang, and Yang]{du2019good}
Simon~S Du, Sham~M Kakade, Ruosong Wang, and Lin~F Yang.
\newblock Is a good representation sufficient for sample efficient
  reinforcement learning?
\newblock \emph{arXiv preprint arXiv:1910.03016}, 2019.

\bibitem[Agarwal et~al.(2020{\natexlab{a}})Agarwal, Jiang, Kakade, and
  Sun]{agarwal2020reinforcement}
Alekh Agarwal, Nan Jiang, Sham~M Kakade, and Wen Sun.
\newblock Reinforcement learning: Theory and algorithms.
\newblock \emph{URL https://rltheorybook. github. io}, 2020{\natexlab{a}}.

\bibitem[Abbasi-Yadkori et~al.(2019)Abbasi-Yadkori, Bartlett, Bhatia, Lazic,
  Szepesvari, and Weisz]{abbasi2019politex}
Yasin Abbasi-Yadkori, Peter Bartlett, Kush Bhatia, Nevena Lazic, Csaba
  Szepesvari, and Gell{\'e}rt Weisz.
\newblock Politex: Regret bounds for policy iteration using expert prediction.
\newblock In \emph{International Conference on Machine Learning}, pages
  3692--3702. PMLR, 2019.

\bibitem[Szepesvári(2022{\natexlab{a}})]{rltheory2022}
Csaba Szepesvári.
\newblock Lecture notes in reinforcement learning theory, Aug
  2022{\natexlab{a}}.
\newblock URL
  \url{https://rltheory.github.io/lecture-notes/planning-in-mdps/lec13/}.

\bibitem[Bradtke and Barto(1996)]{bradtke1996linear}
Steven~J Bradtke and Andrew~G Barto.
\newblock Linear least-squares algorithms for temporal difference learning.
\newblock \emph{Machine learning}, 22\penalty0 (1):\penalty0 33--57, 1996.

\bibitem[Melo and Ribeiro(2007)]{melo2007q}
Francisco~S Melo and M~Isabel Ribeiro.
\newblock Q-learning with linear function approximation.
\newblock In \emph{International Conference on Computational Learning Theory},
  pages 308--322. Springer, 2007.

\bibitem[Osband et~al.(2016)Osband, Van~Roy, and Wen]{osband2016generalization}
Ian Osband, Benjamin Van~Roy, and Zheng Wen.
\newblock Generalization and exploration via randomized value functions.
\newblock In \emph{International Conference on Machine Learning}, pages
  2377--2386. PMLR, 2016.

\bibitem[Yang et~al.(2020)Yang, Jin, Wang, Wang, and Jordan]{yang2020function}
Zhuoran Yang, Chi Jin, Zhaoran Wang, Mengdi Wang, and Michael~I Jordan.
\newblock On function approximation in reinforcement learning: optimism in the
  face of large state spaces.
\newblock In \emph{Proceedings of the 34th International Conference on Neural
  Information Processing Systems}, pages 13903--13916, 2020.

\bibitem[Ayoub et~al.(2020)Ayoub, Jia, Szepesvari, Wang, and
  Yang]{ayoub2020model}
Alex Ayoub, Zeyu Jia, Csaba Szepesvari, Mengdi Wang, and Lin Yang.
\newblock Model-based reinforcement learning with value-targeted regression.
\newblock In \emph{International Conference on Machine Learning}, pages
  463--474. PMLR, 2020.

\bibitem[Zanette et~al.(2020)Zanette, Lazaric, Kochenderfer, and
  Brunskill]{zanette2020learning}
Andrea Zanette, Alessandro Lazaric, Mykel Kochenderfer, and Emma Brunskill.
\newblock Learning near optimal policies with low inherent bellman error.
\newblock In \emph{International Conference on Machine Learning}, pages
  10978--10989. PMLR, 2020.

\bibitem[Du et~al.(2021)Du, Kakade, Lee, Lovett, Mahajan, Sun, and
  Wang]{du2021bilinear}
Simon Du, Sham Kakade, Jason Lee, Shachar Lovett, Gaurav Mahajan, Wen Sun, and
  Ruosong Wang.
\newblock Bilinear classes: A structural framework for provable generalization
  in rl.
\newblock In \emph{International Conference on Machine Learning}, pages
  2826--2836. PMLR, 2021.

\bibitem[Zhou et~al.(2021)Zhou, Gu, and Szepesvari]{zhou2021nearly}
Dongruo Zhou, Quanquan Gu, and Csaba Szepesvari.
\newblock Nearly minimax optimal reinforcement learning for linear mixture
  markov decision processes.
\newblock In \emph{Conference on Learning Theory}, pages 4532--4576. PMLR,
  2021.

\bibitem[Jin et~al.(2020)Jin, Yang, Wang, and Jordan]{jin2020provably}
Chi Jin, Zhuoran Yang, Zhaoran Wang, and Michael~I Jordan.
\newblock Provably efficient reinforcement learning with linear function
  approximation.
\newblock In \emph{Conference on Learning Theory}, pages 2137--2143. PMLR,
  2020.

\bibitem[Agarwal et~al.(2020{\natexlab{b}})Agarwal, Henaff, Kakade, and
  Sun]{agarwal2020pc}
Alekh Agarwal, Mikael Henaff, Sham Kakade, and Wen Sun.
\newblock Pc-pg: Policy cover directed exploration for provable policy gradient
  learning.
\newblock \emph{Advances in neural information processing systems},
  33:\penalty0 13399--13412, 2020{\natexlab{b}}.

\bibitem[Lazic et~al.(2021)Lazic, Yin, Abbasi-Yadkori, and
  Szepesvari]{lazic2021improved}
Nevena Lazic, Dong Yin, Yasin Abbasi-Yadkori, and Csaba Szepesvari.
\newblock Improved regret bound and experience replay in regularized policy
  iteration.
\newblock In \emph{International Conference on Machine Learning}, pages
  6032--6042. PMLR, 2021.

\bibitem[Wei et~al.(2021)Wei, Jahromi, Luo, and Jain]{wei2021learning}
Chen-Yu Wei, Mehdi~Jafarnia Jahromi, Haipeng Luo, and Rahul Jain.
\newblock Learning infinite-horizon average-reward mdps with linear function
  approximation.
\newblock In \emph{International Conference on Artificial Intelligence and
  Statistics}, pages 3007--3015. PMLR, 2021.

\bibitem[Hao et~al.(2022)Hao, Lazic, Yin, Abbasi-Yadkori, and
  Szepesvari]{hao2022confident}
Botao Hao, Nevena Lazic, Dong Yin, Yasin Abbasi-Yadkori, and Csaba Szepesvari.
\newblock Confident least square value iteration with local access to a
  simulator.
\newblock In \emph{International Conference on Artificial Intelligence and
  Statistics}, pages 2420--2435. PMLR, 2022.

\bibitem[Dean et~al.(1998)Dean, Givan, and Kim]{dean1998solving}
Thomas~L Dean, Robert Givan, and Kee-Eung Kim.
\newblock Solving stochastic planning problems with large state and action
  spaces.
\newblock In \emph{AIPS}, pages 102--110, 1998.

\bibitem[Gei{\ss}er et~al.(2020)Gei{\ss}er, Speck, and
  Keller]{geisser2020trial}
Florian Gei{\ss}er, David Speck, and Thomas Keller.
\newblock Trial-based heuristic tree search for mdps with factored action
  spaces.
\newblock In \emph{Thirteenth Annual Symposium on Combinatorial Search}, 2020.

\bibitem[Raghavan et~al.(2012)Raghavan, Joshi, Fern, Tadepalli, and
  Khardon]{raghavan2012planning}
Aswin Raghavan, Saket Joshi, Alan Fern, Prasad Tadepalli, and Roni Khardon.
\newblock Planning in factored action spaces with symbolic dynamic programming.
\newblock In \emph{Twenty-Sixth AAAI Conference on Artificial Intelligence},
  2012.

\bibitem[Osband and Van~Roy(2014)]{osband2014near}
Ian Osband and Benjamin Van~Roy.
\newblock Near-optimal reinforcement learning in factored mdps.
\newblock \emph{Advances in Neural Information Processing Systems}, 27, 2014.

\bibitem[Xu and Tewari(2020)]{xu2020near}
Ziping Xu and Ambuj Tewari.
\newblock Near-optimal reinforcement learning in factored mdps:
  Oracle-efficient algorithms for the non-episodic setting.
\newblock \emph{Advances in Neural Information Processing Systems}, 33, 2020.

\bibitem[Tian et~al.(2020)Tian, Qian, and Sra]{tian2020towards}
Yi~Tian, Jian Qian, and Suvrit Sra.
\newblock Towards minimax optimal reinforcement learning in factored markov
  decision processes.
\newblock \emph{Advances in Neural Information Processing Systems},
  33:\penalty0 19896--19907, 2020.

\bibitem[Chen et~al.(2020)Chen, Hu, Li, and Wang]{chen2020efficient}
Xiaoyu Chen, Jiachen Hu, Lihong Li, and Liwei Wang.
\newblock Efficient reinforcement learning in factored mdps with application to
  constrained rl.
\newblock \emph{arXiv preprint arXiv:2008.13319}, 2020.

\bibitem[Delarue et~al.(2020)Delarue, Anderson, and
  Tjandraatmadja]{delarue2020reinforcement}
Arthur Delarue, Ross Anderson, and Christian Tjandraatmadja.
\newblock Reinforcement learning with combinatorial actions: An application to
  vehicle routing.
\newblock \emph{Advances in Neural Information Processing Systems},
  33:\penalty0 609--620, 2020.

\bibitem[Hubert et~al.(2021)Hubert, Schrittwieser, Antonoglou, Barekatain,
  Schmitt, and Silver]{hubert2021learning}
Thomas Hubert, Julian Schrittwieser, Ioannis Antonoglou, Mohammadamin
  Barekatain, Simon Schmitt, and David Silver.
\newblock Learning and planning in complex action spaces.
\newblock In \emph{International Conference on Machine Learning}, pages
  4476--4486. PMLR, 2021.

\bibitem[Cesa-Bianchi and Lugosi(2012)]{cesa2012combinatorial}
Nicolo Cesa-Bianchi and G{\'a}bor Lugosi.
\newblock Combinatorial bandits.
\newblock \emph{Journal of Computer and System Sciences}, 78\penalty0
  (5):\penalty0 1404--1422, 2012.

\bibitem[Chen et~al.(2013)Chen, Wang, and Yuan]{chen2013combinatorial}
Wei Chen, Yajun Wang, and Yang Yuan.
\newblock Combinatorial multi-armed bandit: General framework and applications.
\newblock In \emph{International conference on machine learning}, pages
  151--159. PMLR, 2013.

\bibitem[Shleyfman et~al.(2014)Shleyfman, Komenda, and
  Domshlak]{shleyfman2014combinatorial}
Alexander Shleyfman, Anton{\'\i}n Komenda, and Carmel Domshlak.
\newblock On combinatorial actions and cmabs with linear side information.
\newblock In \emph{ECAI 2014}, pages 825--830. IOS Press, 2014.

\bibitem[Combes et~al.(2015)Combes, Talebi Mazraeh~Shahi, Proutiere,
  et~al.]{combes2015combinatorial}
Richard Combes, Mohammad~Sadegh Talebi Mazraeh~Shahi, Alexandre Proutiere,
  et~al.
\newblock Combinatorial bandits revisited.
\newblock \emph{Advances in neural information processing systems}, 28, 2015.

\bibitem[Jourdan et~al.(2021)Jourdan, Mutn{\`y}, Kirschner, and
  Krause]{jourdan2021efficient}
Marc Jourdan, Mojm{\'\i}r Mutn{\`y}, Johannes Kirschner, and Andreas Krause.
\newblock Efficient pure exploration for combinatorial bandits with semi-bandit
  feedback.
\newblock In \emph{Algorithmic Learning Theory}, pages 805--849. PMLR, 2021.

\bibitem[Kandasamy et~al.(2015)Kandasamy, Schneider, and
  P{\'o}czos]{kandasamy2015high}
Kirthevasan Kandasamy, Jeff Schneider, and Barnab{\'a}s P{\'o}czos.
\newblock High dimensional bayesian optimisation and bandits via additive
  models.
\newblock In \emph{International conference on machine learning}, pages
  295--304. PMLR, 2015.

\bibitem[Wang et~al.(2019)Wang, Wilder, Suen, Dilkina, and
  Tambe]{wang2019improving}
Kai Wang, Bryan Wilder, Sze-chuan Suen, Bistra Dilkina, and Milind Tambe.
\newblock Improving gp-ucb algorithm by harnessing decomposed feedback.
\newblock In \emph{Joint European Conference on Machine Learning and Knowledge
  Discovery in Databases}, pages 555--569. Springer, 2019.

\bibitem[Kirschner and Krause(2021)]{kirschner2021bias}
Johannes Kirschner and Andreas Krause.
\newblock Bias-robust bayesian optimization via dueling bandits.
\newblock In \emph{International Conference on Machine Learning}, pages
  5595--5605. PMLR, 2021.

\bibitem[Mutny and Krause(2018)]{mutny2018efficient}
Mojmir Mutny and Andreas Krause.
\newblock Efficient high dimensional bayesian optimization with additivity and
  quadrature fourier features.
\newblock \emph{Advances in Neural Information Processing Systems}, 31, 2018.

\bibitem[Rolland et~al.(2018)Rolland, Scarlett, Bogunovic, and
  Cevher]{rolland2018high}
Paul Rolland, Jonathan Scarlett, Ilija Bogunovic, and Volkan Cevher.
\newblock High-dimensional bayesian optimization via additive models with
  overlapping groups.
\newblock In \emph{International conference on artificial intelligence and
  statistics}, pages 298--307. PMLR, 2018.

\bibitem[Busoniu et~al.(2008)Busoniu, Babuska, and
  De~Schutter]{busoniu2008comprehensive}
Lucian Busoniu, Robert Babuska, and Bart De~Schutter.
\newblock A comprehensive survey of multiagent reinforcement learning.
\newblock \emph{IEEE Transactions on Systems, Man, and Cybernetics, Part C
  (Applications and Reviews)}, 38\penalty0 (2):\penalty0 156--172, 2008.

\bibitem[Zhang et~al.(2021)Zhang, Yang, and Ba{\c{s}}ar]{zhang2021multi}
Kaiqing Zhang, Zhuoran Yang, and Tamer Ba{\c{s}}ar.
\newblock Multi-agent reinforcement learning: A selective overview of theories
  and algorithms.
\newblock \emph{Handbook of Reinforcement Learning and Control}, pages
  321--384, 2021.

\bibitem[Shapley(1953)]{shapley1953stochastic}
Lloyd~S Shapley.
\newblock Stochastic games.
\newblock \emph{Proceedings of the national academy of sciences}, 39\penalty0
  (10):\penalty0 1095--1100, 1953.

\bibitem[Song et~al.(2021)Song, Mei, and Bai]{song2021can}
Ziang Song, Song Mei, and Yu~Bai.
\newblock When can we learn general-sum markov games with a large number of
  players sample-efficiently?
\newblock \emph{arXiv preprint arXiv:2110.04184}, 2021.

\bibitem[Tian et~al.(2021)Tian, Wang, Yu, and Sra]{tian2021online}
Yi~Tian, Yuanhao Wang, Tiancheng Yu, and Suvrit Sra.
\newblock Online learning in unknown markov games.
\newblock In \emph{International conference on machine learning}, pages
  10279--10288. PMLR, 2021.

\bibitem[Bai and Jin(2020)]{bai2020provable}
Yu~Bai and Chi Jin.
\newblock Provable self-play algorithms for competitive reinforcement learning.
\newblock In \emph{International conference on machine learning}, pages
  551--560. PMLR, 2020.

\bibitem[Liu et~al.(2021)Liu, Yu, Bai, and Jin]{liu2021sharp}
Qinghua Liu, Tiancheng Yu, Yu~Bai, and Chi Jin.
\newblock A sharp analysis of model-based reinforcement learning with
  self-play.
\newblock In \emph{International Conference on Machine Learning}, pages
  7001--7010. PMLR, 2021.

\bibitem[Leonardos et~al.(2021)Leonardos, Overman, Panageas, and
  Piliouras]{leonardos2021global}
Stefanos Leonardos, Will Overman, Ioannis Panageas, and Georgios Piliouras.
\newblock Global convergence of multi-agent policy gradient in markov potential
  games.
\newblock \emph{arXiv preprint arXiv:2106.01969}, 2021.

\bibitem[Jin et~al.(2021)Jin, Liu, Wang, and Yu]{jin2021v}
Chi Jin, Qinghua Liu, Yuanhao Wang, and Tiancheng Yu.
\newblock V-learning--a simple, efficient, decentralized algorithm for
  multiagent rl.
\newblock \emph{arXiv preprint arXiv:2110.14555}, 2021.

\bibitem[Huang et~al.(2021{\natexlab{a}})Huang, Lee, Wang, and
  Yang]{huang2021towards}
Baihe Huang, Jason~D Lee, Zhaoran Wang, and Zhuoran Yang.
\newblock Towards general function approximation in zero-sum markov games.
\newblock \emph{arXiv preprint arXiv:2107.14702}, 2021{\natexlab{a}}.

\bibitem[Chen et~al.(2021{\natexlab{b}})Chen, Zhou, and Gu]{chen2021almost}
Zixiang Chen, Dongruo Zhou, and Quanquan Gu.
\newblock Almost optimal algorithms for two-player markov games with linear
  function approximation.
\newblock \emph{arXiv preprint arXiv:2102.07404}, 2021{\natexlab{b}}.

\bibitem[Yang et~al.(2018)Yang, Luo, Li, Zhou, Zhang, and Wang]{yang2018mean}
Yaodong Yang, Rui Luo, Minne Li, Ming Zhou, Weinan Zhang, and Jun Wang.
\newblock Mean field multi-agent reinforcement learning.
\newblock In \emph{International conference on machine learning}, pages
  5571--5580. PMLR, 2018.

\bibitem[Pasztor et~al.(2021)Pasztor, Bogunovic, and
  Krause]{pasztor2021efficient}
Barna Pasztor, Ilija Bogunovic, and Andreas Krause.
\newblock Efficient model-based multi-agent mean-field reinforcement learning.
\newblock \emph{arXiv preprint arXiv:2107.04050}, 2021.

\bibitem[Guestrin et~al.(2001)Guestrin, Koller, and
  Parr]{guestrin2001multiagent}
Carlos Guestrin, Daphne Koller, and Ronald Parr.
\newblock Multiagent planning with factored mdps.
\newblock \emph{Advances in neural information processing systems}, 14, 2001.

\bibitem[Rashid et~al.(2018)Rashid, Samvelyan, Schroeder, Farquhar, Foerster,
  and Whiteson]{rashid2018qmix}
Tabish Rashid, Mikayel Samvelyan, Christian Schroeder, Gregory Farquhar, Jakob
  Foerster, and Shimon Whiteson.
\newblock Qmix: Monotonic value function factorisation for deep multi-agent
  reinforcement learning.
\newblock In \emph{International Conference on Machine Learning}, pages
  4295--4304. PMLR, 2018.

\bibitem[Zohar et~al.(2021)Zohar, Mannor, and Tennenholtz]{zohar2021locality}
Roy Zohar, Shie Mannor, and Guy Tennenholtz.
\newblock Locality matters: A scalable value decomposition approach for
  cooperative multi-agent reinforcement learning.
\newblock \emph{arXiv preprint arXiv:2109.10632}, 2021.

\bibitem[Bhattiprolu et~al.(2021)Bhattiprolu, Lee, and
  Naor]{bhattiprolu2021framework}
Vijay Bhattiprolu, Euiwoong Lee, and Assaf Naor.
\newblock A framework for quadratic form maximization over convex sets through
  nonconvex relaxations.
\newblock In \emph{Proceedings of the 53rd Annual ACM SIGACT Symposium on
  Theory of Computing}, pages 870--881, 2021.

\bibitem[Dani et~al.(2008)Dani, Hayes, and Kakade]{dani2008stochastic}
Varsha Dani, Thomas~P Hayes, and Sham~M Kakade.
\newblock Stochastic linear optimization under bandit feedback.
\newblock 2008.

\bibitem[Agrawal and Goyal(2013)]{agrawal2013thompson}
Shipra Agrawal and Navin Goyal.
\newblock Thompson sampling for contextual bandits with linear payoffs.
\newblock In \emph{International conference on machine learning}, pages
  127--135. PMLR, 2013.

\bibitem[Abeille and Lazaric(2017)]{abeille2017linear}
Marc Abeille and Alessandro Lazaric.
\newblock Linear thompson sampling revisited.
\newblock In \emph{Artificial Intelligence and Statistics}, pages 176--184.
  PMLR, 2017.

\bibitem[Srinivas et~al.(2009)Srinivas, Krause, Kakade, and
  Seeger]{srinivas2009gaussian}
Niranjan Srinivas, Andreas Krause, Sham~M Kakade, and Matthias Seeger.
\newblock Gaussian process optimization in the bandit setting: No regret and
  experimental design.
\newblock \emph{arXiv preprint arXiv:0912.3995}, 2009.

\bibitem[Abbasi-Yadkori(2012)]{abbasi2012online}
Yasin Abbasi-Yadkori.
\newblock \emph{Online Learning for Linearly Parametrized Control Problems}.
\newblock PhD thesis, Citeseer, 2012.

\bibitem[Huang et~al.(2021{\natexlab{b}})Huang, Kakade, Lee, and
  Lei]{huang2021short}
Kaixuan Huang, Sham~M Kakade, Jason~D Lee, and Qi~Lei.
\newblock A short note on the relationship of information gain and eluder
  dimension.
\newblock \emph{arXiv preprint arXiv:2107.02377}, 2021{\natexlab{b}}.

\bibitem[Sch{\"o}lkopf et~al.(2001)Sch{\"o}lkopf, Herbrich, and
  Smola]{scholkopf2001generalized}
Bernhard Sch{\"o}lkopf, Ralf Herbrich, and Alex~J Smola.
\newblock A generalized representer theorem.
\newblock In \emph{International conference on computational learning theory},
  pages 416--426. Springer, 2001.

\bibitem[Szepesvári(2022{\natexlab{b}})]{szepesvari2022}
Szepesvári.
\newblock Politex, Aug 2022{\natexlab{b}}.
\newblock URL
  \url{https://rltheory.github.io/lecture-notes/planning-in-mdps/lec14/}.

\bibitem[Vakili et~al.(2021)Vakili, Khezeli, and
  Picheny]{vakili2021information}
Sattar Vakili, Kia Khezeli, and Victor Picheny.
\newblock On information gain and regret bounds in gaussian process bandits.
\newblock In \emph{International Conference on Artificial Intelligence and
  Statistics}, pages 82--90. PMLR, 2021.

\end{thebibliography}

\appendix
\onecolumn

\section{\uppercase{Efficient Policy Sampling}} \label{app:efficienct policy sampling}
    The policy in \citep{yin2021efficient} for \MCLSPI~and \MCPolitex~is as follows
    \begin{equation}
    \pi_k(a|s) \gets 
        \begin{cases}
            \mathds{1}\left(a = \argmax\limits_{\tilde{a} \in \cA} w^\top \phi(s, \tilde{a})\right) & \text{LSPI} \\
            \exp\left(\alpha \sum\limits_{j=0}^{k-1} Q_j(s, a)\right) / \sum\limits_{a \in \cA} \exp\left(\alpha \sum\limits_{j=0}^{k-1} Q_k(s, a)\right) . & \text{Politex}
        \end{cases} \label{eq:yin policy}
    \end{equation}
    with $w_k = (\Phi_\cC^\top \Phi_\cC + \lambda I)^{-1} \Phi_\cC^\top q_\cC$ 
    and $Q_{k-1}(s, a) = \min\{\max\{w_k^\T \phi(s, a), 0\}, 1/(1-\gamma)\}$ for the Politex case only.
    In this section we show that the above policy can be sampled from efficiently if \cref{ass: feature decomposition}  or \cref{ass: argmax oracle} is satisfied for the LSPI case and policy $\pi_k$ can be sampled from efficiently if \cref{ass: feature decomposition} is satisfied for the Politex case.
    To be precise, by efficiently we mean with computation that does not depend on $|\cA|$.
    We assume only $w \in \RR^d$ or $w_0, ..., w_{k-1} \in \RR^d$ (for LSPI and Politex respectively) and a feature map $\phi: \cS \times \cA \to \RR^d$ are given, thus the process of sampling may require calculating the policy if necessary to accurately sample.
    First we handle the LSPI case.
    
    \begin{proposition}[Efficient LSPI Policy Sampling]
        \label{prop: efficient lspi policy sampling}
        Given state $s \in \cS$, parameter vector $w \in \RR^d$, feature map $\phi: \cS \times \cA \to \RR^d$ and assumption \cref{ass: feature decomposition} or \cref{ass: argmax oracle} satisfied. 
        Then policy
        $$\pi_k(a|s) = \mathds{1}\left(a = \argmax_{\tilde{a} \in \cA} w^\top \phi(s, \tilde{a})\right)$$
        can be sampled from in with computation that does not depend on $|\cA|$.
    \end{proposition}
    \begin{proof}
        One can sample from policy $\pi_k$ by simply outputting the result of $\argmax_{\tilde{a} \in \cA} w^\top \phi(s, \tilde{a})$. 
        Under assumption \cref{ass: argmax oracle} $\argmax_{\tilde{a} \in \cA} w^\top \phi(s, \tilde{a})$ can be computed in constant time by applying the oracle to $w$ and $\phi$ (i.e. $\cG(w, \phi)$).
        While, \cref{ass: feature decomposition} implies we can compute $\argmax_{\tilde{a} \in \cA} w^\top \phi(s, \tilde{a})$ in $\poly(\sumA, d)$ time, since 
        \begin{align*}
	&\argmax\nolimits_{a^\allm \in \cA^\allm} w^\T \phi(s, a^\allm)
	\\
	&=\big(\argmax_{a^{(1)} \in \cA^{(1)}} w^\T \phi_1(s, a^{(1)}), ..., \argmax_{a^{(m)} \in \cA^{(m)}} w^\T \phi_m(s, a^{(m)})\big)
\end{align*}
    \end{proof}

    Next, we handle the Politex case. 
    To achieve the result below we assume \cref{ass: feature decomposition} is satisfied.
    We have to modify the Politex policy in \cref{eq:yin policy} slightly, by removing the clipping of the $Q$-function at each iteration $k$ (i.e. we define the $Q$-function at iteration $k$ to be $Q_{k-1}(s, a) = w_k^\T \phi(s, a)$ instead of $Q_{k-1}(s, a) = \min\{\max\{w_k^\T \phi(s, a), 0\}, 1/(1-\gamma)\}$).
    This was done since we were not aware of an efficient way to compute the clipped $Q$-function for all action-vectors in $\cA^\allm$.
    Importantly, removing the clipping does not suffer any increase in the dominating terms of the final policies sub-optimality (shown in \cref{app: theorem proofs})
    \begin{proposition}[Efficient Politex Policy Sampling]
        \label{prop: efficient politex policy sampling}
        Given state $s \in \cS$, parameter vectors $w_0, ..., w_{k-1} \in \RR^d$, feature map $\phi: \cS \times \cA^\allm \to \RR^d$ and \cref{ass: feature decomposition} satisfied. 
        Then policy
        $$\pi_k(a^\allm|s) = \exp\left(\alpha \sum\nolimits_{j=0}^{k-1} w_{j}^\T \phi(s, a^\allm)\right)/ \sum\nolimits_{\tilde{a}^\allm \in \cA^\allm} \exp \left(\alpha \sum\nolimits_{j=0}^{k-1} w_{j}^\T \phi(s, \tilde{a}^\allm)\right)$$ 
        with $a^\allm \in \cA^\allm$ can be sampled from in time $\poly(\sum_{i=1}^m |\cA_i|, d)$.
    \end{proposition}

    \begin{proof}
    Fix arbitrary $a^\allm \in \cA^\allm$.
    To sample from $\pi_k$ it is sufficient to sample actions $a^\allm \in \cA^\allm$ proportional to $\exp(\alpha \sum_{j=0}^{k-1} Q_{j}(s, a^\allm))$.
    Rearranging $\exp(\alpha \sum_{j=0}^{k-1} Q_{j}(s, a^\allm))$ and plugging in that $\phi(s, a^\allm) = \sum_{i=1}^m \phi_i(s, a^\ag{i})$ under assumption \cref{ass: feature decomposition} we have
    \begin{align*}
        \exp\left(\alpha \sum\nolimits_{j=0}^{k-1} w_{j}^\T \phi(s, a^\allm)\right) 
        &= \prod_{j=0}^{k-1} \exp\left(\alpha w_{j}^\T \phi(s, a^\allm)\right) \\
        &= \prod_{j=0}^{k-1} \exp\left(\alpha w_{j}^\T \sum_{i=1}^m \phi_i(s, a^\ag{i})\right) \\
        &= \prod_{i=1}^m \prod_{j=0}^{k-1} \exp\left(\alpha w_{j}^\T \phi_i(s, a^\ag{i})\right)
    \end{align*}

    Which means that the probability of sampling action $a^\allm$ is equal to the product of the probabilities of sampling $a^\ag{i}$ for $i \in [m]$ independently.
    Since $a^\allm$ was arbitrary this completes the proof.
    \end{proof}

    \section{\uppercase{Bound on Core Set Size}}
    \citet{yin2021efficient} showed that when only tuples containing state-action vectors that satisfy $\phi(s, a)^\top (\Phi^\top \Phi + \lambda I)^{-1} \phi(s, a) > \tau$ are add to the core set then it can be bounded as follows.
    
    \begin{lemma}[Bound on Core Set Size (Lemme 5.1 in \citep{yin2021efficient})] \label{lemma:bound on core set size}
        When \cref{ass: bounded features} is satisfied, and $(s,  a) \in (\cS \times \cA)$ that satisfy $\phi(s, a)^\top (\Phi_\cC^\top \Phi_\cC + \lambda I)^{-1} \phi(s, a) > \tau$ are added to the core set, the size of the core set can be bounded by
        \begin{align}
             \tilde{C}_{\max} := \frac{e}{e-1} \frac{1 + \tau}{\tau} d \left( 
                \log(1 + \frac{1}{\tau}) +
                \log(1 + \frac{1}{\lambda})
            \right). \label{eq:cmax-new}       
        \end{align}
    \end{lemma}


    \section{\uppercase{Efficient Uncertainty Check}} \label{app:efficient uncertainty check}
    \label{app: efficient uncertainty check}
    The \MCLSPI~algorithm proposed by \citet{yin2021efficient} is the same as our \MCLSPI~(\cref{alg:confident ma mc-lspi}) algorithm combined with \UncertaintyCheck~(\cref{alg:uncertainty check}) and the policy on line 17 of \MCLSPI~replaced with \cref{eq:yin policy}.
    \begin{algorithm}
    \caption{\UncertaintyCheck} \label{alg:uncertainty check}  
    \begin{algorithmic}[1]
    \State \textbf{Input:} state $s$, core set $\cC$, threshold $\tau$
    \For {$a \in \cA$}
        \If {$\phi(s, a)^\top (\Phi_\cC^\top \Phi_\cC + \lambda I)^{-1} \phi(s, a) > \tau$}
            \State status $\gets$ \UNCERTAIN, result $\gets (s, a, \phi(s, a), \NONE)$
            \State \Return {status, result}
        \EndIf
    \EndFor 
    \State \Return \CERTAIN, \NONE 
    \end{algorithmic}
    \end{algorithm}
    Notice that \UncertaintyCheck~requires iterating over $\cA$ (line 2), which is computationally expensive with the action space is combinatorially large.
    In this appendix we show how the loop over all actions $a \in \cA$ in the \UncertaintyCheck~algorithm can be avoided when either \cref{ass: feature decomposition} or \cref{ass: argmax oracle} is satisfied.
    In particular, we show that \UncertaintyCheckDAV~and \UncertaintyCheckEGSS~algorithms are able to reduce the computation time of \UncertaintyCheck~to no longer depend on $|\cA|$, while still maintaining suitable output policy guarantees.

    Since, we are extending the \MCLSPI~algorithm proposed by \citet{yin2021efficient}, we will be borrowing much of the steps from their proof.
    \citet{yin2021efficient} used a \textit{virtual algorithm} (VA) and \textit{main algorithm} (MA) to prove the sub-optimality of their \MCLSPI~algorithm.
    We give a brief summary of the VA and MA; however, avoid full details since we use the exact same definition as in \citet{yin2021efficient}. 
    Until the next subsection, assume \UncertaintyCheck~is used in \MCLSPI~and \ConfidentRollout.
    The MA is exactly \MCLSPI. 
    The VA is based on the \MCLSPI~algorithm, but has some differences, which we outline next. 
    The VA runs for exactly $C_\text{max}$ loops, $K$ iterations, and completes all $n$ of its rollouts of length $H$. 
    For each loop and iteration $k$ the VA always obtains estimates $q_\cC$ of its policy.
    The VA uses a different policy than the MA for rollouts.
    We will first focus on the LSPI case and return to Politex much later.
    The VA's $Q$-function at iteration $k$ is  
    \begin{equation*}
        \tilde{Q}_{k-1}(s, a)=\begin{cases}
                          \tilde{w}_k^\top \phi(s, a) \quad &\text{if} \, \phi(s, a) \in \mathcal{D} \\
                          Q_{\tilde{\pi}_{k-1}}(s, a)     \quad &\text{if} \, \phi(s, a) \notin \mathcal{D} \\
                    \end{cases}
    \end{equation*}
    where $\tilde w_k = V_\cC^{-1} \Phi_\cC^\top \tilde q_\cC$, and $\tilde q_\cC$ are the estimates obtained from running \ConfidentRollout~on each element of the core set, and $\mathcal{D} = \{\phi(s, a): \|\phi(s, a)\|_{V_\cC^{-1}}^2 \le \tau\}$ is the \textit{good set}.
    The VA's policy is
    \begin{equation*}
        \tilde \pi_k(a | s) = \mathds{1} \left( a = \argmax_{\tilde a \in \cA} \tilde Q_{k-1}(s, \tilde a) \right).
    \end{equation*}

    The nice thing about defining the VA's policy in this way is that we can make use of the following Lemma from \citep{yin2021efficient}.
    
    \begin{lemma}[Lemma B.2 in \citep{yin2021efficient}]
    \label{lemma: yin lemma b.2}
    Suppose that Assumption \cref{ass: feature decomposition} holds. 
    With all terms as defined earlier and $\theta > 0$. 
    Then, with probability at least 
    $$1 - 2C_{\text{max}} \exp(-2 \theta^2(1-\gamma)^2 n)$$
    for any $(s,  a) \in (\cS \times \cA)$ pair such that $\phi(s, a) \in \mathcal{D}$, we have 
    $$|\tilde{Q}_{k-1} (s, a) - Q_{\tilde\pi_{k-1}} (s, a)| \le b\sqrt{\lambda \tau} + \left(\epsilon +  \frac{\gamma^{H-1}}{1 - \gamma} + \theta \right) \sqrt{\tau C_{\text{max}}} + \epsilon := \eta$$
    \end{lemma}

    Notice that for any $(s,  a) \in (\cS \times \cA)$ pair such that $\phi(s, a) \notin \mathcal{D}$, the VA's $Q$-function $\tilde Q_{k-1}$ has access to the true $Q$-function $Q_{\tilde{\pi}_{k-1}}$ of policy $\tilde \pi_{k-1}$.
    Thus, we have that 
    \begin{equation}
    \label{eq:va inf-norm bound}
        \| \tilde{Q}_{k-1} (s, a) - Q_{\tilde\pi_{k-1}} (s, a) \|_\infty \le \eta 
    \end{equation}
    Combined with the fact that $\tilde \pi_k$ is greedy w.r.t. $\tilde Q_{k-1}$ the above result turns out to be especially useful.
    
    To understand why, we state a classic policy improvement result, which can be found as Lemma B.3 in \citet{yin2021efficient} and in other papers.
    \begin{lemma}[approximate policy iteration]
    \label{lemma:approximate policy iteration}
        Suppose that we run K approximate policy iterations and generate a sequence of policies
        $\pi_0, \pi_1, \pi_2, \cdots, \pi_K$.
        Suppose that for every $k = 1, 2, \cdots, K$, in the k-th iteration, we obtain a function
        $\tilde{Q}_{k-1}$ such that, $\| \tilde{Q}_{k - 1} - Q_{\pi_{k - 1}} \|_\infty \leq \eta$,
        and choose $\pi_k$ to be greedy with respect to $\tilde{Q}_{k-1}$.
        Then
        \begin{align*}
            \| Q^* - Q_{\pi_K} \|_\infty \leq \frac{2 \eta}{1 - \gamma} + \frac{\gamma^K}{1 - \gamma},
        \end{align*}
    \end{lemma}

    In our case the VA's policy $\tilde \pi_k$ is greedy w.r.t. $\tilde Q_{k-1}$ and thus we have that
    \begin{align*}
        \| Q^* - Q_{\tilde \pi_K} \|_\infty \leq \frac{2 \eta}{1 - \gamma} + \frac{\gamma^K}{1 - \gamma},
    \end{align*}

    Now we explain how the MA can be related to the VA, and make use of the above result.
    The \UncertaintyCheck~algorithm can have two cases: 
    
    \textbf{Case 1:} $\|\phi(s, a)\|_{V_\cC^{-1}}^2 > \tau$ holds for at least one $a \in \cA$,
    
    \textbf{Case 2:} $\|\phi(s, a)\|_{V_\cC^{-1}}^2 \le \tau$ holds for all $a \in \cA$. This is equivalent to saying $\phi(s, a) \in \mathcal{D}, \ \forall a \in \cA$.

    The VA is exactly the same at the MA algorithm, until Case 1 occurs for the first time.
    This is because the MA's and VA's simulators are coupled, in the sense that at iteration $k$, rollout $i$, and step $t$, when both simulators are queried with the same state-action vector pairs, they sample the exact same next state and reward. 
    The VA also uses the same initial policy as the MA at the start of policy iteration for every loop.
    Once Case 1 occurs the MA would restart policy iteration (else condition in line 14 of \MCLSPI), while the VA does not. 
    The VA records the state-action vector pair when Case 1 occurs for the first time and adds it to the core set once it completes running policy iteration for the current loop.
    In this way the core set maintained by the MA and VA are always the same.
    Since the size of the core set is bounded by $C_\text{max}$ when $(s, a) \in (\cS \times \cA)$ that satisfy $\phi(s, a)^\top (\Phi_\cC^\top \Phi_\cC + \lambda I)^{-1} \phi(s, a) > \tau$ are added to the core set (\cref{lemma:bound on core set size}), there will be a loop of policy iteration at which the MA and VA never encounter Case 1 for any of the $K$ iterations of policy iteration.
    We call this loop the \emph{final loop}.
    This is equivalent to say that all $(s,  a) \in (\cS \times \cA)$ observed during all $K$ iterations of policy iteration in the final loop are in the good set (i.e. $\phi(s, a) \in \mathcal{D}$).
    Notice that this means MA and VA behaved identical in the final loop, since the VA's policy would have always been greedy w.r.t. $\tilde{w}_k^\top \phi$ and the MA and VA use the same initial policy at the start of each loop.
    It turns out this relationship between the MA and VA allows us to bound the sub-optimality of the MA in the final loop, by using the result in \cref{eq:va inf-norm bound} we have for the VA. 
    More precisely, the following result can be extracted from \citep{yin2021efficient}
    \begin{proposition}[equation (B.15) in \citet{yin2021efficient}] \label{prop:optimality of output policy}
    With all terms as defined earlier. 
    Define $\eta \ge \|\tilde Q_{k-1}(s, a) - Q_{\tilde \pi_{k-1}}(s, a)\|_\infty$. 
    Suppose $\eta \ge |\tilde w_{k}^\T \phi(\rho, a) - Q_{\tilde \pi_{k-1}}(\rho, a)|, \ \forall a \in \cA$.
    Then, if the VA and MA behave identically in the final loop, with probability at least $1 - 4KC_{\text{max}}^2 \exp(-2 \theta^2(1-\gamma)^2 n)$ we have
    \begin{align}
        V^*(\rho) - V_{\pi_{K-1}}(\rho) \leq \frac{8 \eta}{(1 - \gamma)^2} + \frac{2 \gamma^{K - 1}}{(1 - \gamma)^2} \label{eq: v function bound main}
    \end{align}
    \end{proposition}
    
    Notice, that we require three things to use the above result. 
    We need a bound on $\|\tilde Q_{k-1}(s, a) - Q_{\tilde \pi_{k-1}}(s, a)\|_\infty$. 
    We need a bound on $|\tilde w_{k}^\T \phi(\rho, a) - Q_{\tilde \pi_{k-1}}(\rho, a)|, \ \forall a \in \cA$.
    We need to ensure that the VA and MA behave identically in the final loop. 
    Then, we can get a bound on the sub-optimality of the MA's output policy $\pi_{K-1}$.
    An important observation is that \UncertaintyCheck~ensured that MA and VA behave identically in the final loop.
    It did this by making sure that the VA's policy $\tilde \pi_k$ would only be able to use $\tilde{w}_k^\top \phi$ to derive its actions, since \UncertaintyCheck~always returns a status of \textsc{certain} in the final loop, which means that $\phi(s, a) \in \mathcal{D}$ for all $s, a \in \cS \times \cA$ encountered in the final loop.
    With this information in mind, we now show that \UncertaintyCheckDAV~and \UncertaintyCheckEGSS~only requires computation independent of $|\cA|$, while providing only slightly worse sub-optimality guarantees when compared to the result in \citep{yin2021efficient}.


    \subsection{Efficient Good Set Search Approach (EGSS)} \label{subsec:good set search}
    
    In this section we prove some useful results for \UncertaintyCheckEGSS.
    Fix a state $s \in \cS$.
    First, we show that with computation independent of $|\cA|$, one can find an action vector $a \in \cA$ that approximately maximizes $\phi(s, a)^\top V_\cC^{-1} \phi(s, a)$.
    
    \begin{lemma}[Efficient good set search]
    \label{lemma:good set search}
    Assume either \cref{ass: argmax oracle} is satisfied.
    With all terms as defined earlier. 
    One can ensure, with $2 d$ calls to the greedy oracle that
    $$\phi(s, a)^\top V_\cC^{-1} \phi(s, a) \le d\tau$$
    for all $a \in \cA$, or there exists an $a \in \cA$ such that
    $$\phi(s, a)^\top V_\cC^{-1} \phi(s, a) > \tau.$$
    Further, if \cref{ass: feature decomposition} is satisfied, then the same guarantees hold with $2 d^2 \sumA$ computation time.
    \end{lemma}
    
    \begin{proof}
    Recall that we are able to compute $\max_{a \in \cA} \langle u, \phi(s, a) \rangle$ for any $u \in \RR^d$ in constant time if \cref{ass: feature decomposition} is satisfied, and in $d \sumA$ time if \cref{ass: argmax oracle} is satisfied.
    We make use of a bi-directional 2-norm to $\infty$-norm inequality that will take advantage of the above mentioned efficient computation.
    Fix $\cC$ and define the lower triangular matrix $L$ via the Cholesky decomposition $V_\cC^{-1} = L L^\top$.
    Define $\{e_i\}_{i=1}^d$ as the standard basis vectors and 
    \begin{equation*}
    (v^*, a_\text{max}) := \text{arg} \left(\max_{v \in \{\pm e_i\}_{i=1}^d} \max_{a \in \cA} \langle L v, \phi(s, a) \rangle \right) 
    \end{equation*}
    Then we have that
    \begin{align}
       \frac{1}{d} \| \phi(s, a_\text{max}) \|_{V_\cC^{-1}}^2
       &= \frac{1}{d} \phi(s, a_\text{max})^\top V_\cC^{-1} \phi(s, a_\text{max}) \nonumber \\
       &= \frac{1}{d} \phi(s, a_\text{max})^\top L L^\top \phi(s, a_\text{max}) \nonumber \\
       &= \frac{1}{d} \| L^\top \phi(s, a_\text{max}) \|_2^2 \nonumber \\
       &\le \max_{a \in \cA} \| L^\top \phi(s, a)\|_\infty^2 \nonumber \\
       &= \max_{v \in \{\pm e_i\}_{i=1}^d} \max_{a \in \cA} \langle v, L^\top \phi(s, a) \rangle^2 \nonumber \\
       &= \max_{v \in \{\pm e_i\}_{i=1}^d} \max_{a \in \cA} \langle L v, \phi(s, a) \rangle^2 \label{inf-norm term} \\
       &= \langle L v^*, \phi(s, a_\text{max}) \rangle^2  \nonumber  \\
       &\le \| L^\top \phi(s, a_\text{max}) \|_2^2 \nonumber 
    \end{align}
    The purpose of writing all the equalities up to \cref{inf-norm term} was to show that \cref{inf-norm term} can be computed efficiently. 
    This is since we are able to compute $\max_{a \in \cA} \langle L v, \phi(s, a) \rangle^2$ in constant time if \cref{ass: feature decomposition} is satisfied, and in $d \sumA$ time if \cref{ass: argmax oracle} is satisfied, and $\{\pm e_i\}_{i=1}^d$ contains $2d$ elements.
    Also, note that $L$ can be computed with at most $d^2$ computation in each loop by doing a rank one update to the Cholesky decomposition of $V_\cC^{-1} = L L^\top$.
    
    If equation (\cref{inf-norm term}) is larger than $\tau$, then $\|\phi(s, a_\text{max}) \|_{V^{-1}}^2 > \tau$.
    While, if equation (\cref{inf-norm term}) is less than or equal $\tau$, then $\|\phi(s, a_\text{max}) \|_{V^{-1}}^2 \le d\tau$, completing the proof.
    \end{proof}

    \UncertaintyCheckEGSS~is essentially an implementation of equation (\cref{inf-norm term}), thus its computation is independent of $|\cA|$, as stated in \cref{lemma:good set search}.
    Also, since only $a \in \cA$ that satisfy $\|\phi(s, a)\|_{V_\cC^{-1}}^2 \ge \|\phi(s, a)\|_\infty^2 > \tau$ are added to the core set, we can still use \cref{lemma:bound on core set size} to bound the size of the core set by $C_\text{max}$.
    Basically, \cref{inf-norm term} is an underestimate of $\| \phi(s, a_\text{max}) \|_{V^{-1}}^2$ and we only add elements to the core set when it is larger than $\tau$, thus the core set is no larger than it was when using \UncertaintyCheck.
    
    Now, we aim to ensure that the VA and MA behave identically in the final loop.
    Notice that \UncertaintyCheckEGSS~provides a weaker guarantee than \UncertaintyCheck, when the returned result is \CERTAIN.
    Specifically, when \UncertaintyCheckEGSS~returns a result of \CERTAIN, then \cref{lemma:good set search} guarantees that $\|\phi(s, a)\|_{V_\cC^{-1}}^2 \le d\tau$ for all $a \in \cA$.
    While when the \UncertaintyCheck~returns a result of \CERTAIN, then $\|\phi(s, a)\|_{V_\cC^{-1}}^2 \le \tau$ for all $a \in \cA$.
    Thus, we define a smaller good set $\mathcal{D}_d = \{ \phi(s, a): \|\phi(s, a)\|_{V_\cC^{-1}}^2 \le d\tau\}$.
    
    Redefine the VA's $Q$-function at iteration $k$ as
    \begin{equation*}
        \tilde{Q}_{k-1}(s, a)=\begin{cases}
                          \tilde{w}_k^\top \phi(s, a) \quad &\text{if} \, \phi(s, a) \in \mathcal{D}_d \\
                          Q_{\tilde{\pi}_{k-1}}(s, a)     \quad &\text{if} \, \phi(s, a) \notin \mathcal{D}_d \\
                    \end{cases}
    \end{equation*}
    and VA's policy as
    \begin{equation*}
        \tilde \pi_k(a | s) = \mathds{1} \left( a = \argmax_{\tilde a \in \cA} \tilde Q_{k-1}(s, \tilde a) \right).
    \end{equation*}
    Notice that in the final loop \UncertaintyCheckEGSS~always returns a \textsc{result} of \textsc{certain}, and thus we are sure that all $a \in \cA$ for all the states encountered in the final loop are in the smaller good set $\mathcal{D}_d$.
    Thus, the VA's policy $\pi_{k}$ would always be greedy w.r.t. $\tilde w_k^\top \phi$ in the final loop.
    This ensures that the VA and MA behave identically in the final loop.

    Next we need show that we can bound $\|\tilde Q_{k-1}(s, a) - Q_{\tilde \pi_{k-1}}(s, a)\|_\infty$ with this new definition of $\tilde Q_{k-1}$. 
    First we state a slight modification of \cref{lemma: yin lemma b.2} that holds for the smaller good set $\mathcal{D}_d$ 
    
    \begin{lemma}[EGSS modified Lemma B.2 from \cite{yin2021efficient}]
    \label{lemma: mod b.2 egss}
    Suppose that \cref{asm:linear-q-pi} holds. 
    With all terms as defined earlier and $\theta > 0$. 
    Then, with probability at least 
    $$1 - 2C_{\text{max}} \exp(-2 \theta^2(1-\gamma)^2 n)$$
    for any $(s,  a) \in (\cS \times \cA)$ pair such that $\phi(s, a) \in \mathcal{D}_d$, we have 
    $$|\tilde{w}_k^\top \phi(s, a) - w_{\tilde\pi_{k-1}}^\top \phi(s, a)| \le b\sqrt{\lambda d \tau} + \left(\epsilon + \frac{\gamma^{H+1}}{1 - \gamma} + \theta \right) \sqrt{d \tau C_{\text{max}}} + \epsilon = \sqrt{d} \bar \eta:= \eta_2$$
    \end{lemma}
    
    \begin{proof}
    The proof is identical to that of Lemme B.2 from \cite{yin2021efficient} except $\tau$ is replaced with $d \tau$ everywhere, due to the weaker guarantee of \UncertaintyCheckEGSS as discussed above. 
    \end{proof}
    
    Essentially we get an extra $\sqrt{d}$ factor due to the smaller good set $\mathcal{D}_d$. 
    Since the VA's policy $\tilde \pi_{k}$ has access to the true $Q$-function $Q_{\tilde \pi_{k-1}}$ for all $\phi(s, a) \notin \mathcal{D}_d$,
    we can show that $\|\tilde Q_{k-1}(s, a) - Q_{\tilde \pi_{k-1}}(s, a)\|_\infty$ can be bounded.
    
    \begin{proposition}[approximate value function bound for EGSS]
    \label{prop: approx value function bound for EGSS}
    Suppose that \cref{asm:linear-q-pi} holds. 
    With all terms as defined earlier and $\theta > 0$. 
    Then, with probability at least 
    $$1 - 2C_{\text{max}} \exp(-2 \theta^2(1-\gamma)^2 n)$$
    we have
    $$\|\tilde Q_{k-1}(s, a) - Q_{\tilde{\pi}_{k-1}}(s, a)\|_\infty \le \eta_2.$$ 
    \end{proposition}

    \begin{proof}
    For any $(s, a) \in (\cS \times \cA)$ such that $\phi(s, a) \in \mathcal{D}_d$, we have
    \begin{align}
        |\tilde Q_{k-1}(s, a) - Q_{\tilde{\pi}_{k-1}}(s, a)| \le \eta_2
    \end{align}
    by \cref{prop: approx value function bound for EGSS}.
    While for any $(s, a) \in (\cS \times \cA)$ such that $\phi(s, a) \notin \mathcal{D}_d$, we have
    \begin{align}
        |\tilde Q_{k-1}(s, a) - Q_{\tilde{\pi}_{k-1}}(s, a)| 
        = |Q_{\tilde \pi_{k-1}}(s, a) - Q_{\tilde{\pi}_{k-1}}(s, a)|
        &= 0 
    \end{align}
    \end{proof}

    Finally, it is left to show that $|\tilde w_{k}^\T \phi(\rho, a) - Q_{\tilde \pi_{k-1}}(\rho, a)|$  can be bounded for all $a \in \cA$.
    Notice that lines 4-8 in \MCLSPI~run \UncertaintyCheckEGSS~with state $\rho$ as input until the returned status is \CERTAIN.
    Recall that once \UncertaintyCheckEGSS~returns a status of \CERTAIN~we know that $\rho \in \mathcal{D}_d$.
    Thus, we can immediately apply \cref{lemma: mod b.2 egss} to bound $\eta_2 \ge |\tilde w_{k}^\T \phi(\rho, a) - Q_{\tilde \pi_{k-1}}(\rho, a)|, \ \forall a \in \cA$.

    \subsection{Default Action Vector (DAV) Method}
    In this section we prove some useful results for \UncertaintyCheckDAV.
    Fix a state $s \in \cS$.
    As mentioned in the body, assume the action space can be decomposed as a product $\cA^\allm = \cA^\ag{1} \times ... \times \cA^\ag{m}$ throughout this section.
    We call elements of $\cA^\allm$ \textit{action vectors}.
    First, \UncertaintyCheckDAV~only iterates over $\sumA$ action vectors instead of all the action vectors like \UncertaintyCheck~does.
    Define the $\sum_{i=1}^m A^\ag{i}$ sized set of modified default action vectors as $\bar \cA^\allm = \{ (a^\ag{i}, \bar a^{(-i)}): a^\ag{i} \in \cA^\ag{i}, \ i \in [m] \}$.
    Notice \UncertaintyCheckDAV~iterates over all the actions in the set $a^\allm \in \bar \cA^\allm$ and checks if any of them satisfy $\|\phi(s, a^\allm)\|_{V_\cC^{-1}}^2 > \tau$.
    This of course achieves the goal of compute independent of $|\cA^\allm|$, since there are only $\sum_{i=1}^m A^\ag{i}$ action vectors in $\bar \cA^\allm$ to iterate over now.
    Also, since only $a^\allm \in \cA^\allm$ that satisfy $\|\phi(s, a^\allm)\|_{V_\cC^{-1}}^2 > \tau$ are added to the core set, we can still use \cref{lemma:bound on core set size} to bound the size of the core set by $C_\text{max}$.

    Now, we aim to ensure that the VA and MA behave identically in the final loop.
    Define the set of states for which all the modified default action vectors are in the good set as $\bar \cS = \{ s \in \cS: \|\phi(s, a^\allm)\|_{V_\cC^{-1}}^2 \le \tau, \forall a^\allm \in \bar \cA^\allm \}$.
    Redefine the VA's $Q$-function as 
    \begin{equation*}
    \tilde Q_{k-1}(s, a^\allm)=\begin{cases}
          \tilde w_k^\top \phi(s, a^\allm) \quad & s \in \bar \cS \\
          Q_{\tilde{\pi}_{k-1}}(s, a^\allm). \quad & s \in \cS \backslash \bar \cS
        \end{cases}
    \end{equation*}
    The VA's policy is
    \begin{equation*}
        \tilde \pi_k(a^\allm | s) = \mathds{1} \left( a^\allm = \argmax_{\tilde a^\allm \in \cA^\allm} \tilde Q_{k-1}(s, \tilde a^\allm) \right).
    \end{equation*}
    Notice that in the final loop the check $\phi(s, (a^{(j)}, \bar a^{(-j)}))^\top (\Phi_\cC^\top \Phi_\cC + \lambda I)^{-1} \phi(s, (a^{(j)}, \bar a^{(-j)})) > \tau$ in \UncertaintyCheckDAV~never returns \textsc{True}, and thus we are sure that all $a^\allm \in \bar \cA^\allm$ for all the states encountered in the final loop are in the good set.
    Notice that these states that satisfy this condition are state in $\bar \cS$.
    Thus, the VA's policy $\pi_{k}$ would always be greedy w.r.t. $\tilde w_k^\top \phi$ in the final loop.
    This ensures that the VA and MA behave identically in the final loop.

    Now we show that we can bound $\|\tilde Q_{k-1}(s, a^\allm) - Q_{\tilde \pi_{k-1}}(s, a^\allm)\|_\infty$ with this new definition of $\tilde Q_{k-1}$. 
    First we state a slight modification of \cref{lemma: yin lemma b.2} for $w_{\tilde{\pi}_{k-1}}^\T \phi$ instead of $Q_{\tilde{\pi}_{k-1}}$ which excludes the $\|w_{\tilde{\pi}_{k-1}}^\top \phi(s, a^\allm) - Q_{\tilde{\pi}_{k-1}}(s, a^\allm)\|_\infty \le \epsilon$ term in the proof of Lemma B.2 in \citep{yin2021efficient}.
    
    \begin{lemma}[Lemma B.2 in \citep{yin2021efficient}]
    \label{lemma: yin lemma b.2 no epsilon}
    Suppose that \cref{ass: feature decomposition} holds. 
    With all terms as defined earlier and $\theta > 0$. 
    Then, with probability at least 
    $$1 - 2C_{\text{max}} \exp(-2 \theta^2(1-\gamma)^2 n)$$
    for any $(s,  a^\allm) \in (\cS \times \cA^\allm)$ pair such that $\phi(s, a^\allm) \in \mathcal{D}$, we have 
    $$|\tilde{w}_{k} (s, a^\allm) - w_{\tilde\pi_{k-1}}^\top (s, a^\allm)| \le b\sqrt{\lambda \tau} + \left(\epsilon +  \frac{\gamma^{H-1}}{1 - \gamma} + \theta \right) \sqrt{\tau C_{\text{max}}}:= \bar\eta$$
    \end{lemma}

    The following Proposition gives us a bound on $\|\tilde Q_{k-1}(s, a^\allm) - Q_{\tilde \pi_{k-1}}(s, a^\allm)\|_\infty$.
    
    \begin{proposition}[approximate value function bound for DAV]
    \label{prop: approx value function bound for DAV}
    Suppose that \cref{ass: feature decomposition} holds. 
    With all terms as defined earlier and $\theta > 0$. 
    Then, with probability at least 
    $$1 - 2C_{\text{max}} \exp(-2 \theta^2(1-\gamma)^2 n)$$
    we have
    $$\|\tilde Q_{k-1}(s, a^\allm) - Q_{\tilde{\pi}_{k-1}}(s, a^\allm)\|_\infty \le \bar\eta (2m-1) + \epsilon := \eta_1.$$ 
    \end{proposition}
    
    \begin{proof}
    
    For any $(s, a^\allm) \in (\bar \cS \times \cA^\allm)$, we have
    \begin{align}
        & |\tilde Q_{k-1}(s, a^\allm) - Q_{\tilde{\pi}_{k-1}}(s, a^\allm)| \nonumber \\
        &= |\tilde{w}_k^\top \phi(s, a^\allm) - Q_{\tilde{\pi}_{k-1}}(s, a^\allm)| \nonumber \\
        &= |\tilde{w}_k^\top \phi(s, a^\allm) \pm w_{\tilde{\pi}_{k-1}}^\top \phi(s, a^\allm) - Q_{\tilde{\pi}_{k-1}}(s, a^\allm)| \nonumber \\
        &\le |\tilde{w}_k^\top \phi(s, a^\allm) - w_{\tilde{\pi}_{k-1}}^\top \phi(s, a^\allm)| + |w_{\tilde{\pi}_{k-1}}^\top \phi(s, a^\allm) - Q_{\tilde{\pi}_{k-1}}(s, a^\allm)| \nonumber \\
        &\le |\tilde{w}_k^\top \phi(s, a^\allm) - w_{\tilde{\pi}_{k-1}}^\top \phi(s, a^\allm)| + \epsilon \nonumber \\
        &= |\tilde{w}_k^\top \phi(s, a^\allm) - w_{\tilde{\pi}_{k-1}}^\top \phi(s, a^\allm) \pm (m-1)\tilde{w}_k^\top \phi(s, \bar a^\allm) \pm (m-1) w_{\tilde{\pi}_{k-1}}^\top \phi(s, \bar a^\allm)| + \epsilon \nonumber \\
        &= \left|\left( \sum_{i=1}^m \tilde{w}_k^\top \phi(s, (a^\ag{i}, \bar a^{(-i)})) - w_{\tilde\pi_{k-1}}^\top \phi(s, (a^\ag{i}, \bar a^{(-i)})) \right) + (m-1)\left[w_{\tilde{\pi}_{k-1}}^\top \phi(s, \bar a^\allm)) - \tilde{w}_k^\top \phi(s, \bar a^\allm\right]\right| + \epsilon \nonumber \\
        &\le m \bar\eta + (m-1) \bar \eta + \epsilon \nonumber \\ 
        &= \bar\eta (2m-1) + \epsilon \label{value function bound 1}
    \end{align}
    where the second last inequality holds by \cref{lemma: yin lemma b.2 no epsilon} (because the features of all the state action pairs considered are in $\mathcal{D}$, since $s \in \bar \cS$).
    
    While for any $(s, a^\allm) \in ((\cS \backslash \bar \cS) \times \cA^\allm)$, we have
    \begin{align}
        |\tilde Q_{k-1}(s, a^\allm) - Q_{\tilde{\pi}_{k-1}}(s, a^\allm)| 
        = |Q_{\tilde \pi_{k-1}}(s, a^\allm) - Q_{\tilde{\pi}_{k-1}}(s, a^\allm)|
        &= 0 \label{value function bound 2}
    \end{align}
    \end{proof}

    Finally, it is left to show that $|\tilde w_{k}^\T \phi(\rho, a^\allm) - Q_{\tilde \pi_{k-1}}(\rho, a^\allm)|$  can be bounded for all $a^\allm \in \cA^\allm$.
    Notice that lines 4-8 in \MCLSPI~run \UncertaintyCheckDAV~with state $\rho$ as input until the returned status is \CERTAIN.
    Recall that once \UncertaintyCheckDAV~returns a status of \CERTAIN~we know that $\rho \in \bar \cS$.
    Thus, we can immediately apply the result in \cref{value function bound 1} to bound $\eta_1 \ge |\tilde w_{k}^\T \phi(\rho, a^\allm) - Q_{\tilde \pi_{k-1}}(\rho, a^\allm)|, \ \forall a^\allm \in \cA^\allm$.

    \subsection{Extending to Politex} \label{subsec:extending to politex}
    Recall the above results where for the \MCLSPI~algorithm.
    The \MCPolitex~algorithm can be found as \cref{alg:confident ma mc-politex}.
    \begin{algorithm}[t]
	\caption{\MCPolitex} \label{alg:confident ma mc-politex}  
	\begin{algorithmic}[1]
		\State \textbf{Input:} initial state $\rho$, initial policy $\pi_0$, number of iterations $K$, threshold $\tau$, number of rollouts $n$, length of rollout $H$
		\State \textbf{Globals:} default action $\bar a$, regularization coefficient $\lambda$, discount $\gamma$, subroutine \UncertaintyCheck
		\State {$\cC \gets \{(\rho, \bar a, \phi(\rho, \bar a), \NONE)\}$} 
\State status, result $\gets \UncertaintyCheck(\rho,  \cC, \tau)$
		\While {status $=$ \UNCERTAIN}
		\State $\cC \gets \cC \cup \{\text{result}\}$
		 \State status, result $\gets \text{\UncertaintyCheck}(\rho, \cC, \tau)$
		\EndWhile
		\State $z_q \gets \NONE, \, \forall z \in \cC$ \quad \Comment{Policy iteration starts $(*)$}
		\For {$k \in 1, \dots, K$}
		\For {$z \in \cC$}
		\State status, result $\gets \text{\ConfidentRollout}(n, H, \pi_{k-1}, z, \cC, \tau)$
		\State \textbf{if} status $=$ \DONE, \textbf{then} $z_q = \text{result}$
		\State \textbf{else} $\cC \gets \cC \cup \{\text{result}\}$ and \textbf{goto} line $(*)$ 
		\EndFor 
		\State $w_k \gets (\Phi_\cC^\top \Phi_\cC + \lambda I)^{-1} \Phi_\cC^\top q_\cC$  
		\State $\pi_k(a^{(1:m)}|s) \gets \propto \prod_{i=1}^m \prod_{j=0}^{k-1} \exp\left(\alpha w_{j}^\T \phi_i(s, a^\ag{i})\right).$ 
		\EndFor
		\State \Return $\bar \pi_{K-1} \sim \text{Unif}\{\pi_k\}_{k=0}^{K-1}$
	\end{algorithmic}
    \end{algorithm}
    It turns out the story for \MCPolitex~is extremely similar and can be argued in nearly the same way. 
    The main difference is that the policy used in \MCPolitex~is different than in \MCLSPI~(line 17 in \MCPolitex~is different from line 17 in \MCLSPI).
    As such, we can no longer use \cref{lemma:approximate policy iteration} (since it relied on a greedy policy) and, thus cannot use \cref{prop:optimality of output policy} to bound the sub-optimality of the policy output by \MCPolitex. 
    Next, we show there is a similar Lemma and Proposition that can derived for \MCPolitex.

    Recall that we do not use clipping on the $Q$-functions in \MCPolitex, so that we can sample from the policy efficiently (\cref{prop: efficient politex policy sampling}).  
    Importantly \cref{prop: efficient politex policy sampling} only holds when \cref{ass: feature decomposition} is satisfied.
    Thus, for the remainder of this section we will be working with the product action space $\cA^\allm$.
    This means we must define the VA's $Q$-function differently from \citep{yin2021efficient}, by removing clipping from the case when $\phi(s, a^\allm) \in \mathcal{D}$.
    \begin{equation*}
        \tilde{Q}_{k-1}(s, a^\allm)=\begin{cases}
                          \tilde{w}_k^\top \phi(s, a^\allm) \quad &\text{if} \, \phi(s, a^\allm) \in \mathcal{D} \\
                          Q_{\tilde{\pi}_{k-1}}(s, a^\allm)     \quad &\text{if} \, \phi(s, a^\allm) \notin \mathcal{D} \\
                    \end{cases}
    \end{equation*}
    Then the VA's policy is
    \begin{equation} \label{eq:politex virtual policy}
        \tilde \pi_k(a^\allm | s) \propto \exp \left( \alpha \sum_{j=0}^{k-1} \tilde Q_{j}(s, a^\allm) \right).
    \end{equation}
    
    Also, due to no clipping, the sequence of $Q$-functions during policy iteration is now in the $[-\eta, (1-\gamma)^{-1} + \eta]$ interval, where $\eta \ge \|\tilde Q_{k-1}(s, a^\allm) - Q_{\tilde \pi_{k-1}}(s, a^\allm)\|_\infty$.
    We now restate Lemma D.1 from \citet{yin2021efficient} which bounds the mixture policy output by Politex for an arbitrary sequence of $Q$-functions
    Since we do not use clipping the theorem is slightly modified (we replace the interval $[0, (1-\gamma)^{-1}]$ with a general interval $[a, b], \ a, b \in \RR$, which can be extracted from the calculations in \citet{szepesvari2022}).

    \begin{lemma}[modified Lemma D.1 in \citet{yin2021efficient} also in \citet{szepesvari2022}] \label{lemma:politex mixture policy bound}
    Given an initial policy $\pi_0$, a sequence of functions $Q_k: \cS \times \cA^\allm \to [a, b], \ k \in [K-1], a, b \in \RR$, and $Q_{\pi^*} \in [0, 1/(1-\gamma)]$, construct a sequence of policies $\pi_1, ..., \pi_{K-1}$ according to (\cref{eq:politex virtual policy}) with $\alpha = 1/(b-a) \sqrt{\frac{2 \log(|\cA^\allm|)}{K}}$, then, for any $s \in \cS$, the mixture policy $\bar \pi_{K-1} \sim \text{Unif}\{\pi_k\}_{k=0}^{K-1}$ satisfies

    \begin{equation}
        V^*(s) - V_{\bar \pi_K}(s) \le \frac{b-a}{(1-\gamma)}\sqrt{\frac{2 \log(|\cA^\allm|)}{K}} + \frac{2 \max_{0 \le k \le K-1} \|Q_k - Q_{\pi_k}\|_\infty}{1 - \gamma}
    \end{equation}
    \end{lemma}

    Notice that the above result suggests we just need to control the term $\|Q_k - Q_{\pi_k}\|_\infty$.
    For the VA this is $\|\tilde Q_k - Q_{\tilde \pi_k}\|_\infty$ and as we have already seen, this can be bounded using the high probability bound on policy evaluation for \textsc{Uncertainty Chcek with DAV} (\cref{prop: approx value function bound for DAV}) and \textsc{Uncertainty Chcek with EGSS} (Proposition \cref{prop: approx value function bound for EGSS}).
    Using \cref{lemma:politex mixture policy bound} instead of Lemma D.1 in \citet{yin2021efficient}, one can extract another slightly modified result from \citet{yin2021efficient}.
    
    \begin{proposition}[equation (D.8) in \citet{yin2021efficient}] \label{prop:politex optimality of output policy}
    With all terms as defined earlier. 
    Define $\eta \ge \|\tilde Q_{k-1}(s, a^\allm) - Q_{\tilde \pi_{k-1}}(s, a^\allm)\|_\infty$. 
    Suppose $\eta \ge |\tilde w_{k}^\T \phi(\rho, a^\allm) - Q_{\tilde \pi_{k-1}}(\rho, a^\allm)|_\infty, \ \forall a^\allm \in \cA^\allm$.
    Then, if the VA and MA behave identically in the final loop, with probability at least $1 - 4KC_{\text{max}}^2 \exp(-2 \theta^2(1-\gamma)^2 n)$ we have
    \begin{align}
            V^*(s) - V_{\bar \pi_{K-1}}(\rho) \le \frac{b-a}{(1-\gamma)} \sqrt{\frac{2 \log(|\cA^\allm|)}{K}} + \frac{4 \eta}{1 - \gamma}
    \end{align}
    \end{proposition}
    
    Notice, that we require the same three things as in the \MCLSPI~case (\cref{prop:optimality of output policy}).
    We need a bound on $\|\tilde Q_{k-1}(s, a^\allm) - Q_{\tilde \pi_{k-1}}(s, a^\allm)\|_\infty$. 
    We need a bound on $|\tilde w_{k}^\T \phi(\rho, a^\allm) - Q_{\tilde \pi_{k-1}}(\rho, a^\allm)|_\infty, \ \forall a^\allm \in \cA^\allm$.
    We need to ensure that the VA and MA behave identically in the final loop. 
    Then, we can get a bound on the sub-optimality of the MA's output policy $\bar \pi_{K-1}$.
    Using the same steps as in the previous sections, one can verify that indeed, \MCPolitex~combined with \UncertaintyCheckDAV~or \UncertaintyCheckEGSS~does satisfy the above three conditions, with $\eta = \eta_1$ ($\eta_1$ as defined in \cref{prop: approx value function bound for DAV}) and $\eta = \eta_2$ ($\eta_2$ as defined in \cref{prop: approx value function bound for EGSS}) respectively.
    
    We bound $|\cA^\allm| = \prodA \le \max_{i \in [m]} |\cA^\ag{i}|$. 
    We can replace $b-a$ with $1/(1-\gamma) + 2\eta$, since $w^\T \phi(s, a^\allm) \in [-\eta, (1-\gamma)^{-1} + \eta], \ \forall (s \times a^\allm) \in (\cS \times \cA^\allm)$ in the final loop for the same event which holds with probability at least $1 - 4KC_{\text{max}}^2 \exp(-2 \theta^2(1-\gamma)^2 n)$ in \cref{prop:politex optimality of output policy}. 
    Applying \cref{prop:politex optimality of output policy} we get with probability at least $1 - 4KC_{\text{max}}^2 \exp(-2 \theta^2(1-\gamma)^2 n)$ that
    \begin{align} 
            V^*(s) - V_{\bar \pi_{K-1}}(\rho) \le \left(\frac{1}{(1-\gamma)^2} + \frac{2\eta}{(1-\gamma)}\right) \sqrt{\frac{2 m \log(\max_{i \in [m]} |\cA^\ag{i}|)}{K}} + \frac{4 \eta}{1 - \gamma}. \label{eq:politex actual optimality of output policy}
    \end{align}

        \section{\uppercase{Kernel Setting}} \label{app:kernel setting}

    Define $q_k(s,a^\allm)$ as the estimated rollout value for $(s,a^\allm) \in \cC$ in round $k \in [K]$ of policy iteration, and $q_k = [q_k(s, a^\allm)]_{(s, a^\allm) \in \cC} \in \bR^{|\cC|}$ as the vector containing all rollout results at round $k$, using some fixed ordering of $\cC$.
    In round $k$ of policy iteration we need to compute the ridge estimate $\hat Q_k$ using $q_k$ as least squares targets,
    \begin{align}
    	\hat Q_k = \argmin_{Q \in \cH} \sum_{(s,a^\allm)\in\cC} (Q(s,a^\allm) - q_k(s,a^\allm))^2 + \lambda \|Q\|_{\cH}^2 = (\Phi_{\cC}^\T\Phi_{\cC} + \lambda \eye_\hH)^{-1}\Phi_{\cC}^\top q_k
    \end{align}
    Here, $\eye_{\cH} : \hH \rightarrow \hH$ is the identity mapping, $\Phi_{\cC}$ can be formally defined as a map $\Phi_\cC : \cH \rightarrow \bR^{|\cC|}, f \mapsto [f(s,a^\allm)]_{(s,a^\allm) \in \cC}, \, f \in \cH$; and $\Phi_{\cC}^\T : \bR^{|\cC|} \rightarrow \hH$ is the adjoint of $\Phi_{\cC}$.
    
    Using the `kernel trick' we express the estimator as follows
    \begin{align}
    	\hat Q_k = \Phi_{\cC}^\T(K_{\cC} + \lambda \eye_{|\cC|})^{-1}q_k
    \end{align}
    where $K_{\cC} = \Phi_{\cC} \Phi_{\cC}^\T \in \bR^{|\cC| \times |\cC|}$ is the kernel matrix. Lastly, we can evaluate for any $(s,a^\allm) \in \cS \times \cA^\allm$:
    \begin{align}
    	\hat Q_k(s,a^\allm) = \kernel_{\cC}(s,a^\allm)^\T(K_{\cC} + \lambda \eye_{|\cC|})^{-1}q_k
    \end{align}
    where we defined $\kernel_\cC(s,a^\allm) = [\kernel(s,a^\allm, s',a'^\allm)]_{(s',a'^\allm) \in \cC} \in \bR^{|\cC|}$ (using the same fixed ordering of $\cC$).
    Importantly, the last display only involves finite-dimensional quantities that can be computed from kernel evaluations.
    Moreover, since $\kernel(s,a^\allm,s',a'^\allm) = \sum_{j=1}^m \kernel_j(s,a^{(j)}, s', a'^{(j)})$ we can write
    \begin{align}
    	\hat Q_k(s,a^\allm) &= \sum_{j=1}^m \hat{Q}_{k, j}(s, a^\ag{j})\\
    	\hat{Q}_{k, j}(s, a^\ag{j}) &:= \kernel_{j, \cC}(s,a^{(j)})^\T(K_{\cC} + \lambda \eye_{|\cC|})^{-1}q_k
     \label{eq:kernel agent-wise q}
    \end{align}
    where $\kernel_{j, \cC}(s,a^{(j)}) = [\kernel_j(s,a^{(j)}, s^\prime , a^{\prime(j)})]_{(s^\prime, a^{\prime(1:m)}) \in \cC} \in \bR^{|\cC|}$.
    Since the term $(K_{\cC} + \lambda \eye_{|\cC|})^{-1}q_k$ is fixed for each $j$, we can still compute the maximizer independently for each $j \in [m]$ by iterating over all actions.
    This allows us to define \MCKern~(\cref{alg:confident kernel mc-lspi/politex}), which makes use of \cref{eq:kernel agent-wise q} in line 17 for calculating the policy.

    \begin{algorithm}[t]
	\caption{\textsc{Confident Kernel MC-LSPI/Politex}} \label{alg:confident kernel mc-lspi/politex}  
	\begin{algorithmic}[1]
		\State \textbf{Input:} initial state $\rho$, initial policy $\pi_0$, number of iterations $K$, threshold $\tau$, number of rollouts $n$, length of rollout $H$
		\State \textbf{Globals:} default action $\bar a$, regularization coefficient $\lambda$, discount $\gamma$, subroutine \UncertaintyCheck, kernel $\kernel$
		\State {$\cC \gets \{(\rho, \bar a, \phi(\rho, \bar a), \NONE)\}$} 
\State status, result $\gets \UncertaintyCheck(\rho,  \cC, \tau)$
		\While {status $=$ \UNCERTAIN}
		\State $\cC \gets \cC \cup \{\text{result}\}$
		 \State status, result $\gets \text{\UncertaintyCheck}(\rho, \cC, \tau)$
		\EndWhile
		\State $z_q \gets \NONE, \, \forall z \in \cC$ \quad \Comment{Policy iteration starts $(*)$}
		\For {$k \in 1, \dots, K$}
		\For {$z \in \cC$}
		\State status, result $\gets \text{\ConfidentRollout}(n, H, \pi_{k-1}, z, \cC, \tau)$
		\State \textbf{if} status $=$ \DONE, \textbf{then} $z_q = \text{result}$
		\State \textbf{else} $\cC \gets \cC \cup \{\text{result}\}$ and \textbf{goto} line $(*)$ 
		\EndFor 
		\State $\hat Q_k = \Phi_{\cC}^\T(K_{\cC} + \lambda \eye_{|\cC|})^{-1}q_k$  
		\State {$\pi_k(a^{(1:m)}|s) \gets 
			\begin{cases}
				\mathds{1}\left(a^\allm = \argmax\limits_{\tilde{a}^\allm \in \cA^\allm} \hat Q_k(s,\tilde{a}^\allm) \right) & \text{LSPI} \\
				\propto \prod_{i=1}^m \prod_{j=0}^{k-1} \exp\left(\alpha \hat{Q}_{j, i}(s, a^\ag{i})\right). & \text{Politex}
			\end{cases}$}
		\EndFor
		\State \Return $\pi_{K-1}$ for LSPI, or $\bar \pi_{K-1} \sim \text{Unif}\{\pi_k\}_{k=0}^{K-1}$ for Politex 
	\end{algorithmic}
    \end{algorithm}
    
    The second quantity required by the algorithm is the squared norm $\|\phi(s,a^\allm)\|_{(\Phi_\cC^\T \Phi_{\cC} + \lambda \eye_\hH)^{-1}}^2$, where now $\phi(s,a^\allm) = \kernel(s,a^\allm, \cdot, \cdot) \in \cH$. A direct extension of the Woodbury formula to infinite vector spaces shows that
    \begin{align}
    	\lambda(\Phi_\cC^\top \Phi_{\cC} + \lambda \eye_\hH)^{-1} = \eye_{\cH} - \Phi_{\cC}^\T (K_\cC + \lambda \eye_{|\cC|})^{-1}\Phi_{\cC}
    \end{align}
    Therefore the feature norm can be written using finite-dimensional quantities: 
    \begin{align}
    	\|\phi(s,a^\allm)\|_{(\Phi_\cC^\T \Phi_{\cC} + \lambda \eye_\hH)^{-1}}^2 = \frac{1}{\lambda} \left( \kernel(s,a^\allm,s,a^\allm) - \kernel_\cC(s,a^\allm)^\T(K_\cC+ \lambda \eye_{|\cC|})^{-1}\kernel_{\cC}(s,a^\allm)\right)
     \label{eq:kernel uncertainty}
    \end{align}
    With this, we can define \UncertaintyCheckKDAV~(\cref{alg:uncertainty check k-dav}) which makes use of \cref{eq:kernel uncertainty}.

    \begin{algorithm}
    	\caption{\textsc{Uncertainty Check with Kernel-Default Action Vector (K-DAV)}} \label{alg:uncertainty check k-dav}  
    	\begin{algorithmic}[1]
    		\State \textbf{Input:} state $s$, core set $\Phi_\cC$, threshold $\tau$.
    		\State \textbf{Globals:} number of action components $m$. 
    		\For {$j \in [m]$}
    		\For {$a^{(j)} \in \cA^{(j)}$}
                \State $\tilde a \gets (a^\ag{j}, \bar a^\ag{-j})$
    		\If {$\frac{1}{\lambda} \left( \kernel(s, \tilde a,s, \tilde a) - \kernel_\cC(s, \tilde a)^\T(\Phi_\cC \Phi_\cC^\T + \lambda \eye_{|\cC|})^{-1}\kernel_{\cC}(s, \tilde a)\right) > \tau$}
    		\State result $\gets (s, \tilde a, \phi(s, \tilde a), \NONE)$
    		\State \Return {\UNCERTAIN, result}
    		\EndIf
    		\EndFor 
    		\EndFor 
    		\State \Return \CERTAIN, \NONE 
    	\end{algorithmic}
    \end{algorithm}

    \paragraph{Analysis}
    Our goal next is to extend the analysis used in the finite case to the kernel case, carefully arguing that the linear dimension $d$ can be replaced by a more benign quantity. A common complexity measure is the total information gain, which we define as follows:
    \begin{align}
    	\Gamma(\lambda; \cC) = \log \det (\Phi_{\cC}^\T\Phi_{\cC} + \lambda \eye_\hH) - \log \det (\lambda \eye_\hH)
    \end{align}
    Note that we can compute $\Gamma(\lambda; \cC)$ for any given core set $\cC$. In the kernel case, we can compute $\Gamma(\lambda; \cC) = \log \det (\eye_{|\cC|} + \lambda^{-1} K_{\cC})$ using similar arguments as before.
    
    The maximum information gain is 
    $$\Gamma_t(\lambda) = \max_{\cC : |\cC|=t} \Gamma(\lambda; \cC).$$ 
    It serves as a complexity measure in the bandit literature and can be bounded for many kernels of interests \citep{srinivas2009gaussian,vakili2021information}. 
    Following \citet{du2021bilinear, huang2021short}, we further define the \emph{critical information gain} for any fixed constant $c > 0$,
    \begin{align}
    	\tilde \Gamma(\lambda, c) = \max \{t \geq 1 : c t \le \Gamma_t(\lambda) \}.\label{eq:critical infogain}
    \end{align}
    Note that the proof of  \cite[Lemma 5.1]{yin2021efficient} implies that $\log(1+\tau)|C| \leq \Gamma_{|C|}(\lambda)$
    Thus, $|\cC| \le C_\text{max} = \tilde \Gamma(\lambda, \log(1+\tau)) $

    Since the dimension $d$ enters our bounds only through $C_{\max}$ we can immediately get a query complexity bound for the kernelized algorithm in terms of $\tilde\Gamma$. 
    For the finite-dimensional case, \cite[Lemma 5.1]{yin2021efficient} shows that $\tilde \Gamma \leq \oO(d)$, recovering the previous bound.

    \section{\uppercase{Proofs of Theorems}} \label{app: theorem proofs}
    We make a remark on the query complexity of \MCLSPI~and \MCPolitex.
    From \cref{lemma:bound on core set size} we know the core set size is bounded by $C_\text{max} = \tilde \cO(d)$.
    The total number of times Policy iteration is restarted (restart means line 14 in \MCLSPI~or \MCPolitex~is reached) is thus at most $C_\text{max}$.
    Each run of policy iteration can take as much as $K$ iterations.
    In each iteration \ConfidentRollout~is run at most $C_\text{max}$ times.
    \ConfidentRollout~does $n$ rollouts of length $H$ which queries the simulator once for each step.
    In total the number of queries performed by \MCLSPI~or \MCPolitex is bounded by $C_\text{max}^2 K n H$.
    This equation is used to calculate the query cost for the different variants of \MCLSPI, once all the parameter values have been calculated.
    Since, the only difference between \MCKern~and \MCLSPI~or \MCPolitex~is how the policy is calculated (lines 16-17 in each of the algorithms), thus we can use the same expression as above ($C_\text{max}^2 K n H$) to bound the query complexity of \MCKern~, with $C_\text{max} = \tilde \Gamma(\lambda, c)$. 


    \subsection{Proof of \cref{thm:mc-lspi-egss sub-optimality}}
    Plugging in $\eta=\eta_2$ ($\eta_2$ as defined in \cref{prop: approx value function bound for EGSS}) into \cref{prop:optimality of output policy}.
    Suppose \cref{asm:linear-q-pi,ass: argmax oracle,ass: bounded features} are satisfied with $\epsilon=0$.
    By choosing appropriate parameters according to $\delta$ and $\kappa$, we can ensure that with probability at least $1 - \delta$ that the policy output by \MCLSPI~combined with \UncertaintyCheckEGSS~, $\pi_{K-1}$ satisfies:
    \begin{align*}
        V^*(\rho) - V_{\pi_{K-1}}(\rho) \leq \kappa,
    \end{align*}
    with the following parameter settings 
    \begin{align*}
        \tau &= 1\\
        \lambda &= \frac{\kappa^2(1 - \gamma)^4}{1024 b^2 d}\\
        \theta &= \frac{\kappa(1- \gamma)^2}{32 \sqrt{d} \sqrt{C_{\text{max}}}}\\
        H &= \frac{
            \log \left ( 32 \sqrt{C_{\text{max}}} \sqrt{d} \right)
            - \log \left( \kappa(1 - \gamma)^3 \right)
        }{
            1-\gamma
        } - 1\\
        K &= \frac{\log\left(\frac{1}{\kappa(1 - \gamma)^2}\right) + \log(8)}{1-\gamma} + 1 \\
        n &= \frac{\log(\delta) - \log(4KC_{\text{max}}^2)}{2 \theta^2(1-\gamma)^2} \\
        C_{\max} &= \frac{e}{e-1} \frac{1 + \tau}{\tau} d \left( 
            \log(1 + \frac{1}{\tau}) +
            \log(1 + \frac{1}{\lambda})
        \right) 
    \end{align*}
    with computational cost of $\poly(d, \frac{1}{1 - \gamma}, \frac{1}{\kappa}, \log(\frac{1}{\delta}))$
    and query cost $\cO\left(\tfrac{d^4}{\kappa^2 (1-\gamma)^8} \right)$ 

    If $\epsilon > 0$, then by choosing parameters as above, with $\kappa = \frac{32 \epsilon d}{(1-\gamma)^2} (1 + \log( b^2 \epsilon^{-2} d^{-1}))^{1/2}$, we can ensure that with probability of at least $1 - \delta$ that $\pi_{K-1}$ satisfies:
    
    $$V^*(\rho) - V_{\pi_{K-1}}(\rho) \leq \frac{64 \epsilon d}{(1-\gamma)^2} (1 +\log(1+b^2 \epsilon^{-2} d^{-1}))^{1/2}$$
    
    with computational cost of $\poly(d, \frac{1}{1 - \gamma}, \frac{1}{\epsilon}, \log(\frac{1}{\delta}), \log(1+b))$
    and query cost $\cO\left(\tfrac{d^2}{\epsilon^2 (1-\gamma)^4} \right)$ 
    

    \subsection{Proof of \cref{thm:mc-lspi-dav sub-optimality}}
    Plugging in $\eta=\eta_1$ ($\eta_1$ as defined in \cref{prop: approx value function bound for DAV}) into \cref{prop:optimality of output policy}.
    Suppose \cref{ass: feature decomposition,ass: bounded features} are satisfied with $\epsilon=0$.
    By choosing appropriate parameters according to $\delta$ and $\kappa$, we can ensure that with probability at least $1 - \delta$ that the policy output by \MCLSPI~combined with \UncertaintyCheckEGSS~, $\pi_{K-1}$ satisfies:
    \begin{align*}
        V^*(\rho) - V_{\pi_{K-1}}(\rho) \leq \kappa,
    \end{align*}
    with the following parameter settings 
    \begin{align*}
        \tau &= 1\\
        \lambda &= \frac{\kappa^2(1 - \gamma)^4}{1024 b^2 (2m-1)^2}\\
        \theta &= \frac{\kappa(1- \gamma)^2}{32 (2m-1) \sqrt{C_{\text{max}}}}\\
        H &= \frac{
            \log \left ( 32 \sqrt{C_{\text{max}}} (2m-1) \right)
            - \log \left( \kappa(1 - \gamma)^3 \right)
        }{
            1-\gamma
        } - 1\\
        K &= \frac{\log\left(\frac{1}{\kappa(1 - \gamma)^2}\right) + \log(8)}{1-\gamma} + 1 \\
        n &= \frac{\log(\delta) - \log(4KC_{\text{max}}^2)}{2 \theta^2(1-\gamma)^2} \\
        C_{\max} &= \frac{e}{e-1} \frac{1 + \tau}{\tau} d \left( 
            \log(1 + \frac{1}{\tau}) +
            \log(1 + \frac{1}{\lambda})
        \right) 
    \end{align*}
    with computational cost of $\poly(\sumA, d, \frac{1}{1 - \gamma}, \frac{1}{\kappa}, \log(\frac{1}{\delta}))$
    and query cost $\cO\left(\tfrac{m^2 d^3}{\kappa^2 (1-\gamma)^8} \right)$ 

    If $\epsilon > 0$, then by choosing parameters as above, with $\kappa = \frac{32 \epsilon \sqrt{d} m}{(1-\gamma)^2} (1 + \log( b^2 \epsilon^{-2} d^{-1}))^{1/2}$, we can ensure that with probability of at least $1 - \delta$ that $\pi_{K-1}$ satisfies:
    
    $$V^*(\rho) - V_{\pi_{K-1}}(\rho) \leq \frac{128 \epsilon \sqrt{d} m}{(1-\gamma)^2} (1 +\log(1+b^2 \epsilon^{-2} d^{-1}))^{1/2}$$
    
    with computational cost of $\poly(\sumA, d, \frac{1}{1 - \gamma}, \frac{1}{\epsilon}, \log(\frac{1}{\delta}), \log(1+b))$
    and query cost $\cO\left(\tfrac{d^2}{\epsilon^2 (1-\gamma)^4} \right)$

    \subsection{Proof of \cref{thm:mc-politex sub-optimality} + \UncertaintyCheckEGSS~case}
    Plugging in $\eta=\eta_1$ when \UncertaintyCheckDAV~is used ($\eta_1$ as defined in \cref{prop: approx value function bound for DAV}) and $\eta=\eta_2$ when \UncertaintyCheckEGSS~is used ($\eta_2$ as defined in \cref{prop: approx value function bound for EGSS}) into \cref{eq:politex actual optimality of output policy}.
    Setting $\zeta=2m-1$ when \UncertaintyCheckDAV~is used, and $\zeta = \sqrt{d}$ when \UncertaintyCheckEGSS~is used.
    Suppose \cref{ass: feature decomposition,ass: bounded features} are satisfied with $\epsilon=0$.
    By choosing appropriate parameters according to $\delta$ and $\kappa$, we can ensure that with probability at least $1 - \delta$ that the policy output by \MCPolitex~$\bar \pi_{K-1}$ satisfies:
    \begin{align*}
        V^*(\rho) - V_{\bar \pi_{K-1}}(\rho) \leq \kappa,
    \end{align*}
    with the following parameter settings
    \begin{align*}
        \tau &= 1\\
        \lambda &= \frac{\kappa^2(1 - \gamma)^2}{576 b^2 \zeta^2}\\
        \theta &= \frac{\kappa(1- \gamma)}{24 \zeta \sqrt{C_{\text{max}}}}\\
        H &= \frac{
            \log \left ( 24 \sqrt{C_{\text{max}}} \zeta \right)
            - \log \left( \kappa(1 - \gamma)^2 \right)
        }{
            1-\gamma
        } - 1\\
        K &= 2m \log(A) \left( \frac{4}{\kappa^2 (1-\gamma)^4} + \frac{3}{\kappa (1-\gamma)^2} + \frac{9}{16} \right)\\
        n &= \frac{\log(\delta) - \log(4KC_{\text{max}}^2)}{2 \theta^2(1-\gamma)^2} \\
        C_{\max} &= \frac{e}{e-1} \frac{1 + \tau}{\tau} d \left( 
            \log(1 + \frac{1}{\tau}) +
            \log(1 + \frac{1}{\lambda})
        \right) 
    \end{align*}
    with computational cost of $\poly(\sumA, d, \frac{1}{1 - \gamma}, \frac{1}{\kappa}, \log(\frac{1}{\delta}))$
    and query cost $\cO\left(\tfrac{m \zeta^2 d^3}{\kappa^4 (1-\gamma)^9} \right)$ 

    If $\epsilon > 0$, then by choosing parameters as above, with $\kappa = \frac{16 \epsilon \sqrt{d} \zeta}{(1-\gamma)} (1 + \log( b^2 \epsilon^{-2} d^{-1}))^{1/2}$, we can ensure that with probability of at least $1 - \delta$ that $\bar \pi_{K-1}$ satisfies:
    
    $$V^*(\rho) - V_{\bar \pi_{K-1}}(\rho) \leq \frac{32 \epsilon \sqrt{d} \zeta}{1-\gamma} (1 +\log(1+ b^2 \epsilon^{-2} d^{-1}))^{1/2}$$

    with computational cost of $\poly(\sumA, d, \frac{1}{1 - \gamma}, \frac{1}{\epsilon}, \log(\frac{1}{\delta}), \log(1+b))$
    and query cost $\cO\left(\tfrac{m d}{\epsilon^4 (1-\gamma)^5} \right)$

    \subsection{Proof of \cref{thm:kernel mc-lspi-dav sub-optimality}}
    Plugging in $\eta=\eta_1$ ($\eta_1$ as defined in \cref{prop: approx value function bound for DAV}) into \cref{prop:optimality of output policy}.
    Suppose \cref{ass: bounded features,ass:kernel q-pi,ass:kernel additive} are satisfied with $\epsilon=0$.
    By choosing appropriate parameters according to $\delta$ and $\kappa$, we can ensure that with probability at least $1 - \delta$ that the policy output by \textsc{Confident Kernel MC-LSPI} $\pi_{K-1}$ satisfies:
    \begin{align*}
        V^*(\rho) - V_{\pi_{K-1}}(\rho) \leq \kappa,
    \end{align*}
    with the following parameter settings 
    \begin{align*}
        \tau &= 1\\
        \lambda &= \frac{\kappa^2(1 - \gamma)^4}{1024 b^2 (2m-1)^2}\\
        \theta &= \frac{\kappa(1- \gamma)^2}{32 (2m-1) \sqrt{C_{\text{max}}}}\\
        H &= \frac{
            \log \left ( 32 \sqrt{C_{\text{max}}} (2m-1) \right)
            - \log \left( \kappa(1 - \gamma)^3 \right)
        }{
            1-\gamma
        } - 1\\
        K &= \frac{\log\left(\frac{1}{\kappa(1 - \gamma)^2}\right) + \log(8)}{1-\gamma} + 1 \\
        n &= \frac{\log(\delta) - \log(4KC_{\text{max}}^2)}{2 \theta^2(1-\gamma)^2} \\
        C_{\max} &= \critinf
    \end{align*}
    with computational cost of $\poly(\sumA, \critinf, \frac{1}{1 - \gamma}, \frac{1}{\kappa}, \log(\frac{1}{\delta}))$
    and query cost $\cO\left(\tfrac{m^2 \critinf^3}{\kappa^2 (1-\gamma)^8} \right)$ 

    If $\epsilon > 0$, then by choosing parameters as above, with $\kappa = \frac{16 \epsilon m \sqrt{\critinf}}{(1-\gamma)^2}$, we can ensure that with probability of at least $1 - \delta$ that $\pi_{K-1}$ satisfies:
    
    $$V^*(\rho) - V_{\pi_{K-1}}(\rho) \leq \frac{32 \epsilon m \sqrt{\critinf}}{(1-\gamma)^2}$$
    
    with computational cost of $\poly(\sumA, \critinf, \frac{1}{1 - \gamma}, \frac{1}{\epsilon}, \log(\frac{1}{\delta}), \log(1+b))$
    and query cost $\cO\left(\tfrac{\critinf^2}{\epsilon^2 (1-\gamma)^4} \right)$ 

    \paragraph{Theroem for \MCPolitexKern~combined with \UncertaintyCheckDAV}

    As mentioned in the body we state the theorem bounding the sub-optimality of the policy output by \MCPolitexKern~combined with \UncertaintyCheckKDAV. 
    \begin{theorem}[\textsc{Confident Kernel MC-Politex DAV} Sub-Optimality] \label{thm:kernel mc-politex-dav sub-optimality}	
            Suppose Assumption \cref{ass:kernel q-pi,ass:kernel additive,ass: bounded features} hold.
            Define $\scritinf := \critinf$.
    	If $\epsilon = 0$, for any $\kappa > 0$, with probability at least $1 - \delta$, the policy $\bar \pi_{K-1}$, output by \MCPolitexKern~combined with \UncertaintyCheckKDAV~satisfies
    	\begin{equation*}
    		V^*(\rho) - V_{\bar \pi_{K-1}}(\rho) \leq \kappa. 
    	\end{equation*}
    	Further, the query cost is $\cO\left(\tfrac{m^3\scritinf^3}{\kappa^4 (1-\gamma)^9} \right)$
    	and computation cost is $\poly(\sumA, \scritinf, \frac{1}{1 - \gamma}, \frac{1}{\kappa}, \log(\frac{1}{\delta}))$ 
    	  If $\epsilon > 0$, then with probability at least $1 - \delta$, the policy $\bar \pi_{K-1}$, output satisfies
    	\begin{equation*}
    		V^*(\rho) - V_{\bar \pi_{K-1}}(\rho) \leq \frac{16 \epsilon m \sqrt{\scritinf}}{1-\gamma}
    	\end{equation*}
    	Further, the query cost is $\cO\left(\tfrac{m \scritinf}{\epsilon^4 (1-\gamma)^5} \right)$
    	and computation cost is $\poly(\sumA, \scritinf, \frac{1}{1 - \gamma}, \frac{1}{\epsilon}, \log(\frac{1}{\delta}), \log(1+b))$
    	The parameter settings for both cases are defined below. 
    \end{theorem}

    \subsection{Proof of \cref{thm:kernel mc-politex-dav sub-optimality}}
    Plugging in $\eta=\eta_1$ ($\eta_1$ as defined in \cref{prop: approx value function bound for DAV}) into \cref{eq:politex actual optimality of output policy}.
    Suppose \cref{ass:kernel additive,ass:kernel q-pi,ass: bounded features} are satisfied with $\epsilon=0$.
    By choosing appropriate parameters according to $\delta$ and $\kappa$, we can ensure that with probability at least $1 - \delta$ that the policy output by \MCPolitexKern~combined with \UncertaintyCheckKDAV~, $\bar \pi_{K-1}$ satisfies:
    \begin{align*}
        V^*(\rho) - V_{\bar \pi_{K-1}}(\rho) \leq \kappa,
    \end{align*}
    with the following parameter settings
    \begin{align*}
        \tau &= 1\\
        \lambda &= \frac{\kappa^2(1 - \gamma)^2}{576 b^2 (2m-1)^2}\\
        \theta &= \frac{\kappa(1- \gamma)}{24 (2m-1) \sqrt{C_{\text{max}}}}\\
        H &= \frac{
            \log \left ( 24 \sqrt{C_{\text{max}}} (2m-1) \right)
            - \log \left( \kappa(1 - \gamma)^2 \right)
        }{
            1-\gamma
        } - 1\\
        K &= 2m \log(A) \left( \frac{4}{\kappa^2 (1-\gamma)^4} + \frac{3}{\kappa (1-\gamma)^2} + \frac{9}{16} \right)\\
        n &= \frac{\log(\delta) - \log(4KC_{\text{max}}^2)}{2 \theta^2(1-\gamma)^2} \\
        C_{\max} &= \critinf
    \end{align*}
    with computational cost of $\poly(\sumA, \critinf, \frac{1}{1 - \gamma}, \frac{1}{\kappa}, \log(\frac{1}{\delta}))$
    and query cost $\cO\left(\tfrac{m^3 \critinf^3}{\kappa^4 (1-\gamma)^9} \right)$ 

    If $\epsilon > 0$, then by choosing parameters as above, with $\kappa = \frac{8 \epsilon m \sqrt{\critinf}}{(1-\gamma)} $, we can ensure that with probability of at least $1 - \delta$ that $\bar \pi_{K-1}$ satisfies:
    
    $$V^*(\rho) - V_{\bar \pi_{K-1}}(\rho) \leq \frac{16 \epsilon m \sqrt{\critinf}}{1-\gamma}$$

    with computational cost of $\poly(\sumA, \critinf, \frac{1}{1 - \gamma}, \frac{1}{\epsilon}, \log(\frac{1}{\delta}), \log(1+ b))$
    and query cost $\cO\left(\tfrac{m \critinf}{\epsilon^4 (1-\gamma)^5} \right)$

\section{\uppercase{Examples and Experiments}}\label{app:additive mdp example}
\begin{figure}[h]
	\begin{center}
		
		\begin{tikzpicture}[auto,node distance=8mm,>=latex,font=\small]
			\tikzstyle{round}=[thick,draw=black,circle]
			\node[round] (s1) {$s_1$};
			\node[round, right=0mm and 20mm of s1] (s2) {$s_2$};
			\node[round, left=0mm and 20mm of s1] (s3) {$s_3$};
			
			\draw[->] (s1) [out=0,in=180] to node {
				\tiny $\phi(s_1,(1,1))=\begin{bmatrix} 1 \\ 0 \\ 
				\end{bmatrix}$
			} node [swap] {
				\tiny $\phi(s_1,(1,0))=\begin{bmatrix} 1 \\ 0 \\ 
				\end{bmatrix}$
			} node [anchor=mid, fill=white!20] {\tiny r = 0} (s2);
			\draw[->] (s1) [out=180,in=0] to node {
				\tiny $\phi(s_1,(0,0))=\begin{bmatrix} 0 \\ 1 \\ 
				\end{bmatrix}$
			} node [swap] {
				\tiny $\phi(s_1,(0,1))=\begin{bmatrix} 0 \\ 1 \\ 
				\end{bmatrix}$
			} node [anchor=mid, fill=white!20] {\tiny r = 0} (s3);
			\draw[->] (s2) [out=60,in=120,loop] to node {\tiny r = 0} node [swap] { \tiny
				$\phi(s_2,(1,0))=\phi(s_2,(0,0))=\begin{bmatrix} 1 \\ 0 \\ 
				\end{bmatrix}$
			}(s2);
			\draw[->] (s2) [out=-60,in=-120,loop] to node { \tiny
				$\phi(s_2,(0,1))=\phi(s_2,(1,1))=\begin{bmatrix} 2 \\ 0 \\ 
				\end{bmatrix}$
			} node [swap] {\tiny r = 1} (s2);
			\draw[->] (s3) [out=60,in=120,loop] to node {\tiny r = 1} node [swap] { \tiny
				$\phi(s_3,(1,0))=\phi(s_3,(0,0))=\begin{bmatrix} 0 \\ 2 \\ 
				\end{bmatrix}$
			}(s3);
			\draw[->] (s3) [out=-60,in=-120,loop] to node { \tiny
				$\phi(s_3,(0,1))=\phi(s_3,(1,1))= \begin{bmatrix} 0 \\ 1 \\ 
				\end{bmatrix}$
			} node [swap] {\tiny r = 0} (s3);
		\end{tikzpicture}
	\end{center}
\caption{Illustration of \cref{ex:ski}.}\label{fig:coordination-example}
\end{figure}
\subsection{Additive MDP, Cooperation Example}
\begin{example}[Coordination]\label{ex:ski}

	
	Consider the MDP in \cref{fig:coordination-example}, which can be verified to satisfy \cref{ass: feature decomposition} with $\gamma = 1/2$ (proof in the next subsection).
	At every time step, two agents in the MDP take actions from $\mathcal{A}^{(1)}= \mathcal{A}^{(2)} = \{0, 1\}$, and move to a next state together.
	The starting state is $s_1$ and by taking a joint action they move to $s_2$ or $s_3$, which 
	are absorbing states and the agents will remain in them once they get there.
	It is easy to see that if we fix the policy for one of the agents in all states, the other agent will face a reduced MDP where the transitions only depends on the its actions.
	We will show that for two different policies followed by the second agent,
	the problem (the MDP) the first agent faces changes.
	More specifically, the best action for the first agent in $s_1$ is different in the resulting MDPs, which suggests that 
	the first agent should coordinate
	with the second agent to achieve a higher value.
	It also shows that this example cannot be reduced to a product MDP, since in product MDPs the best action for each agent is irrespective of the behavior of the other agents.
	
	Assume two different policies $\pi_0, \pi_1: \cS \rightarrow \Delta_{\cA^{(2)}}$ for the second agent,
	such that $\pi_0(s_i) = \delta_0 , \pi_1(s_i) = \delta_1$ for $i \in [3]$ where $\delta_j$ for $j \in [2]$ is the Dirac delta distribution.
	Policy $\pi_0$ causes the joint policy $\pi$ to get reward 1 in $s_3$ and get reward 0 in $s_2$, regardless of the policy followed by the first agent.
	The effect of following $\pi_1$ is exactly the opposite, meaning getting reward 1 in $s_2$ and 0 in $s_3$.
	Consequently, the optimal action for agent 1 depends on choosing $\pi_0$ or $\pi_1$ by the second agent.
	Therefore, agent 1 needs to coordinate its action with the second agent's policy to get the higher reward.
	This property, coordination with other agent's policy, cannot be modeled with separate MDPs
	since in those cases the optimal action for each agent only depends on the agent's MDP, and does not depend
	on the behavior of other agents.
	This example shows that the \cref{ass: feature decomposition} is not limited
	to solving multiple MDPs with joint reward observation, and can model some cases where cooperation is needed.\looseness=-1
\end{example}

\subsubsection*{Realizability}
In this section we prove that the MDP in \cref{fig:coordination-example} satisfies \cref{ass: feature decomposition}.
We start by showing that all the deterministic policies are realizable using the shown feature vectors.
We use the weight vector
$w_{(a_1^1,a_1^2),(a_2^1,a_2^2),(a_3^1,a_3^2)}$ 
for a deterministic policy
that takes action vector $(a^1_i, a^2_i)$ in state $s_i$ for $i \in \{1, 2, 3\}$ and $a^1_i, a^2_i \in \{0,1\}$. We also use $\cdot$ to show that the choice of an action in the respective state does not change the weight vector. 
One can verify that the following vectors satisfy realizability assumption:
\begin{align*}
    w_{(\cdot, \cdot),(\cdot, 0),(\cdot, 0)} = \begin{bmatrix} 0 \\ 1 \end{bmatrix},
    \quad \quad
    w_{(\cdot, \cdot),(\cdot, 0),(\cdot, 1)} = \begin{bmatrix} 0 \\ 0 \end{bmatrix},
    \\
    w_{(\cdot, \cdot),(\cdot, 1),(\cdot, 0)} = \begin{bmatrix} 1 \\ 1 \end{bmatrix},
    \quad \quad
    w_{(\cdot, \cdot),(\cdot, 1),(\cdot, 1)} = \begin{bmatrix} 1 \\ 0 \end{bmatrix}.
\end{align*}
It remains to show that the non-deterministic policies are also realizable.
For a policy $\pi$ that takes action $(\cdot, 1)$ at $s_2$ with probability $p_2$, and action $(\cdot, 0)$ at $s_3$ with probability $p_3$, the realizable weight vector is:
\[
    w_\pi = \begin{bmatrix}
    p_2\\
    p_3\\
    \end{bmatrix}.
\]
This holds since the choice of the action in $s_1$ does not change the weight vector in this example.

\subsection{Experimental Results}

\begin{wrapfigure}{r}{0.3\textwidth}
	\begin{center}
		\includegraphics[width=0.3\textwidth]{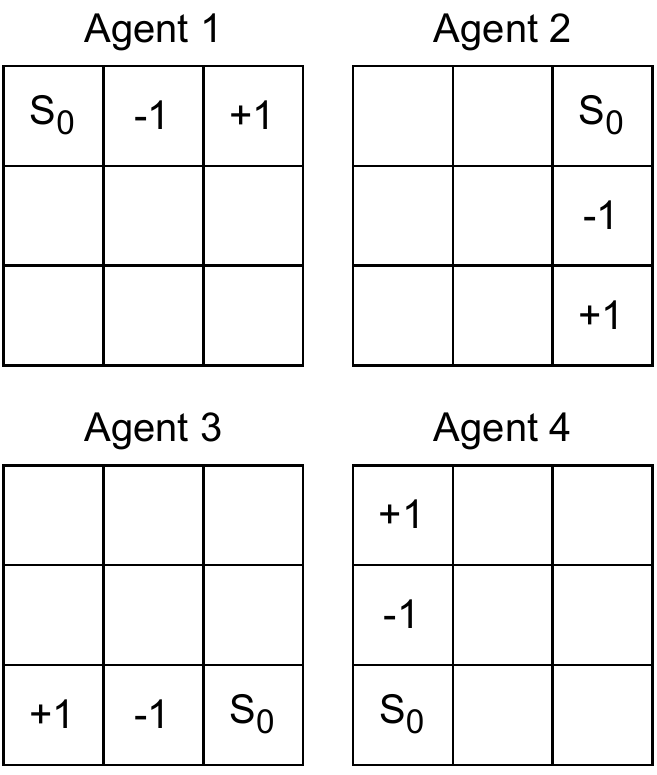}
	\end{center}
	\caption{Four agent grid world.}\label{fig:grid}
\end{wrapfigure}

We evaluate the performance of the proposed algorithms in a small grid world example as shown in \cref{fig:grid}. 
Each of four agents is placed in a 3x3 grid world. 
The agents obtain a +1 reward for reaching the goal state and a -1 reward in a `trap' state. 
Reaching either the trap state or the reward state terminates the episode. 
Each agent has four actions to move to a neighboring cell. 
The selected action is applied with probability 0.95 while with 0.05 probability an action is chosen uniformly at random. 
The global reward is the sum of the agents rewards. 
Note that the individual rewards are not observed, therefore the example is different from four separate grid worlds.

We run each variant of the algorithm for 50 iterations ($K=50$) without resets (the resets are mainly for simplicity of analysis). 
The discount factor is set to $\gamma = 0.8$, the regularization parameter is set to $\lambda = 10^{-5}$, for Politex we set $\alpha=1$ and the rollout length is $H=15$. 
The agents' individual features are one-hot encodings of agent, agent positions and actions which results in a feature of dimension $d=4 \cdot 9 \cdot 4 = 144$. 
Note, however, that the joint MDP is \emph{not} tabular, as the joint features,~i.e. the sum over the agent features, are not one-hot vectors. 
In fact, the features are crucial for generalization as there are a total $9^4 = 9561$ joint states for all four agents combined.

\Cref{fig:exp} shows two experiments with $n=10$ and $n=50$ rollouts. 
The plots show the performance of the policy estimate after each iteration averaged over 25 random seeds. 
We run both \MCLSPI~and \MCPolitex~with \EGSS~(\cref{alg:uncertainty check egss}) and \DAV~(\cref{alg:uncertainty check dav}) uncertainty checks. 
In addition we compare to the \NAIVE~uncertainty check (\cref{alg:uncertainty check}) that iterates over all $|A| = 4^4$ actions \citep{yin2021efficient}.  
Note that with 50 rollouts, both the \EGSS~and \DAV~variants perform essentially the same as \NAIVE, despite the relaxed uncertainty bound. 
LSPI finds a good policy within at most five iterations. 
With only 10 rollouts, the final policy of \MCLSPI~ converges to a suboptimal value on average. This can be understood as the data between iterations is not shared, and the noise from the Monte-Carlo estimates sometimes leads to a deteriorating in the policy improvement step. With 50 rollouts per iteration, LSPI reliably finds the optimal policy in all MDPs.
On the other hand, \MCPolitex~is much more stable even with just 10 rollouts, but also requires more iterations to converge. 
This is expected because in \MCPolitex,~the policy estimates from all iterations are averaged.

\begin{figure}[t]
	\includegraphics{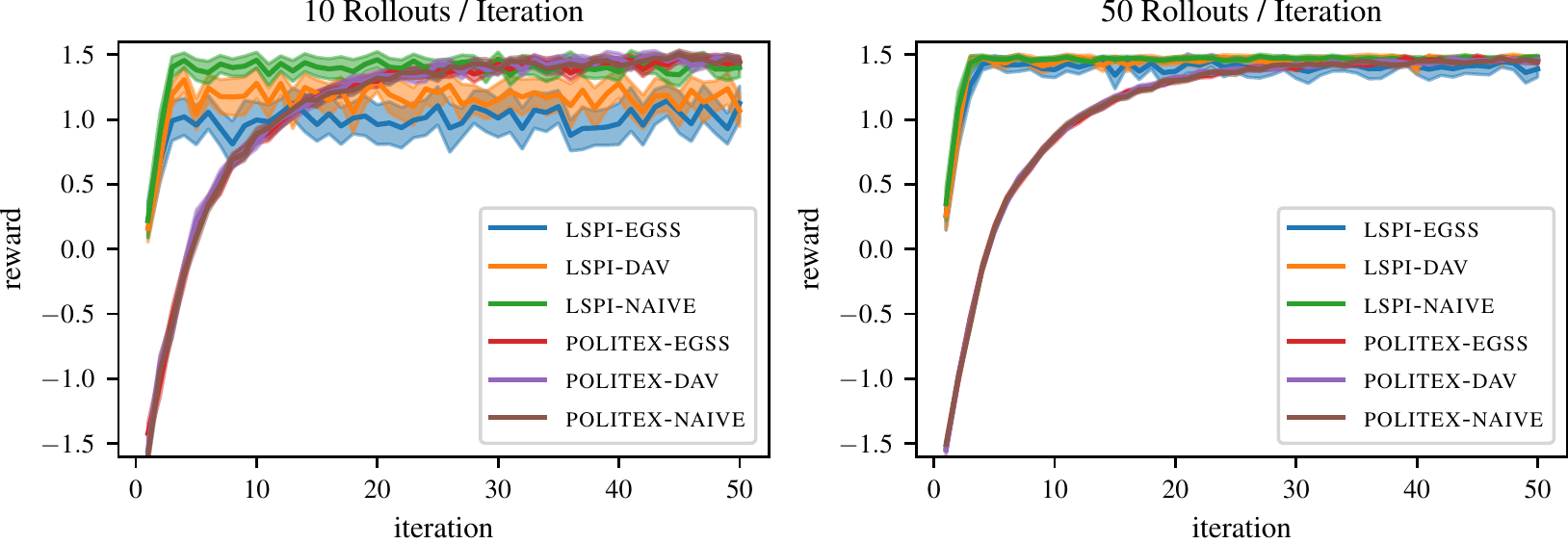}
	\caption{Numerical results on a grid world with four agents.}\label{fig:exp}
\end{figure}

\end{document}


%

%

\onecolumn
\aistatstitle{Provably Efficient Cooperative Multi-Agent Reinforcement Learning with an Online Simulator \\
Supplementary Materials}

\section{Efficient Policy Sampling} \label{app:efficienct policy sampling}
    The policy in \citep{yin2021efficient} can be extended to the multi-agent setting with action space $\cA^\allm$ as follows
    \begin{equation}
    \pi_k(a^\allm|s) \gets 
        \begin{cases}
            \mathbbm{1}\left(a^\allm = \argmax\limits_{\tilde{a}^\allm \in \cA^\allm} w^\top \phi(s, \tilde{a}^\allm)\right) & \text{LSPI} \\
            \exp\left(\alpha \sum\limits_{j=0}^{k-1} w_{j}^\T \phi(s, a^\allm)\right) / \sum\limits_{a^\allm \in \cA^\allm} \exp\left(\alpha \sum\limits_{j=0}^{k-1} w_{j}^\T \phi(s, a^\allm)\right) . & \text{Politex}
        \end{cases} \label{eq:yin policy}
    \end{equation}
    with $w_k = (\Phi_\cC^\top \Phi_\cC + \lambda I)^{-1} \Phi_\cC^\top q_\cC$ 
    and $Q_{k-1}(s, a) = \Pi_{[0, (1-\gamma)^{-1}]} (w_k^\T \phi(s, a))$ for the Politex case only.
    In this appendix we show that the above policy can be sampled from efficiently if assumption \cref{ass: feature decomposition}  or \cref{ass: argmax oracle} is satisfied for the LSPI case and policy $\pi_k$ can be sampled from efficiently if \cref{ass: argmax oracle} is satisfied for the Politex case.
    To be precise, by efficiently we mean with computation that depends on $\poly(m, A, d)$ and not $\poly(A^m, d)$.
    We assume only $w \in \RR^d$ or $w_0, ..., w_{k-1} \in \RR^d$ (for LSPI and Politex respectively) and a feature map $\phi: \cS \times \cA^\allm \to \RR^d$ are given, thus the process of sampling may require calculating the policy if necessary to accurately sample.
    First we handle the LSPI case.
    
    \begin{proposition}[Efficient LSPI Policy Sampling]
        \label{prop: efficient lspi policy sampling}
        Given state $s \in \cS$, parameter vector $w \in \RR^d$, feature map $\phi: \cS \times \cA^\allm \to \RR^d$ and assumption \ref{ass: feature decomposition} or \ref{ass: argmax oracle} satisfied. 
        Then policy
        $$\pi_k(a^\allm|s) = \mathbbm{1}\left(a^\allm = \argmax_{\tilde{a}^\allm \in \cA^\allm} w^\top \phi(s, \tilde{a}^\allm)\right)$$
        can be sampled from in time $\poly(d, m, A)$.
    \end{proposition}
    \begin{proof}
        One can sample from policy $\pi_k$ by simply outputting the result of $\argmax_{\tilde{a}^\allm \in \cA^\allm} w^\top \phi(s, \tilde{a}^\allm)$. 
        Under assumption \ref{ass: argmax oracle} $\argmax_{\tilde{a}^\allm \in \cA^\allm} w^\top \phi(s, \tilde{a}^\allm)$ can be computed in time $\poly(d, m, A)$ by applying the oracle to $w$ and $\phi$ (i.e. $\cO(w, \phi)$).
        Since assumption \ref{ass: feature decomposition} implies assumption \ref{ass: argmax oracle} the result also holds under assumption \ref{ass: feature decomposition}. 
    \end{proof}

    Next, we handle the Politex case. 
    To achieve the result below we had to modify the Politex algorithm slightly, by removing the clipping of the $Q$-function at each iteration $k$ (i.e. we define the $Q$-function at iteration $k$ to be $Q_{k-1}(s, a) = w_k^\T \phi(s, a^\allm)$ instead of $Q_{k-1}(s, a) = \min\{\max\{w_k^\T \phi(s, a^\allm), 0\}, 1/(1-\gamma)\}$).
    This was done since we were not aware of an efficient way to compute the clipped $Q$-function for all action-vectors in $\cA^\allm$.
    Removing the clipping only suffers a minor increase in the final policies sub-optimality (shown in Appendix \ref{app: theorem proofs} and in the Notes section of \citet{szepesvari2022})
    \begin{proposition}[Efficient Politex Policy Sampling]
        \label{prop: efficient politex policy sampling}
        Given state $s \in \cS$, parameter vectors $w_0, ..., w_{k-1} \in \RR^d$, feature map $\phi: \cS \times \cA^\allm \to \RR^d$ and assumption \ref{ass: feature decomposition} satisfied. 
        Then policy
        $$\pi_k(a^\allm|s) = \exp\left(\alpha \sum\nolimits_{j=0}^{k-1} w_{j}^\T \phi(s, a^\allm)\right)/ \sum\nolimits_{\tilde{a}^\allm \in \cA^\allm} \exp \left(\alpha \sum\nolimits_{j=0}^{k-1} w_{j}^\T \phi(s, \tilde{a}^\allm)\right)$$ 
        can be sampled from in time $\poly(d, m, A)$.
    \end{proposition}

    \begin{proof}
    Fix arbitrary $a^\allm \in \cA^\allm$.
    To sample from $\pi_k$ it is sufficient to sample actions $a^\allm \in \cA^\allm$ proportional to $\exp(\alpha \sum_{j=0}^{k-1} Q_{j}(s, a^\allm))$.
    Rearranging $\exp(\alpha \sum_{j=0}^{k-1} Q_{j}(s, a^\allm))$ and plugging in that $\phi(s, a^\allm) = \sum_{i=1}^m w^\T \phi_i(s, a^\ag{i})$ under assumption \ref{ass: feature decomposition} we have
    \begin{align*}
        \exp\left(\alpha \sum\nolimits_{j=0}^{k-1} w_{j}^\T \phi(s, a^\allm)\right) 
        &= \prod_{j=0}^{k-1} \exp\left(\alpha w_{j}^\T \phi(s, a^\allm)\right) \\
        &= \prod_{j=0}^{k-1} \exp\left(\alpha w_{j}^\T \sum_{i=1}^m \phi_i(s, a^\ag{i})\right) \\
        &= \prod_{i=1}^m \prod_{j=0}^{k-1} \exp\left(\alpha w_{j}^\T \phi_i(s, a^\ag{i})\right)
    \end{align*}

    Which means that the probability of sampling action $a^\allm$ is equal to the product of the probabilities of sampling $a^\ag{i}$ for agents $i \in [m]$ independently.
    Since $a^\allm$ was arbitrary this completes the proof.
    \end{proof}

    \section{{Bound on Core Set Size}}
    Recall, that only tuples containing state-action vectors that satisfy $\phi(s, a^\allm)^\top (\Phi^\top \Phi + \lambda I)^{-1} \phi(s, a^\allm) > \tau$ are add to the core set. 
    This ensures that the size of the core set can be be bounded by a $\poly(d, \tau, \log(m), \log(1/\lambda))$ function.
    
    \begin{lemma}[Bound on Core Set Size (modified Lemma 5.1 in \citep{yin2021efficient})] \label{lemma:bound on core set size}
        When \cref{ass: bounded features} is satisfied, and $(s,  a^\allm) \in (\cS \times \cA^\allm)$ that satisfy $\phi(s, a^\allm)^\top (\Phi_\cC^\top \Phi_\cC + \lambda I)^{-1} \phi(s, a^\allm) > \tau$ are added to the core set, the size of the core set can be bounded by
        \begin{align}
             \tilde{C}_{\max} := \frac{e}{e-1} \frac{1 + \tau}{\tau} d \left( 
                \log(1 + \frac{1}{\tau}) +
                \log(1 + \frac{m}{\lambda})
            \right). \label{eq:cmax-new}       
        \end{align}
    \end{lemma}

    The Lemma displayed above is borrowed directly from \citet{yin2021efficient}.
    The only difference is that the features in our setting are bounded by $m$ instead of $1$ (Assumption \ref{ass: bounded features}), and thus we obtain an extra $\log(m)$ factor, which can easily be verified.

    \section{Efficient Uncertainty Check} \label{app:efficient uncertainty check}
    \label{app: efficient uncertainty check}
    The \textsc{Uncertainty Check} algorithm is the implementation used by \citet{yin2021efficient} extended to the multi-agent setting, where the action space is $\cA^\allm$.
    The \textsc{Confident MC-LSPI/Politex} algorithm proposed by \citet{yin2021efficient} algorithm is identical to our \textsc{Confident Multi-Agent MC-LSPI/Politex} algorithm with \textsc{Uncertainty Check with (DAV/EGSS)} replaced with \textsc{Uncertainty Check} and the policy on line 15 replaced with \cref{eq:yin policy}.
    \begin{algorithm}
    \caption{\textsc{Uncertainty Check}} \label{alg:uncertainty check}  
    \begin{algorithmic}[1]
    \State \textbf{Input:} state $s$, feature matrix $\Phi_\cC$, regularization coefficient $\lambda$, threshold $\tau$.
    \For {$a^\allm \in \cA^\allm$}
        \If {$\phi(s, a^\allm)^\top (\Phi_\cC^\top \Phi_\cC + \lambda I)^{-1} \phi(s, a^\allm) > \tau$}
            \State status $\gets$ uncertain, result $\gets (s, a^\allm, \phi(s, a^\allm), \text{none})$
            \State \Return {status, result}
        \EndIf
    \EndFor 
    \State \Return certain, none 
    \end{algorithmic}
    \end{algorithm}
    In this appendix we show how the loop over all action vectors $a^\allm \in \cA^\allm$ in the \textsc{Uncertainty Check} algorithm (line 2) can be avoided when either \cref{ass: feature decomposition} or \cref{ass: argmax oracle} is satisfied.
    In particular, we show that \textsc{Uncertainty Check with DAV} and \textsc{Uncertainty Check with EGSS} algorithms are able to reduce the computation time of the \textsc{Uncertainty Check} from $\poly(A^m, d)$ to $\poly(m, A, d)$, while still maintaining suitable output policy guarantees.

    Since, we are extending the \textsc{Confident MC-LSPI/Politex} algorithm proposed by \citet{yin2021efficient}, we will be borrowing much of the steps from their proof.
    \citet{yin2021efficient} used a \textit{virtual algorithm} (VA) and \textit{main algorithm} (MA) to prove the sub-optimality of their \textsc{Confident MC-LSPI/Politex} algorithm.
    We give a brief summary of the VA and MA; however, avoid full details since we use the exact same definition as in \citet{yin2021efficient}. 
    Until the next subsection, assume \textsc{Uncertainty Check} is used in place of \textsc{Uncertainty Check with (DAV/EGSS)} in \textsc{Confident Multi-Agent MC-LSPI/Politex}.
    The MA is exactly \textsc{Confident Multi-Agent MC-LSPI/Politex}. 
    The VA is based on the \textsc{Confident Multi-Agent MC-LSPI/Politex} algorithm, but has some differences, which we outline next. 
    The VA runs for exactly $C_\text{max}$ loops, $K$ iterations, and completes all $n$ of its rollouts of length $H$. 
    For each loop and iteration $k$ the VA always obtains estimates $q_\cC$ of its policy.
    The VA uses a different policy than the MA for rollouts.
    We will first focus on the LSPI case and return to Politex much later.
    The VA's $Q$-function at iteration $k$ is  
    \begin{equation*}
        \tilde{Q}_{k-1}(s, a^\allm)=\begin{cases}
                          \tilde{w}_k^\top \phi(s, a^\allm) \quad &\text{if} \, \phi(s, a^\allm) \in \cH \\
                          Q_{\tilde{\pi}_{k-1}}(s, a^\allm)     \quad &\text{if} \, \phi(s, a^\allm) \notin \cH \\
                    \end{cases}
    \end{equation*}
    where $\tilde w_k = V_\cC^{-1} \Phi_\cC^\top \tilde q_\cC$, and $\tilde q_\cC$ are the estimates obtained from running \textsc{Multi-Agent-Confident Rollout} on each element of the core set, and $\cH = \{\phi(s, a^\allm): \|\phi(s, a^\allm)\|_{V_\cC^{-1}}^2 \le \tau\}$ is the \textit{good set}.
    The VA's policy is
    \begin{equation*}
        \tilde \pi_k(a^\allm | s) = \mathbbm{1} \left( a^\allm = \argmax_{\tilde a^\allm \in \cA^\allm} \tilde Q_{k-1}(s, \tilde a^\allm) \right).
    \end{equation*}

    The nice thing about defining the VA's policy in this way is that we can make use of the following Lemma from \citep{yin2021efficient}.
    
    \begin{lemma}[Lemma B.2 in \citep{yin2021efficient}]
    \label{lemma: yin lemma b.2}
    Suppose that Assumption \ref{ass: feature decomposition} holds. 
    With all terms as defined earlier and $\theta > 0$. 
    Then, with probability at least 
    $$1 - 2C_{\text{max}} \exp(-2 \theta^2(1-\gamma)^2 n)$$
    for any $(s,  a^\allm) \in (\cS \times \cA^\allm)$ pair such that $\phi(s, a^\allm) \in \mathcal{H}$, we have 
    $$|\tilde{Q}_{k-1} (s, a^\allm) - Q_{\tilde\pi_{k-1}} (s, a^\allm)| \le b\sqrt{\lambda \tau} + \left(\epsilon +  \frac{\gamma^{H-1}}{1 - \gamma} + \theta \right) \sqrt{\tau C_{\text{max}}} + \epsilon := \eta$$
    \end{lemma}

    Notice that for any $(s,  a^\allm) \in (\cS \times \cA^\allm)$ pair such that $\phi(s, a^\allm) \notin \mathcal{H}$, the VA's $Q$-function $\tilde Q_{k-1}$ has access to the true $Q$-function $Q_{\tilde{\pi}_{k-1}}$ of policy $\tilde \pi_{k-1}$.
    Thus, we have that 
    \begin{equation}
    \label{eq:va inf-norm bound}
        \| \tilde{Q}_{k-1} (s, a^\allm) - Q_{\tilde\pi_{k-1}} (s, a^\allm) \|_\infty \le \eta 
    \end{equation}
    Combined with the fact that $\tilde \pi_k$ is greedy w.r.t. $\tilde Q_{k-1}$ the above result turns out to be especially useful.
    
    To understand why, we state a classic policy improvement result, which can be found as Lemma B.3 in \citet{yin2021efficient} and in other papers.
    \begin{lemma}[approximate policy iteration]
    \label{lemma:approximate policy iteration}
        Suppose that we run K approximate policy iterations and generate a sequence of policies
        $\pi_0, \pi_1, \pi_2, \cdots, \pi_K$.
        Suppose that for every $k = 1, 2, \cdots, K$, in the k-th iteration, we obtain a function
        $\tilde{Q}_{k-1}$ such that, $\| \tilde{Q}_{k - 1} - Q_{\pi_{k - 1}} \|_\infty \leq \eta$,
        and choose $\pi_k$ to be greedy with respect to $\tilde{Q}_{k-1}$.
        Then
        \begin{align*}
            \| Q^* - Q_{\pi_K} \|_\infty \leq \frac{2 \eta}{1 - \gamma} + \frac{\gamma^K}{1 - \gamma},
        \end{align*}
    \end{lemma}

    In our case the VA's policy $\tilde \pi_k$ is greedy w.r.t. $\tilde Q_{k-1}$ and thus we have that
    \begin{align*}
        \| Q^* - Q_{\tilde \pi_K} \|_\infty \leq \frac{2 \eta}{1 - \gamma} + \frac{\gamma^K}{1 - \gamma},
    \end{align*}

    Now we explain how the MA can be related to the VA, and make use of the above result.
    The \textsc{Uncertainty Check} algorithm can have two cases: 
    
    \textbf{Case 1:} $\|\phi(s, a^\allm)\|_{V_\cC^{-1}}^2 > \tau$ holds for at least one $a^\allm \in \cA^\allm$,
    
    \textbf{Case 2:} $\|\phi(s, a^\allm)\|_{V_\cC^{-1}}^2 \le \tau$ holds for all $a^\allm \in \cA^\allm$. This is equivalent to saying $\phi(s, a^\allm) \in \cH, \ \forall a^\allm \in \cA^\allm$.

    The VA is exactly the same at the MA algorithm, until Case 1 occurs for the first time.
    This is because the MA's and VA's simulators are coupled, in the sense that at iteration $k$, rollout $i$, and step $t$, when both simulators are queried with the same state-action vector pairs, they sample the exact same next state and reward. 
    The VA also uses the same initial policy as the MA at the start of policy iteration for every loop.
    Once Case 1 occurs the MA would restart policy iteration (else condition in line 12 of \textsc{Confident Multi-Agent MC-LSPI/Politex}), while the VA does not. 
    The VA records the state-action vector pair when Case 1 occurs for the first time and adds it to the core set once it completes running policy iteration for the current loop.
    In this way the core set maintained by the MA and VA are always the same.
    Since the size of the core set is bounded by $C_\text{max}$ when $(s, a^\allm) \in (\cS \times \cA^\allm)$ that satisfy $\phi(s, a^\allm)^\top (\Phi_\cC^\top \Phi_\cC + \lambda I)^{-1} \phi(s, a^\allm) > \tau$ are added to the core set (\cref{lemma:bound on core set size}), there will be a loop of policy iteration at which the MA and VA never encounter Case 1 for any of the $K$ iterations of policy iteration.
    We call this loop the \emph{final loop}.
    This is equivalent to say that all $(s,  a^\allm) \in (\cS \times \cA^\allm)$ observed during all $K$ iterations of policy iteration in the final loop are in the good set (i.e. $\phi(s, a^\allm) \in \cH$).
    Notice that this means MA and VA behaved identical in the final loop, since the VA's policy would have always been greedy w.r.t. $\tilde{w}_k^\top \phi$ and the MA and VA use the same initial policy at the start of each loop.
    It turns out this relationship between the MA and VA allows us to bound the sub-optimality of the MA in the final loop, by using the result in \cref{eq:va inf-norm bound} we have for the VA. 
    More precisely, the following result can be extracted from \citep{yin2021efficient}
    \begin{proposition}[equation (B.15) in \citet{yin2021efficient}] \label{prop:optimality of output policy}
    With all terms as defined earlier. 
    Define $\eta \ge \|\tilde Q_{k-1}(s, a^\allm) - Q_{\tilde \pi_{k-1}}(s, a^\allm)\|_\infty$. 
    Suppose $\eta \ge |\tilde w_{k}^\T \phi(\rho, a^\allm) - Q_{\tilde \pi_{k-1}}(\rho, a^\allm)|_\infty, \ \forall a^\allm \in \cA^\allm$.
    Then, if the VA and MA behave identically in the final loop, with probability at least $1 - 4KC_{\text{max}}^2 \exp(-2 \theta^2(1-\gamma)^2 n)$ we have
    \begin{align}
        V^*(\rho) - V_{\pi_{K-1}}(\rho) \leq \frac{8 \eta}{(1 - \gamma)^2} + \frac{2 \gamma^{K - 1}}{(1 - \gamma)^2} \label{eq: v function bound main}
    \end{align}
    \end{proposition}
    
    Notice, that we require three things to use the above result. 
    We need a bound on $\|\tilde Q_{k-1}(s, a^\allm) - Q_{\tilde \pi_{k-1}}(s, a^\allm)\|_\infty$. 
    We need a bound on $|\tilde w_{k}^\T \phi(\rho, a^\allm) - Q_{\tilde \pi_{k-1}}(\rho, a^\allm)|_\infty, \ \forall a^\allm \in \cA^\allm$.
    We need to ensure that the VA and MA behave identically in the final loop. 
    Then, we can get a bound on the sub-optimality of the MA's output policy $\pi_{K-1}$.
    An important observation is that \textsc{Uncertainty Check} ensured that MA and VA behave identically in the final loop.
    It did this by making sure that the VA's policy $\tilde \pi_k$ would only be able to use $\tilde{w}_k^\top \phi$ to derive its actions, since \textsc{Uncertainty Check} always returns a \textsc{status} of \textsc{certain} in the final loop, which means that $\phi(s, a^\allm) \in \cH$ for all $s, a^\allm \in \cS \times \cA^\allm$ encountered in the final loop.
    With this information in mind, we now show that \textsc{Uncertainty Check with DAV} and \textsc{Uncertainty Check with EGSS} only requires computation $\poly(d, m, A)$, while providing only slightly worse sub-optimality guarantees when compared to the result in \citep{yin2021efficient}.


    \subsection{Default Action Vector (DAV) Method}
    In this section we prove some useful results for \textsc{Uncertainty Check with DAV}.
    Fix a state $s \in \cS$.
    First, \textsc{Uncertainty Check with DAV} only iterates over $\sumA$ action vectors instead of all the action vectors like \textsc{Uncertainty Check} does.
    Define the $\sum_{i=1}^m A^\ag{i}$ sized set of modified default action vectors as $\bar \cA^\allm = \{ (a^\ag{i}, \bar a^{(-i)}): a^\ag{i} \in \cA^\ag{i}, \ i \in [m] \}$.
    Notice \textsc{Uncertainty Check with DAV} iterates over all the actions in the set $a^\allm \in \bar \cA^\allm$ and checks if any of them satisfy $\|\phi(s, a^\allm)\|_{V_\cC^{-1}}^2 > \tau$.
    This of course achieves the goal of only $\poly(d, m, A)$ computation.
    Also, since only $a^\allm \in \cA^\allm$ that satisfy $\|\phi(s, a^\allm)\|_{V_\cC^{-1}}^2 > \tau$ are added to the core set, we can still use \cref{lemma:bound on core set size} to bound the size of the core set by $C_\text{max}$.

    Now, we aim to ensure that the VA and MA behave identically in the final loop.
    Define the set of states for which all the modified default action vectors are in the good set as $\bar \cS = \{ s \in \cS: \|\phi(s, a^\allm)\|_{V_\cC^{-1}}^2 \le \tau, \forall a^\allm \in \bar \cA^\allm \}$.
    Redefine the VA's $Q$-function as 
    \begin{equation*}
    \tilde Q_{k-1}(s, a^\allm)=\begin{cases}
          \tilde w_k^\top \phi(s, a^\allm) \quad & s \in \bar \cS \\
          Q_{\tilde{\pi}_{k-1}}(s, a^\allm). \quad & s \in \cS \backslash \bar \cS
        \end{cases}
    \end{equation*}
    The VA's policy is
    \begin{equation*}
        \tilde \pi_k(a^\allm | s) = \mathbbm{1} \left( a^\allm = \argmax_{\tilde a^\allm \in \cA^\allm} \tilde Q_{k-1}(s, \tilde a^\allm) \right).
    \end{equation*}
    Notice that in the final loop the check $\phi(s, (a^{(j)}, \bar a^{(-j)}))^\top (\Phi_\cC^\top \Phi_\cC + \lambda I)^{-1} \phi(s, (a^{(j)}, \bar a^{(-j)})) > \tau$ in \textsc{Uncertainty Check with DAV} never returns \textsc{True}, and thus we are sure that all $a^\allm \in \bar \cA^\allm$ for all the states encountered in the final loop are in the good set.
    Notice that these states that satisfy this condition are state in $\bar \cS$.
    Thus, the VA's policy $\pi_{k}$ would always be greedy w.r.t. $\tilde w_k^\top \phi$ in the final loop.
    This ensures that the VA and MA behave identically in the final loop.

    Now we show that we can bound $\|\tilde Q_{k-1}(s, a^\allm) - Q_{\tilde \pi_{k-1}}(s, a^\allm)\|_\infty$ with this new definition of $\tilde Q_{k-1}$. 
    First we state a slight modification of \cref{lemma: yin lemma b.2} for $w_{\tilde{\pi}_{k-1}}^\T \phi$ instead of $Q_{\tilde{\pi}_{k-1}}$ which excludes the $\|w_{\tilde{\pi}_{k-1}}^\top \phi(s, a^\allm) - Q_{\tilde{\pi}_{k-1}}(s, a^\allm)\|_\infty \le \epsilon$ term in the proof of Lemma B.2 in \citep{yin2021efficient}.
    
    \begin{lemma}[Lemma B.2 in \citep{yin2021efficient}]
    \label{lemma: yin lemma b.2 no epsilon}
    Suppose that Assumption \ref{ass: feature decomposition} holds. 
    With all terms as defined earlier and $\theta > 0$. 
    Then, with probability at least 
    $$1 - 2C_{\text{max}} \exp(-2 \theta^2(1-\gamma)^2 n)$$
    for any $(s,  a^\allm) \in (\cS \times \cA^\allm)$ pair such that $\phi(s, a^\allm) \in \mathcal{H}$, we have 
    $$|\tilde{w}_{k} (s, a^\allm) - w_{\tilde\pi_{k-1}}^\top (s, a^\allm)| \le b\sqrt{\lambda \tau} + \left(\epsilon +  \frac{\gamma^{H-1}}{1 - \gamma} + \theta \right) \sqrt{\tau C_{\text{max}}}:= \bar\eta$$
    \end{lemma}

    The following Proposition gives us a bound on $\|\tilde Q_{k-1}(s, a^\allm) - Q_{\tilde \pi_{k-1}}(s, a^\allm)\|_\infty$.
    
    \begin{proposition}[approximate value function bound for DAV]
    \label{prop: approx value function bound for DAV}
    Suppose that Assumption \ref{ass: feature decomposition} holds. 
    With all terms as defined earlier and $\theta > 0$. 
    Then, with probability at least 
    $$1 - 2C_{\text{max}} \exp(-2 \theta^2(1-\gamma)^2 n)$$
    we have
    $$\|\tilde Q_{k-1}(s, a^\allm) - Q_{\tilde{\pi}_{k-1}}(s, a^\allm)\|_\infty \le \bar\eta (2m-1) + \epsilon := \eta_1.$$ 
    \end{proposition}
    
    \begin{proof}
    
    For any $(s, a^\allm) \in (\bar \cS \times \cA^\allm)$, we have
    \begin{align}
        & |\tilde Q_{k-1}(s, a^\allm) - Q_{\tilde{\pi}_{k-1}}(s, a^\allm)| \nonumber \\
        &= |\tilde{w}_k^\top \phi(s, a^\allm) - Q_{\tilde{\pi}_{k-1}}(s, a^\allm)| \nonumber \\
        &= |\tilde{w}_k^\top \phi(s, a^\allm) \pm w_{\tilde{\pi}_{k-1}}^\top \phi(s, a^\allm) - Q_{\tilde{\pi}_{k-1}}(s, a^\allm)| \nonumber \\
        &\le |\tilde{w}_k^\top \phi(s, a^\allm) - w_{\tilde{\pi}_{k-1}}^\top \phi(s, a^\allm)| + |w_{\tilde{\pi}_{k-1}}^\top \phi(s, a^\allm) - Q_{\tilde{\pi}_{k-1}}(s, a^\allm)| \nonumber \\
        &\le |\tilde{w}_k^\top \phi(s, a^\allm) - w_{\tilde{\pi}_{k-1}}^\top \phi(s, a^\allm)| + \epsilon \nonumber \\
        &= |\tilde{w}_k^\top \phi(s, a^\allm) - w_{\tilde{\pi}_{k-1}}^\top \phi(s, a^\allm) \pm (m-1)\tilde{w}_k^\top \phi(s, \bar a^\allm) \pm (m-1) w_{\tilde{\pi}_{k-1}}^\top \phi(s, \bar a^\allm)| + \epsilon \nonumber \\
        &= \left|\left( \sum_{i=1}^m \tilde{w}_k^\top \phi(s, (a^\ag{i}, \bar a^{(-i)})) - w_{\tilde\pi_{k-1}}^\top \phi(s, (a^\ag{i}, \bar a^{(-i)})) \right) + (m-1)\left[w_{\tilde{\pi}_{k-1}}^\top \phi(s, \bar a^\allm)) - \tilde{w}_k^\top \phi(s, \bar a^\allm\right]\right| + \epsilon \nonumber \\
        &\le m \bar\eta + (m-1) \bar \eta + \epsilon \nonumber \\ 
        &= \bar\eta (2m-1) + \epsilon \label{value function bound 1}
    \end{align}
    where the second last inequality holds by Lemma \ref{lemma: yin lemma b.2 no epsilon} (because the features of all the state action pairs considered are in $\cH$, since $s \in \bar \cS$).
    
    While for any $(s, a^\allm) \in ((\cS \backslash \bar \cS) \times \cA^\allm)$, we have
    \begin{align}
        |\tilde Q_{k-1}(s, a^\allm) - Q_{\tilde{\pi}_{k-1}}(s, a^\allm)| 
        = |Q_{\tilde \pi_{k-1}}(s, a^\allm) - Q_{\tilde{\pi}_{k-1}}(s, a^\allm)|
        &= 0 \label{value function bound 2}
    \end{align}
    \end{proof}

    Finally, it is left to show that $|\tilde w_{k}^\T \phi(\rho, a^\allm) - Q_{\tilde \pi_{k-1}}(\rho, a^\allm)|$  can be bounded for all $a^\allm \in \cA^\allm$.
    Notice that lines 3-6 in \textsc{Confident Multi-Agent MC-LSPI} run \textsc{Uncertainty Check with DAV} with state $\rho$ as input until the returned \textsc{status} is \textsc{certain}.
    Recall that once \textsc{Uncertainty Check with DAV} returns a \textsc{status} of \textsc{certain} we know that $\rho \in \bar \cS$.
    Thus, we can immediately apply the result in \cref{value function bound 1} to bound $\eta_1 \ge |\tilde w_{k}^\T \phi(\rho, a^\allm) - Q_{\tilde \pi_{k-1}}(\rho, a^\allm)|, \ \forall a^\allm \in \cA^\allm$.
    
    \subsection{Efficient Good Set Search Approach (EGSS)} \label{subsec:good set search}
    
    In this section we prove some useful results for \textsc{Uncertainty Check with EGSS}.
    Fix a state $s \in \cS$.
    First, we show that with $\poly(d, m, A)$ computation, one can find an action vector $a^\allm \in \cA^\allm$ that approximately maximizes $\phi(s, a^\allm)^\top V_\cC^{-1} \phi(s, a^\allm)$.
    
    \begin{lemma}[Efficient good set search]
    \label{lemma:good set search}
    Assume either assumption \ref{ass: feature decomposition} or \ref{ass: argmax oracle} is satisfied.
    With all terms as defined earlier. 
    One can ensure, with computation time $2 d^2 \sum_{i=1}^m A^\ag{i}$ that either
    $$\phi(s, a^\allm)^\top V_\cC^{-1} \phi(s, a^\allm) \le d\tau$$
    for all $a^\allm \in \cA^\allm$, or there exists an $a^\allm \in \cA^\allm$ such that
    $$\phi(s, a^\allm)^\top V_\cC^{-1} \phi(s, a^\allm) > \tau$$
    \end{lemma}
    
    \begin{proof}
    Recall that we are able to compute $\max_{a^\allm \in \cA^\allm} \langle u, \phi(s, a^\allm) \rangle$ for any $u \in \RR^d$ in $\poly(d ,m, A)$ time (due to assumption \ref{ass: feature decomposition} or \ref{ass: argmax oracle}).
    Now, we make use of a bi-directional 2-norm to $\infty$-norm inequality that will take advantage of the above mentioned efficient computation.
    
    Fix $\cC$ and define the lower triangular matrix $L$ via the Cholesky decomposition $V_\cC^{-1} = L L^\top$.
    Define $\{e_i\}_{i=1}^d$ as the standard basis vectors and 
    \begin{equation*}
    (v^*, a_\text{max}^\allm) := \text{arg} \left(\max_{v \in \{\pm e_i\}_{i=1}^d} \max_{a^\allm \in \cA^\allm} \langle L v, \phi(s, a^\allm) \rangle \right) 
    \end{equation*}
    Then we have that
    \begin{align}
       \frac{1}{d} \| \phi(s, a_\text{max}^\allm) \|_{V_\cC^{-1}}^2
       &= \frac{1}{d} \phi(s, a_\text{max}^\allm)^\top V_\cC^{-1} \phi(s, a_\text{max}^\allm) \nonumber \\
       &= \frac{1}{d} \phi(s, a_\text{max}^\allm)^\top L L^\top \phi(s, a_\text{max}^\allm) \nonumber \\
       &= \frac{1}{d} \| L^\top \phi(s, a_\text{max}^\allm) \|_2^2 \nonumber \\
       &\le \max_{a^\allm \in \cA^\allm} \| L^\top \phi(s, a^\allm)\|_\infty^2 \nonumber \\
       &= \max_{v \in \{\pm e_i\}_{i=1}^d} \max_{a^\allm \in \cA^\allm} \langle v, L^\top \phi(s, a^\allm) \rangle^2 \nonumber \\
       &= \max_{v \in \{\pm e_i\}_{i=1}^d} \max_{a^\allm \in \cA^\allm} \langle L v, \phi(s, a^\allm) \rangle^2 \nonumber \\
       &= \langle L v^*, \phi(s, a^\allm_\text{max}) \rangle^2  \label{inf-norm term} \\
       &\le \| L^\top \phi(s, a_\text{max}^\allm) \|_2^2 \nonumber 
    \end{align}
    Notice that the purpose of writing all the equalities up to equation (\ref{inf-norm term}) was to show that equation (\ref{inf-norm term}) can be computed in $\poly(d, m, A)$ time. 
    Since $\max_{a^\allm \in \cA^\allm} \langle L v, \phi(s, a^\allm) \rangle^2$ can be computed in $\poly(d, m, A)$ time (due to assumption \ref{ass: feature decomposition} or \ref{ass: argmax oracle}) and $\{\pm e_i\}_{i=1}^d$ contains $2d$ elements. 
    Also, note that $L$ can be computed with at most $d^2$ computation in each loop by doing a rank one update to the Cholesky decomposition of $V_\cC^{-1} = L L^\top$.
    
    If equation (\ref{inf-norm term}) is larger than $\sqrt{\tau}$, then $\|\phi(s, a_\text{max}^\allm) \|_{V^{-1}}^2 > \tau$.
    While, if equation (\ref{inf-norm term}) is less than or equal $\tau$, then $\|\phi(s, a_\text{max}^\allm) \|_{V^{-1}}^2 \le d\tau$, completing the proof.
    \end{proof}

    \textsc{Uncertainty Check with EGSS} is essentially an implementation of equation (\ref{inf-norm term}), thus it only takes computation $\poly(d, m, A)$ to run, as stated in Lemma \ref{lemma:good set search}.
    Also, since only $a^\allm \in \cA^\allm$ that satisfy $\|\phi(s, a^\allm)\|_{V_\cC^{-1}}^2 \ge \|\phi(s, a^\allm)\|_\infty^2 > \tau$ are added to the core set, we can still use \cref{lemma:bound on core set size} to bound the size of the core set by $C_\text{max}$.
    Basically, \ref{inf-norm term} is an underestimate of $\| \phi(s, a_\text{max}^\allm) \|_{V^{-1}}^2$ and we only add elements to the core set when it is larger than $\tau$, thus the core set is no larger than it was when using \textsc{Uncertainty Check}.
    
    Now, we aim to ensure that the VA and MA behave identically in the final loop.
    Notice that \textsc{Uncertainty Check with EGSS} provides a weaker guarantee than \textsc{Uncertainty Check}, when the returned \textsc{result} is \textsc{certain}.
    Specifically, when \textsc{Uncertainty Check with EGSS} returns a \textsc{result} of \textsc{certain}, then Lemma \ref{lemma:good set search} guarantees that $\|\phi(s, a^\allm)\|_{V_\cC^{-1}}^2 \le d\tau$ for all $a^\allm \in \cA^\allm$.
    While when the \textsc{Uncertainty Check} returns a \textsc{result} of \textsc{certain}, then $\|\phi(s, a^\allm)\|_{V_\cC^{-1}}^2 \le \tau$ for all $a^\allm \in \cA^\allm$.
    Thus, we define a smaller good set $\cH_d = \{ \phi(s, a^\allm): \|\phi(s, a^\allm)\|_{V_\cC^{-1}}^2 \le d\tau\}$.
    
    Redefine the VA's $Q$-function at iteration $k$ as
    \begin{equation*}
        \tilde{Q}_{k-1}(s, a^\allm)=\begin{cases}
                          \tilde{w}_k^\top \phi(s, a^\allm) \quad &\text{if} \, \phi(s, a^\allm) \in \cH_d \\
                          Q_{\tilde{\pi}_{k-1}}(s, a^\allm)     \quad &\text{if} \, \phi(s, a^\allm) \notin \cH_d \\
                    \end{cases}
    \end{equation*}
    and VA's policy as
    \begin{equation*}
        \tilde \pi_k(a^\allm | s) = \mathbbm{1} \left( a^\allm = \argmax_{\tilde a^\allm \in \cA^\allm} \tilde Q_{k-1}(s, \tilde a^\allm) \right).
    \end{equation*}
    Notice that in the final loop \textsc{Uncertainty Check with DAV} always returns a \textsc{result} of \textsc{certain}, and thus we are sure that all $a^\allm \in \cA^\allm$ for all the states encountered in the final loop are in the smaller good set $\cH_d$.
    Thus, the VA's policy $\pi_{k}$ would always be greedy w.r.t. $\tilde w_k^\top \phi$ in the final loop.
    This ensures that the VA and MA behave identically in the final loop.

    We need show that we can bound $\|\tilde Q_{k-1}(s, a^\allm) - Q_{\tilde \pi_{k-1}}(s, a^\allm)\|_\infty$ with this new definition of $\tilde Q_{k-1}$. 
    First we state a slight modification of \cref{lemma: yin lemma b.2} that holds for the smaller good set $\cH_d$ 
    
    \begin{lemma}[EGSS modified Lemma B.2 from \cite{yin2021efficient}]
    \label{lemma: mod b.2 egss}
    Suppose that Assumption \ref{ass: feature decomposition} holds. 
    With all terms as defined earlier and $\theta > 0$. 
    Then, with probability at least 
    $$1 - 2C_{\text{max}} \exp(-2 \theta^2(1-\gamma)^2 n)$$
    for any $(s,  a^\allm) \in (\cS \times \cA^\allm)$ pair such that $\phi(s, a^\allm) \in \cH_d$, we have 
    $$|\tilde{w}_k^\top \phi(s, a^\allm) - w_{\tilde\pi_{k-1}}^\top \phi(s, a^\allm)| \le b\sqrt{\lambda d \tau} + \left(\epsilon + \frac{\gamma^{H+1}}{1 - \gamma} + \theta \right) \sqrt{d \tau C_{\text{max}}} + \epsilon = \sqrt{d} \bar \eta:= \eta_2$$
    \end{lemma}
    
    \begin{proof}
    The proof is identical to that of Lemme B.2 from \cite{yin2021efficient} except $\tau$ is replaced with $d \tau$ everywhere, due to the weaker guarantee of algorithm \ref{alg:uncertainty check egss} as discussed above. 
    \end{proof}
    
    Essentially we get an extra $\sqrt{d}$ factor due to the smaller good set $\cH_d$. 
    Since the VA's policy $\tilde \pi_{k}$ has access to the true $Q$-function $Q_{\tilde \pi_{k-1}}$ for all $\phi(s, a^\allm) \notin \cH_d$,
    Now we show that $\|\tilde Q_{k-1}(s, a^\allm) - Q_{\tilde \pi_{k-1}}(s, a^\allm)\|_\infty$ can be bounded.
    
    \begin{proposition}[approximate value function bound for EGSS]
    \label{prop: approx value function bound for EGSS}
    Suppose that Assumption \ref{ass: feature decomposition} holds. 
    With all terms as defined earlier and $\theta > 0$. 
    Then, with probability at least 
    $$1 - 2C_{\text{max}} \exp(-2 \theta^2(1-\gamma)^2 n)$$
    we have
    $$\|\tilde Q_{k-1}(s, a^\allm) - Q_{\tilde{\pi}_{k-1}}(s, a^\allm)\|_\infty \le \eta_2.$$ 
    \end{proposition}

    \begin{proof}
    For any $(s, a^\allm) \in (\cS \times \cA^\allm)$ such that $\phi(s, a^\allm) \in \cH_d$, we have
    \begin{align}
        |\tilde Q_{k-1}(s, a^\allm) - Q_{\tilde{\pi}_{k-1}}(s, a^\allm)| \le \eta_2
    \end{align}
    by \cref{prop: approx value function bound for EGSS}.
    While for any $(s, a^\allm) \in (\cS \times \cA^\allm)$ such that $\phi(s, a^\allm) \notin \cH_d$, we have
    \begin{align}
        |\tilde Q_{k-1}(s, a^\allm) - Q_{\tilde{\pi}_{k-1}}(s, a^\allm)| 
        = |Q_{\tilde \pi_{k-1}}(s, a^\allm) - Q_{\tilde{\pi}_{k-1}}(s, a^\allm)|
        &= 0 
    \end{align}
    \end{proof}

    Finally, it is left to show that $|\tilde w_{k}^\T \phi(\rho, a^\allm) - Q_{\tilde \pi_{k-1}}(\rho, a^\allm)|$  can be bounded for all $a^\allm \in \cA^\allm$.
    Notice that lines 3-6 in \textsc{Confident Multi-Agent MC-LSPI} run \textsc{Uncertainty Check with EGSS} with state $\rho$ as input until the returned \textsc{status} is \textsc{certain}.
    Recall that once \textsc{Uncertainty Check with EGSS} returns a \textsc{status} of \textsc{certain} we know that $\rho \in \cH_d$.
    Thus, we can immediately apply \cref{lemma: mod b.2 egss} to bound $\eta_2 \ge |\tilde w_{k}^\T \phi(\rho, a^\allm) - Q_{\tilde \pi_{k-1}}(\rho, a^\allm)|, \ \forall a^\allm \in \cA^\allm$.

    \subsection{Extending to Politex} \label{subsec:extending to politex}
    Recall the above results where for the \textsc{Confident Multi-Agent MC-LSPI} case.
    It turns out the story for the \textsc{Confident Multi-Agent MC-Politex} is extremely similar and can be argued in nearly the same way. 
    The main difference is that the policy used in \textsc{Confident Multi-Agent MC-Politex} is different than in \textsc{Confident Multi-Agent MC-LSPI} (line 15).
    As such, we can no longer use \cref{lemma:approximate policy iteration} (since it relied on a greedy policy) and, thus cannot use \cref{prop:optimality of output policy} to bound the sub-optimality of the policy output by \textsc{Confident Multi-Agent MC-Politex}. 
    Next, we show there is a similar Lemma and Proposition that can derived for \textsc{Confident Multi-Agent MC-Politex}.

    Recall that we do not use clipping on the $Q$-functions in \textsc{Confident Multi-Agent MC-Politex}, so that we can sample from the policy efficiently (Proposition \ref{prop: efficient politex policy sampling}).  
    This means we must define the VA's $Q$-function differently from \citep{yin2021efficient}, by removing clipping from the case when $\phi(s, a^\allm) \in \cH$.
    \begin{equation*}
        \tilde{Q}_{k-1}(s, a^\allm)=\begin{cases}
                          \tilde{w}_k^\top \phi(s, a^\allm) \quad &\text{if} \, \phi(s, a^\allm) \in \cH \\
                          Q_{\tilde{\pi}_{k-1}}(s, a^\allm)     \quad &\text{if} \, \phi(s, a^\allm) \notin \cH \\
                    \end{cases}
    \end{equation*}
    Then the VA's policy is
    \begin{equation} \label{eq:politex virtual policy}
        \tilde \pi_k(a^\allm | s) \propto \exp \left( \alpha \sum_{j=0}^{k-1} \tilde Q_{j}(s, a^\allm) \right).
    \end{equation}
    
    Also, due to no clipping, the sequence of $Q$-functions during policy iteration is now in the $[-\eta, (1-\gamma)^{-1} + \eta]$ interval, where $\eta \ge \|\tilde Q_{k-1}(s, a^\allm) - Q_{\tilde \pi_{k-1}}(s, a^\allm)\|_\infty$.
    We now restate Lemma D.1 from \citet{yin2021efficient} which bounds the mixture policy output by Politex for an arbitrary sequence of $Q$-functions
    Since we do not use clipping the theorem is slightly modified (we replace the interval $[0, (1-\gamma)^{-1}]$ with a general interval $[a, b], \ a, b \in \RR$, which can be extracted from the calculations in \citet{szepesvari2022}).

    \begin{lemma}[modified Lemma D.1 in \citet{yin2021efficient} also in \citet{szepesvari2022}] \label{lemma:politex mixture policy bound}
    Given an initial policy $\pi_0$, a sequence of functions $Q_k: \cS \times \cA^\allm \to [a, b], \ k \in [K-1], a, b \in \RR$, and $Q_{\pi^*} \in [0, 1/(1-\gamma)]$, construct a sequence of policies $\pi_1, ..., \pi_{K-1}$ according to (\ref{eq:politex virtual policy}) with $\alpha = 1/(b-a) \sqrt{\frac{2 \log(|\cA^\allm|)}{K}}$, then, for any $s \in \cS$, the mixture policy $\bar \pi_{K-1} \sim \text{Unif}\{\pi_k\}_{k=0}^{K-1}$ satisfies

    \begin{equation}
        V^*(s) - V_{\bar \pi_K}(s) \le \frac{b-a}{(1-\gamma)}\sqrt{\frac{2 \log(|\cA^\allm|)}{K}} + \frac{2 \max_{0 \le k \le K-1} \|Q_k - Q_{\pi_k}\|_\infty}{1 - \gamma}
    \end{equation}
    \end{lemma}

    Notice that the above result suggests we just need to control the term $\|Q_k - Q_{\pi_k}\|_\infty$.
    For the VA this is $\|\tilde Q_k - Q_{\tilde \pi_k}\|_\infty$ and as we have already seen, this can be bounded using the high probability bound on policy evaluation for \textsc{Uncertainty Chcek with DAV} (Proposition \ref{prop: approx value function bound for DAV}) and \textsc{Uncertainty Chcek with EGSS} (Proposition \ref{prop: approx value function bound for EGSS}).
    Using Lemma \ref{lemma:politex mixture policy bound} instead of Lemma D.1 in \citet{yin2021efficient}, one can extract another slightly modified result from \citet{yin2021efficient}.
    
    \begin{proposition}[equation (D.8) in \citet{yin2021efficient}] \label{prop:politex optimality of output policy}
    With all terms as defined earlier. 
    Define $\eta \ge \|\tilde Q_{k-1}(s, a^\allm) - Q_{\tilde \pi_{k-1}}(s, a^\allm)\|_\infty$. 
    Suppose $\eta \ge |\tilde w_{k}^\T \phi(\rho, a^\allm) - Q_{\tilde \pi_{k-1}}(\rho, a^\allm)|_\infty, \ \forall a^\allm \in \cA^\allm$.
    Then, if the VA and MA behave identically in the final loop, with probability at least $1 - 4KC_{\text{max}}^2 \exp(-2 \theta^2(1-\gamma)^2 n)$ we have
    \begin{align}
            V^*(s) - V_{\bar \pi_{K-1}}(\rho) \le \frac{b-a}{(1-\gamma)} \sqrt{\frac{2 \log(|\cA^\allm|)}{K}} + \frac{4 \eta}{1 - \gamma}
    \end{align}
    \end{proposition}
    
    Notice, that we require the same three things as in the \textsc{Confident Multi-Agent MC-LSPI} case (Proposition \ref{prop:optimality of output policy}).
    We need a bound on $\|\tilde Q_{k-1}(s, a^\allm) - Q_{\tilde \pi_{k-1}}(s, a^\allm)\|_\infty$. 
    We need a bound on $|\tilde w_{k}^\T \phi(\rho, a^\allm) - Q_{\tilde \pi_{k-1}}(\rho, a^\allm)|_\infty, \ \forall a^\allm \in \cA^\allm$.
    We need to ensure that the VA and MA behave identically in the final loop. 
    Then, we can get a bound on the sub-optimality of the MA's output policy $\bar \pi_{K-1}$.
    Using the same steps as in the previous sections, one can verify that indeed, \textsc{Confident Multi-Agent MC-Politex} with \textsc{Uncertainty Check with DAV} or \textsc{Uncertainty Check with EGSS} does satisfy the above three conditions, with $\eta = \eta_1$ ($\eta_1$ as defined in \ref{prop: approx value function bound for DAV}) and $\eta = \eta_2$ ($\eta_2$ as defined in \ref{prop: approx value function bound for EGSS}) respectively.
    
    We bound $|\cA^\allm| \le A^m$. 
    We can replace $b-a$ with $1/(1-\gamma) + 2\eta$, since $w^\T \phi(s, a^\allm) \in [-\eta, (1-\gamma)^{-1} + \eta], \ \forall (s \times a^\allm) \in (\cS \times \cA^\allm)$ in the final loop for the same event which holds with probability at least $1 - 4KC_{\text{max}}^2 \exp(-2 \theta^2(1-\gamma)^2 n)$ in \cref{prop:politex optimality of output policy}. 
    We get with probability at least $1 - 4KC_{\text{max}}^2 \exp(-2 \theta^2(1-\gamma)^2 n)$ that
    \begin{align} 
            V^*(s) - V_{\bar \pi_{K-1}}(\rho) \le \left(\frac{1}{(1-\gamma)^2} + \frac{2\eta}{(1-\gamma)}\right) \sqrt{\frac{2 m \log(A)}{K}} + \frac{4 \eta}{1 - \gamma}. \label{eq:politex actual optimality of output policy}
    \end{align}

    \section{Proofs of Theorems} \label{app: theorem proofs}
    We make a remark on the query complexity of \textsc{Confident Multi-Agent MC-LSPI/Politex}.
    From \cref{lemma:bound on core set size} we know the core set size is bounded by $C_\text{max} = \tilde \cO(d)$.
    The total number of times Policy iteration is thus at most $C_\text{max}$.
    Each run of policy iteration can take as much as $K$ iterations.
    In each iteration \textsc{MA-Confident Rollout} is run at most $C_\text{max}$ times.
    \textsc{MA-Confident Rollout} does $n$ rollouts of length $H$ which queries the simulator once for each step.
    In total the number of queries performed by \textsc{Confident Multi-Agent MC-LSPI/Politex} is bounded by $C_\text{max}^2 K n H$.
    This equation is used to calculate the query cost for the different variants of \textsc{Confident Multi-Agent MC-LSPI/Politex}, once all the parameter values have been calculated.


    \subsection{Proof of \cref{thm:ma-mc-lspi sub-optimality}}
    Plugging in $\eta=\eta_1$ when \textsc{Uncertainty Check with DAV} is used ($\eta_1$ as defined in \ref{prop: approx value function bound for DAV}) and $\eta=\eta_2$ when \textsc{Uncertainty Check with EGSS} is used ($\eta_2$ as defined in \ref{prop: approx value function bound for EGSS}) into Proposition \ref{prop:optimality of output policy}.
    Setting $z=2m-1$ when \textsc{Uncertainty Check with DAV} is used, and $z = \sqrt{d}$ when \textsc{Uncertainty Check with EGSS} is used.
    Suppose assumption \ref{ass: feature decomposition} is satisfied with $\epsilon=0$.
    By choosing appropriate parameters according to $\delta$ and $\kappa$, we can ensure that with probability at least $1 - \delta$ that the policy output by \textsc{Confident Multi-Agent MC-LSPI} $\pi_{K-1}$ satisfies:
    \begin{align*}
        V^*(\rho) - V_{\pi_{K-1}}(\rho) \leq \kappa,
    \end{align*}
    with the following parameter initialization (see \ref{sec:parameter mclsipi-dav})
    \begin{align*}
        \tau &= 1\\
        \lambda &= \frac{\kappa^2(1 - \gamma)^4}{1024 b^2 z^2}\\
        \theta &= \frac{\kappa(1- \gamma)^2}{32 z \sqrt{C_{\text{max}}}}\\
        H &= \frac{
            \log \left ( 32 \sqrt{C_{\text{max}}} z \right)
            - \log \left( \kappa(1 - \gamma)^3 \right)
        }{
            1-\gamma
        } - 1\\
        K &= \frac{\log\left(\frac{1}{\kappa(1 - \gamma)^2}\right) + \log(8)}{1-\gamma} + 1 \\
        n &= \frac{\log(\delta) - \log(4KC_{\text{max}}^2)}{2 \theta^2(1-\gamma)^2} \\
        C_{\max} &= \frac{e}{e-1} \frac{1 + \tau}{\tau} d \left( 
            \log(1 + \frac{1}{\tau}) +
            \log(1 + \frac{m}{\lambda})
        \right) 
    \end{align*}
    with computational cost of $\poly(d, \frac{1}{1 - \gamma}, \frac{1}{\kappa}, \log(\frac{1}{\delta}), \log(b), m, |\mathcal{A}|)$.
    and query cost $\cO\left(\tfrac{z^2 d^3}{\kappa^2 (1-\gamma)^8} \right)$ 

    Suppose assumption \ref{ass: feature decomposition} is satisfied with $\epsilon \neq 0$,
    By choosing parameters as above, with $\kappa = \frac{32 \epsilon \sqrt{d} z}{(1-\gamma)^2} (1 + \log(m b^2 \epsilon^{-2} d^{-1}))^{1/2}$, we can ensure that with probability of at least $1 - \delta$ that the policy output by \textsc{Confident Multi-Agent MC-LSPI} $\pi_{K-1}$ satisfies:
    
    $$V^*(\rho) - V_{\pi_{K-1}}(\rho) \leq \frac{64 \epsilon \sqrt{d} z}{(1-\gamma)^2} (1 +\log(1+m b^2 \epsilon^{-2} d^{-1}))^{1/2}$$
    
    with computational cost of $\poly(d, \frac{1}{1 - \gamma}, \frac{1}{\kappa}, \log(\frac{1}{\delta}), \log(b), m, |\mathcal{A}|)$.
    and query cost $\cO\left(\tfrac{d^2}{\epsilon^2 (1-\gamma)^4} \right)$ 
    
    Moreover, the above results also holds under \cref{ass: argmax oracle} when \textsc{Uncertainty Check with EGSS} is used.

    \subsection{Proof of \cref{thm:ma-mc-politex sub-optimality}}
    Plugging in $\eta=\eta_1$ when \textsc{Uncertainty Check with DAV} is used ($\eta_1$ as defined in \ref{prop: approx value function bound for DAV}) and $\eta=\eta_2$ when \textsc{Uncertainty Check with EGSS} is used ($\eta_2$ as defined in \ref{prop: approx value function bound for EGSS}) into \cref{eq:politex actual optimality of output policy}.
    Setting $z=2m-1$ when \textsc{Uncertainty Check with DAV} is used, and $z = \sqrt{d}$ when \textsc{Uncertainty Check with EGSS} is used.
    Suppose assumption \ref{ass: feature decomposition} is satisfied with $\epsilon=0$.
    By choosing appropriate parameters according to $\delta$ and $\kappa$, we can ensure that with probability at least $1 - \delta$ that the policy output by \textsc{Confident Multi-Agent MC-Politex} $\pi_{K-1}$ satisfies:
    \begin{align*}
        V^*(\rho) - V_{\pi_{K-1}}(\rho) \leq \kappa,
    \end{align*}
    with the following parameter initialization (see \ref{sec:parameter mclsipi-dav})
    \begin{align*}
        \tau &= 1\\
        \lambda &= \frac{\kappa^2(1 - \gamma)^2}{576 b^2 z^2}\\
        \theta &= \frac{\kappa(1- \gamma)}{24 z \sqrt{C_{\text{max}}}}\\
        H &= \frac{
            \log \left ( 24 \sqrt{C_{\text{max}}} z \right)
            - \log \left( \kappa(1 - \gamma)^2 \right)
        }{
            1-\gamma
        } - 1\\
        K &= 2m \log(A) \left( \frac{4}{\kappa^2 (1-\gamma)^4} + \frac{3}{\kappa (1-\gamma)^2} + \frac{9}{16} \right)\\
        n &= \frac{\log(\delta) - \log(4KC_{\text{max}}^2)}{2 \theta^2(1-\gamma)^2} \\
        C_{\max} &= \frac{e}{e-1} \frac{1 + \tau}{\tau} d \left( 
            \log(1 + \frac{1}{\tau}) +
            \log(1 + \frac{m}{\lambda})
        \right) 
    \end{align*}
    with computational cost of $\poly(d, \frac{1}{1 - \gamma}, \frac{1}{\kappa}, \log(\frac{1}{\delta}), \log(b), m, |\mathcal{A}|)$.
    and query cost $\cO\left(\tfrac{m z^2 d^3}{\kappa^4 (1-\gamma)^9} \right)$ 

    Suppose assumption \ref{ass: feature decomposition} is satisfied with $\epsilon \neq 0$,
    By choosing parameters as above, with $\kappa = \frac{16 \epsilon \sqrt{d} z}{(1-\gamma)} (1 + \log(m b^2 \epsilon^{-2} d^{-1}))^{1/2}$, we can ensure that with probability of at least $1 - \delta$ that the policy output by \textsc{Confident Multi-Agent MC-Politex} $\pi_{K-1}$ satisfies:
    
    $$V^*(\rho) - V_{\pi_{K-1}}(\rho) \leq \frac{32 \epsilon \sqrt{d} z}{(1-\gamma)^2} (1 +\log(1+m b^2 \epsilon^{-2} d^{-1}))^{1/2}$$

    with computational cost of $\poly(d, \frac{1}{1 - \gamma}, \frac{1}{\kappa}, \log(\frac{1}{\delta}), \log(b), m, |\mathcal{A}|)$.
    and query cost $\cO\left(\tfrac{m d}{\epsilon^4 (1-\gamma)^5} \right)$

    \section{Kernel Setting} \label{app:kernel setting}
    
    The kernelized setting is a standard extension of the finite-dimensional linear setup \citep{srinivas2009gaussian,abbasi2012online}. It lifts the restriction that features and parameter vector are elements of $\bR^d$. Instead we require that the $Q_\pi$-function is contained in a reproducing kernel Hilbert space (RKHS). This includes cases where the linear dimension of function class is infinite.
    
    The more general setup requires us to address two main challenges: First, the scaling of the sample complexity with the dimension $d$ needs to be improved to a notion of effective dimension that can be bounded for the RKHS of interest. Second, computationally we cannot directly work with infinite dimensional features $\phi(s,a)$. Instead, we need to rely on the `kernel trick' and compute all quantities of interest in the finite-dimensional data space. 
    
    Formally for each agent $j \in [m]$, the function $k^{(j)} : (\cS \times \cA^\allm)^2 \rightarrow \bR$ is defined as
    
    \begin{align}
        k^{(j)}(s_1, a_1^\allm, s_2, a_2^\allm) = k_j(s_1, a_1^{(j)}, s_2, a_2^{(j)}), \label{eq:rkhs_j}
    \end{align}
    where $k_j: (\cS \times \cA)^2 \rightarrow \bR$ is the underlying kernel function for agent $j$,
    and $\cH_j$ is the RKHS associated with it.
    
    Based on definition \eqref{eq:rkhs_j}, it's easy to see that $\{k^{(j)}\}_{j \in [m]}$ is a set of kernel functions too, and they share the same vector space which is $V := \RR^{\cS \times \cA^\allm}$.
    However, they have different inner products on this space which produce a different RKHS for every $j \in [m]$.
    We denote RKHS of $k^{(j)}$ as $\cH^{(j)}$, and its inner product follows from equation \eqref{eq:rkhs_j} as
    \begin{align}
        \langle k^{(j)}(s_1, a_1^\allm, \cdot, \cdot), k^{(j)}(s_2, a_2^\allm, \cdot, \cdot) \rangle_{\cH^{(j)}} =
        \langle k_j(s_1, a_1^{(j)}, \cdot, \cdot), k_j(s_2, a_2^{(j)}, \cdot, \cdot) \rangle_{\cH_j}. \label{eq:rkhs-inner-1}
    \end{align}
    By defining $\phi_j(s, a) := k_j(s,a, \cdot, \cdot) \in \cH_j$ and $\phi^{(j)}(s, a^\allm) := k^{(j)}(s, a^\allm, \cdot, \cdot) \in \cH^{(j)}$, we can rewrite \eqref{eq:rkhs-inner-1} for fixed $s_1, s_2, a_1^\allm, a_2^\allm$ as
    \begin{align}
        \langle \phi^{(j)}(s_1, a_1^\allm), \phi^{(j)}(s_2, a_2^\allm) \rangle_{\cH^{(j)}} &=
        \langle \phi_j(s_1, a_1^{(j)}), \phi_j(s_2, a_2^{(j)})\rangle_{\cH_j}.\label{eq:rkhs-inner-1-rev}
    \end{align}
    Intuitively, equation \eqref{eq:rkhs-inner-1-rev} suggests that the inner product $\langle \cdot, \cdot \rangle_{H^{(j)}}$ only depends on the state $s$, and the action taken by agent $j$.
    
    Next, we define the joint additive kernel $k : (\sS \times \aA^\allm)^2 \rightarrow \bR$ as follows
    \begin{align}
    	k(s_1,a_1^\allm, s_2,a_2^\allm) &= \sum_{j=1}^m k^{(j)}(s_1,a_1^\allm, s_2,a_2^\allm) \label{eq:jak-def}\\
    	&= \sum_{j=1}^m k_j(s_1,a_1^{(j)}, s_2,a_2^{(j)})\\
    	&= \sum_{j=1}^m \langle k_j(s_1,a_1^{(j)}, \cdot, \cdot), k_j(s_2,a_2^{(j)}, \cdot, \cdot) \rangle_{\cH_j},
    \end{align}
    and we denote its associated RKHS as $\cH$. Again, note that $\cH$ uses the same vector space, namely $V$, as all the $\cH^{(j)}$s.
    
    Now, we can restate Assumption for the kernel case.
    
    TODO: I think we still want to solve the problem under assumption one for the kernel case, right?\ref{ass: feature decomposition}
    
    \begin{assumption}[Assumption 1 for RKHS]
        \label{ass:kernel-2}
    	For each (deterministic) policy $\pi$, there exists
    	$f_\pi \in \cH$, such that
    	$Q_\pi(s, a^\allm) = \langle \phi(s,a^\allm), f_\pi \rangle_{\cH}$.
    \end{assumption}
    Next, we show that there exist a function $f_\pi^\ag{i} \in \cH^\ag{i}$ for $i \in [m]$, such that:
    \begin{align*}
        Q_\pi (s, a^\allm) &= \sum_{j=1}^m Q_\pi^{(j)}(s, a^\allm)\\
        Q_\pi^\ag{i} (s, a^\allm) &= \langle \phi^\ag{i}(s,a^\allm), f_\pi^\ag{i} \rangle_{\cH^\ag{i}}
    \end{align*}
    Or, there exist $f_{\pi, j} \in \cH_j$ for $j \in \cH_j$, such that:
    \begin{align*}
        Q_\pi (s, a^\allm) &= \sum_{j=1}^m Q_{\pi,j}(s, a^{(j)})\\
        Q_\pi^\ag{i} (s, a^\allm) &= \langle \phi^\ag{i}(s,a^\allm), f_{\pi,j} \rangle_{\cH^\ag{i}}
    \end{align*}
    
    \begin{proof}
        As $f_\pi$ is an element of $\cH$ we know that it can be shown based on the basis vectors of $\cH$:
        \begin{align*}
            f_\pi = \sum_{i=1}^{\infty} \alpha_i k(s_i, a_i, \cdot, \cdot).
        \end{align*}
        From the definition of the joint additive kernel and the assumption \ref{ass:kernel-2} we have:
        \begin{align}
            Q_\pi (s, a^\allm) &= \langle \phi(s,a^\allm), f_\pi \rangle_{\cH} \nonumber \\
            &= \langle \phi(s,a^\allm), \sum_{i=1}^{\infty} \alpha_i k(s_i, a_i, \cdot, \cdot) \rangle_{\cH} \nonumber \\
            &= \sum_{i=1}^{\infty} \alpha_i \langle k(s,a^\allm, \cdot, \cdot), k(s_i, a_i, \cdot, \cdot) \rangle_{\cH} \nonumber \\
            &= \sum_{i=1}^{\infty} \alpha_i k(s,a^\allm, s_i, a_i) \nonumber \\
            &= \sum_{i=1}^{\infty} \alpha_i \sum_{j=1}^m k^{(j)}(s,a^\allm, s_i, a_i) & \text{Based on }\ref{eq:jak-def}\label{eq:checkpoint}\\
            &= \sum_{j=1}^m \sum_{i=1}^{\infty} \alpha_i k^{(j)}(s,a^\allm, s_i, a_i) \nonumber \\
            &= \sum_{j=1}^m \sum_{i=1}^{\infty} \alpha_i \langle \phi^{(j)}(s,a^\allm) , \phi^{(j)}(s_i, a_i) \rangle_{\cH^{(j)}} \nonumber\\
            &= \sum_{j=1}^m \langle \phi^{(j)}(s,a^\allm) , \underbrace{\sum_{i=1}^{\infty} \alpha_i \phi^{(j)}(s_i, a_i)}_{:=E^{(j)}(f_\pi) := f_\pi^\ag{i}} \rangle_{\cH^{(j)}} \nonumber \\
            &= \sum_{j=1}^{m} \langle \phi^{(j)}(s,a^\allm), f_\pi^{(j)} \rangle_{\cH^{(j)}}. \nonumber
        \end{align}
        or from \eqref{eq:checkpoint} we have:
        \begin{align*}
            Q_\pi (s, a^\allm)
            &= \sum_{i=1}^{\infty} \alpha_i \sum_{j=1}^m k^{(j)}(s,a^\allm, s_i, a_i) \\
            &= \sum_{i=1}^{\infty} \alpha_i \sum_{j=1}^m k_j(s,a^{(j)}, s_i, a_i^{(j)}) & \text{Based on }\ref{eq:rkhs_j}\\
            &= \sum_{i=1}^{\infty} \alpha_i \sum_{j=1}^m \langle \phi_j(s,a^{(j)}), \phi_j(s_i, a_i^{(j)}) \rangle_{\cH_j}\\
            &= \sum_{j=1}^m \langle
            \phi_j(s,a^{(j)}),
            \underbrace{
            \sum_{i=1}^{\infty} \alpha_i \phi_j(s_i, a_i^{(j)})
            }_{:=E_j(f_\pi) :=f_{\pi, j}}
            \rangle_{\cH_j}\\
            &= \sum_{j=1}^m \langle
            \phi_j(s,a^{(j)}),
            f_{\pi, j}
            \rangle_{\cH_j}.
        \end{align*}
        We may need to show that $f_\pi^{j}$ and $f_{\pi, j}$ have finite norms in their corresponding Hilbert spaces.
    \end{proof}
    
    \paragraph{Kernelized Algorithm}
    
    As before we can compute the ridge estimate
    \begin{align}
    	\hat Q_t = \argmin_{Q \in \cH} \sum_{(s,a^\allm)\in\cC_t} (Q(s,a^\allm) - q_{(s,a^\allm)})^2 + \lambda \|Q\|_{\cH}^2 = (\Phi_{\cC_t}\Phi_{\cC_t}^\T + \lambda \eye_\hH)^{-1}\Phi_{\cC_t}q_{\cC_t}
    \end{align}
    Here, $\eye_{\cH} : \hH \rightarrow \hH$ is the identity mapping, and $\Phi_{\cC}^\T$ can be formally defined as map $\Phi_\cC^\T : \cH \rightarrow \bR^{|\cC|}, f \mapsto [f(s,a^\allm)]_{(s,a^\allm) \in \cC}, \, f \in \cH$; and $\Phi_{\cC} : \bR^{|\cC|} \rightarrow \hH$ is the adjoint of $\Phi_{\cC}^\T$.
    
    Using the `kernel trick' we express the estimator as follows
    \begin{align}
    	\hat Q_t = \Phi_{\cC_t}(K_{\cC_t} + \lambda \eye_{t})^{-1}q_{\cC_t}
    \end{align}
    where $K_{\cC_t} = \Phi_{\cC_t}^\T \Phi_{\cC_t} \in \bR^{t \times t}$ is the kernel matrix. Lastly, we can evaluate for any $s,a^\allm$:
    \begin{align}
    	\hat Q_t(s,a^\allm) = k_{\cC_t}(s,a^\allm)^\T(K_{\cC_t} + \lambda \eye_{t})^{-1}q_{\cC_t}
    \end{align}
    where we defined $k_\cC(s,a^\allm) = [k(s,a^\allm, s',a'^\allm)]_{(s',a'^\allm) \in \cC} \in \bR^{|\cC|}$ (for some fixed ordering of $\cC$). Importantly, the last display only involves finite-dimensional quantities that can be computed from kernel evaluations. Moreover, since $k(s,a^\allm,s',a'^\allm) = \sum_{j=1}^m k_j(s,a^{(j)}, s', a^{(j)})$ we can write
    \begin{align}
    	\hat Q_t(s,a^\allm) = \sum_{j=1}^m k_{j, \cC_t}(s,a^{(j)})^\T(K_{\cC_t} + \lambda \eye_{t})^{-1}q_{\cC_t}
    \end{align}
    where $k_{j, \cC}(s,a^{(j)}) = [k_j(s,a^{(j)}, s^\prime , a^{\prime(j)})]_{(s^\prime, a^{\prime(1:m)}) \in \cC} \in \bR^{|\cC|}$. Hence we can still compute the maximizer independently for each agent.
    
    The second quantity required by the algorithm is the squared norm $\|\phi(s,a^\allm)\|_{(\Phi_\cC \Phi_{\cC}^\T + \lambda \eye_\hH)^{-1}}^2$, where now $\phi(s,a^\allm) = k(s,a^\allm, \cdot, \cdot) \in \cH$. Using the Woodbury identity, we can write 
    \begin{align}
    	\lambda (\Phi_\cC \Phi_{\cC}^\top + \lambda \eye_\hH)^{-1} = \eye_{\cH} - \Phi_{\cC} (K_\cC + \lambda \eye_{|\cC|})^{-1}\Phi_{\cC}^\T
    \end{align}
    Therefore the feature norm can be written using finite-dimensional quantities: 
    \begin{align}
    	\|\phi(s,a^\allm)\|_{(\Phi_\cC \Phi_{\cC}^\top + \lambda \eye_\hH)^{-1}}^2 = \frac{1}{\lambda} \left( k(s,a^\allm,s,a^\allm) - k_\cC(s,a^\allm)^\T(K_\cC+ \lambda \eye_{|\cC|})^{-1}k_{\cC}(s,a^\allm)\right)
    \end{align}
    With this, we can implement the DAV version of the algorithm directly. The EGSS is more tricky to implement, but this is potentially possible using eigenfunctions from Mercer's theorem.

    \paragraph{Analysis}
    Our goal next is to extend the analysis to the kernel case, carefully arguing that the linear dimension $d$ can be replaced by a more benign quantity. A common complexity measure is the total information gain, which we define as follows:
    \begin{align}
    	\Gamma_{\cC} = \log \det (\Phi_{\cC}\Phi_{\cC}^\T + \lambda \eye_d) - \log \det (\lambda \eye_d)
    \end{align}
    Note that we can compute $\Gamma_{\cC}$ for any given core set $\cC$. In the kernel case, we can compute $\Gamma_{\cC} = \log \det (\eye_{|\cC|} + \lambda^{-1} K_{\cC})$ using similar arguments as before.
    
    The maximum information gain is $\Gamma_t = \max_{\cC : |\cC|=t} \Gamma_{\cC}$. It serves as a complexity measure in the bandit literature and can be bounded for many kernels of interests \citep{srinivas2009gaussian,vakili2021information}. Following \citet{du2021bilinear}, we further define the \emph{critical information gain},
    \begin{align}
    	\tilde \Gamma = \max \{t \geq 1 : t \leq \Gamma_t \}
    \end{align}
    Note that the proof of  \cite[Lemma 5.1]{yin2021efficient} implies that $|C| \leq  \log(1+\tau)^{-1}\Gamma_{|C|}$
    
    Since the dimension $d$ enters our bounds only through $C_{\max}$ we can immediately get a sample complexity bound for the kernelized algorithm in terms of $\tilde\Gamma$. For the finite-dimensional case, \cite[Lemma 5.1]{yin2021efficient} shows that $\tilde \Gamma \leq \oO(d)$, recovering the previous bound.

    \paragraph{Unknown Critical Information Gain} Somewhat impractical for the algorithm is that we need to know a bound on $C_{\max}$ or $\tilde \Gamma$ respectively to set the number of episodes required for some target level of accuracy $\kappa > 0$ (roughly, $m = C_{\max}/\kappa^2$).
    
    As a remedy, we can replace the check $\|\phi(s,a)\|_{(\Phi_\cC \Phi_{\cC} + \lambda \eye_d)^{-1}}^2 > \tau$ by
    
    TODO: this needs some more thinking, as we don't want to set $\tau$ to be too small - maybe? An easier approach could be to set $m = |\cC|/\kappa^2$?
    
    \begin{align*}
    	\|\phi(s,a)\|_{(\Phi_\cC \Phi_{\cC} + \lambda \eye_d)^{-1}}^2 > \frac{\tilde \tau}{\max(\Gamma_{\cC},1)}
    \end{align*}
    
    Let $\cC_1, \dots, \cC_t$ be the sequence of core sets obtained by adding elements that satisfy the above condition. Note that $\Gamma(\cC_t)$ is a non-decreasing sequence. Combined with \citep[Lemma 5.1]{yin2021efficient} , this implies that
    \begin{align}
    	t \log\left(1 + \tfrac{\tilde \tau}{\max(\Gamma_t,1)}\right) \leq t \log\left(1 + \tfrac{\tilde \tau}{\max( \Gamma_{\cC_{t}}, 1)}\right) \leq \sum_{s=1}^t \log\left(1 + \tfrac{\tilde \tau}{\max(\Gamma_{\cC_{s}},1)}\right) \leq \Gamma_{\cC_{t}} \leq \Gamma_t\label{eq:cmax adaptiv}
    \end{align}
    Hence the condition is triggered at most
    \begin{align}
    	\tilde C_{\max}(\tilde \tau) = \max \left\{t \geq 1 : t \leq \Gamma_t \log\left(1 + \tfrac{\tilde \tau}{\max(\Gamma_t, 1)} \right)^{-1}\right\}
    \end{align}
    times. 
    
    TODO: Would be great to show some bounds for $\tilde C_{\max}$, e.g. in the finite-dimensional case
    
    Moreover, we can set $\tilde \tau=1$ and $m = \frac{1}{\kappa^2}$ (i.e.~without knowing a bound on $C_{\max}$) to obtain the required target accuracy $\tilde \oO(\kappa)$.
    
    TODO: this requires some introspection of \citep[Lemma B.2]{yin2021efficient} and \eqref{eq:cmax adaptiv}

    \subsection{We need to prove that if we add a state-action pair to the core set, it remains in the good set in future.}

    \begin{theorem}
    Assume that $\Phi_\cC \in \RR^{t \times d}$,  $V_\cC = \Phi_\cC^\top \Phi_\cC + \lambda I$, and $\phi_{t+1} \in \RR^d$. Define $\hat \Phi = [\Phi^\top \, \phi_{t+1}]^\top$,
    and $\hat V_\cC = \hat{\Phi}_\cC^\top \hat{\Phi}_\cC + \lambda I$.
    Then we have:
    \begin{align*}
        \| \phi_{t+1} \|_{\hat V_\cC^{-1}} < 1
    \end{align*}
    \end{theorem}
    \begin{proof}
    By the definition of the norm we have:
    \[
    \begin{aligned}
        \| \phi_{t+1} \|_{\hat{V}_\cC^{-1}} &= \phi_{t+1}^\top \hat{V}_\cC^{-1} \phi_{t+1}\\ 
        &= \phi_{t+1}^\top \left( 
            \hat{\Phi}_\cC^\top \hat{\Phi}_\cC + \lambda I
        \right)^{-1} \phi_{t+1}\\
        &= \phi_{t+1}^\top \left ( 
            \sum_{i = 1}^{t+1} \phi_i \phi_i^\top
            + \lambda I
        \right)^{-1} \phi_{t+1}\\
        &= \phi_{t+1}^\top \left ( 
            \sum_{i = 1}^{t} \phi_i \phi_i^\top
            + \lambda I
            + \phi_{t+1} \phi_{t+1}^\top
        \right)^{-1} \phi_{t+1}\\
        &= \phi_{t+1}^\top \left ( 
            V_\cC
            + \phi_{t+1} \phi_{t+1}^\top
        \right)^{-1} \phi_{t+1}\\
        &= \phi_{t+1}^\top \left ( 
            V_\cC^{-1}
            - \frac{V_\cC^{-1} \phi_{t+1} \phi_{t+1}^\top V_\cC^{-1}}{1 + \phi_{t+1}^\top V_\cC^{-1} \phi_{t+1}^\top}
        \right) \phi_{t+1} & \text{Sherman-Morrison}\\
        &= \phi_{t+1}^\top V_\cC^{-1} \phi_{t+1}
            - \frac{
            \phi_{t+1}^\top V_\cC^{-1} \phi_{t+1} \phi_{t+1}^\top V_\cC^{-1}\phi_{t+1}
            }{
            1 + \phi_{t+1}^\top V_\cC^{-1} \phi_{t+1}
            }
        \\
        &= \frac{
            \phi_{t+1}^\top V_\cC^{-1} \phi_{t+1} 
            }{
            1 + \phi_{t+1}^\top V_\cC^{-1} \phi_{t+1}
            }\\
        &< 1
        \\
    \end{aligned}
    \]
    \end{proof}
    Therefore, if we set $\tau \geq 1$, then none of the features that have been added to the core set can produce $V_\cC^{-1}$-norm greater than $\tau$, so they remain in the good set.

\section*{Parameter Assignments} \label{app:parameter assignments}

\subsection{MCLSPI-DAV}
\label{sec:parameter mclsipi-dav}
The total error is the following:
\begin{align}
    \frac{8 \eta_1}{(1 - \gamma)^2} + \frac{2 \gamma^{K - 1}}{(1 - \gamma)^2} &\leq \kappa \label{eq:egss-main} \\
    \frac{8}{(1 - \gamma)^2}
    \left(
        b\sqrt{\lambda  \tau} + \left( \frac{\gamma^{H+1}}{1 - \gamma} + \theta \right) \sqrt{ \tau C_{\text{max}}}
        \right)(2m-1) +
        \frac{2\gamma^{K - 1}}{(1 - \gamma)^2}
    &\leq \kappa \nonumber \\
    &\Rightarrow \nonumber \\
    \frac{8(2m-1)}{(1 - \gamma)^2} b\sqrt{\lambda  \tau} & \leq \frac{\kappa}{4} \label{eq:egss-lambda}\\
    \frac{8\sqrt{ \tau C_{\text{max}}}(2m-1)}{(1 - \gamma)^2} \frac{\gamma^{H+1}}{1 - \gamma} & \leq \frac{\kappa}{4} \label{eq:egss-H}\\
    \frac{8\sqrt{ \tau C_{\text{max}}}(2m-1)}{(1 - \gamma)^2} \theta & \leq \frac{\kappa}{4} \label{eq:egss-theta}\\
    \frac{2\gamma^{K - 1}}{(1 - \gamma)^2} & \leq \frac{\kappa}{4} \label{eq:egss-K}
\end{align}
First we assume that $\tau = 1$.

From \eqref{eq:egss-lambda} we get:
\begin{align*}
    \frac{8(2m - 1)}{(1 - \gamma)^2} b\sqrt{\lambda } & \leq \frac{\kappa}{4} \\
    \sqrt{\lambda } &\leq \frac{(1 - \gamma)^2\kappa}{32 b (2m -1)}\\
    \lambda &\leq \frac{(1 - \gamma)^4\kappa^2}{1024 b^2 (2m -1)^2}\\
\end{align*}
From \eqref{eq:egss-H} we get:
\begin{align*}
    \frac{8\sqrt{ C_{\text{max}}}(2m -1)}{(1 - \gamma)^3} \gamma^{H+1} & \leq \frac{\kappa}{4} \\
    (2m -1) \sqrt{C_{\text{max}}} \gamma^{H+1} & \leq \frac{\kappa(1 - \gamma)^3}{32} \\
    \gamma^{H+1} & \leq \frac{\kappa(1 - \gamma)^3}{32 \sqrt{C_{\text{max}}} (2m - 1)} \\
    H & \geq \frac{\log\left (\frac{\kappa(1 - \gamma)^3}{32 \sqrt{C_{\text{max}}} (2m - 1)} \right)}{
        \log(\gamma)
    } - 1\\
\end{align*}

From \eqref{eq:egss-theta} we get:
\begin{align*}
    \frac{8\sqrt{ C_{\text{max}}}(2m-1)}{(1 - \gamma)^2} \theta & \leq \frac{\kappa}{4}\\
    (2m-1) \sqrt{C_{\text{max}}} \theta & \leq \frac{\kappa(1- \gamma)^2}{32}\\
    \theta & \leq \frac{\kappa(1- \gamma)^2}{32 (2m-1) \sqrt{C_{\text{max}}}}\\
\end{align*}

From \eqref{eq:egss-K} we get:
\begin{align*}
    \frac{2\gamma^{K - 1}}{(1 - \gamma)^2} & \leq \frac{\kappa}{4}\\
    \gamma^{K - 1} & \leq \frac{\kappa(1 - \gamma)^2}{8}\\
    K & \leq \frac{\log \left( \frac{\kappa(1 - \gamma)^2}{8} \right)}{\log(\gamma)} + 1\\
    K & \leq \frac{\log\left(\kappa(1 - \gamma)^2\right) - \log(8)}{\log(\gamma)} + 1\\
\end{align*}

We know that \eqref{eq:egss-main} holds with probability at least $1 - 2KC_{\text{max}} \exp(-2 \theta^2(1-\gamma)^2 n)$.
Therefore from that the rest of values we get:
\begin{align*}
        2KC_{\text{max}} \exp(-2 \theta^2(1-\gamma)^2 n) &\leq \delta\\
        \exp(-2 \theta^2(1-\gamma)^2 n) &\leq \frac{\delta}{2KC_{\text{max}}}\\
        -2 \theta^2(1-\gamma)^2 n &\leq log( \frac{\delta}{2KC_{\text{max}}})\\
        n &\geq \frac{log(\delta) - \log(2KC_{\text{max}})}{2 \theta^2(1-\gamma)^2}\\
\end{align*}

\subsection{MCLSPI-EGSS}
\label{sec:parameter mclspi-egss}
The total error is the following:
\begin{align}
    \frac{8 \eta_2}{(1 - \gamma)^2} + \frac{2 \gamma^{K - 1}}{(1 - \gamma)^2} &\leq \kappa \label{eq:dav-main} \\
    \frac{8}{(1 - \gamma)^2}
    \left(
        b\sqrt{\lambda  \tau} + \left( \frac{\gamma^{H+1}}{1 - \gamma} + \theta \right) \sqrt{ \tau C_{\text{max}}}
        \right) +
        \frac{2\gamma^{K - 1}}{(1 - \gamma)^2}
    &\leq \kappa \nonumber \\
    &\Rightarrow \nonumber \\
    \frac{8}{(1 - \gamma)^2} b\sqrt{\lambda  \tau} & \leq \frac{\kappa}{4} \label{eq:dav-lambda}\\
    \frac{8\sqrt{ \tau C_{\text{max}}}}{(1 - \gamma)^2} \frac{\gamma^{H+1}}{1 - \gamma} & \leq \frac{\kappa}{4} \label{eq:dav-H}\\
    \frac{8\sqrt{ \tau C_{\text{max}}}}{(1 - \gamma)^2} \theta & \leq \frac{\kappa}{4} \label{eq:dav-theta}\\
    \frac{2\gamma^{K - 1}}{(1 - \gamma)^2} & \leq \frac{\kappa}{4} \label{eq:dav-K}
\end{align}
First we assume that $\tau = 1$.

From \eqref{eq:dav-lambda} we get:
\begin{align*}
    \frac{8}{(1 - \gamma)^2} b\sqrt{\lambda } & \leq \frac{\kappa}{4} \\
    \sqrt{\lambda } &\leq \frac{(1 - \gamma)^2\kappa}{32 b}\\
    \lambda &\leq \frac{(1 - \gamma)^4\kappa^2}{1024 b^2 }\\
\end{align*}
From \eqref{eq:dav-H} we get:
\begin{align*}
    \frac{8\sqrt{ C_{\text{max}}}}{(1 - \gamma)^3} \gamma^{H+1} & \leq \frac{\kappa}{4} \\
    d^{\frac{1}{4}} C_{\text{max}}^{\frac{1}{2}} \gamma^{H+1} & \leq \frac{\kappa(1 - \gamma)^3}{32} \\
    \gamma^{H+1} & \leq \frac{\kappa(1 - \gamma)^3}{32 d^{\frac{1}{4}} C_{\text{max}}^{\frac{1}{2}}} \\
    H & \geq \frac{\log\left (\frac{\kappa(1 - \gamma)^3}{32 d^{\frac{1}{4}} C_{\text{max}}^{\frac{1}{2}}} \right)}{
        \log(\gamma)
    } - 1\\
\end{align*}

From \eqref{eq:dav-theta} we get:
\begin{align*}
    \frac{8\sqrt{C_{\text{max}}}}{(1 - \gamma)^2} \theta & \leq \frac{\kappa}{4}\\
    d^{\frac{1}{4}} C_{\text{max}}^{\frac{1}{2}} \theta & \leq \frac{\kappa(1- \gamma)^2}{32}\\
    \theta & \leq \frac{\kappa(1- \gamma)^2}{32 d^{\frac{1}{4}} C_{\text{max}}^{\frac{1}{2}}}\\
\end{align*}

From \eqref{eq:dav-K} we get:
\begin{align*}
    \frac{2\gamma^{K - 1}}{(1 - \gamma)^2} & \leq \frac{\kappa}{4}\\
    \gamma^{K - 1} & \leq \frac{\kappa(1 - \gamma)^2}{8}\\
    K & \leq \frac{\log \left( \frac{\kappa(1 - \gamma)^2}{8} \right)}{\log(\gamma)} + 1\\
    K & \leq \frac{\log\left(\kappa(1 - \gamma)^2\right) - \log(8)}{\log(\gamma)} + 1\\
\end{align*}

We know that \eqref{eq:dav-main} holds with probability at least $1 - 2KC_{\text{max}} \exp(-2 \theta^2(1-\gamma)^2 n)$.
Therefore from that the rest of values we get:
\begin{align*}
        2KC_{\text{max}} \exp(-2 \theta^2(1-\gamma)^2 n) &\leq \delta\\
        \exp(-2 \theta^2(1-\gamma)^2 n) &\leq \frac{\delta}{2KC_{\text{max}}}\\
        -2 \theta^2(1-\gamma)^2 n &\leq \log( \frac{\delta}{2KC_{\text{max}}})\\
        n &\geq \frac{\log(\delta) - \log(2KC_{\text{max}})}{2 \theta^2(1-\gamma)^2}\\
\end{align*}